\documentclass[preprint, 11pt, authoyear]{article}
\usepackage{jmlr2e} 

\usepackage{afterpage}
\usepackage{mathrsfs}
\usepackage{bm}
\usepackage{mathtools}

\usepackage{refcount}
\newcommand{\myfootnotemark}{\footnotemark\label{fn:mark}}
\newcommand{\myfootnotetext}[1]{\footnotetext{#1\label{fn:text}%
        \edef\fnmark{\getpagerefnumber{fn:mark}}%
        \edef\fntext{\getpagerefnumber{fn:text}}%
        \ifx\fnmark\fntext\else\ClassWarning{}{footnote mark and text on different pages!}\fi}}

\usepackage{appendix}
\usepackage{hyperref}
\usepackage{amsmath}
\usepackage{amssymb}
\usepackage{graphicx}
\usepackage{amsfonts}
\usepackage{latexsym}
\usepackage{pdfsync}
\usepackage{subfigure}
\usepackage{color}
\usepackage{tabularx}
\usepackage[english]{babel}
\usepackage{array}
\usepackage{algorithm}
\usepackage{algorithmic}
\usepackage{esint}

\usepackage[bottom]{footmisc}

\usepackage{enumerate}

\usepackage{bbm}

\usepackage{color, graphicx}              
\usepackage{xcolor}
\usepackage{caption}
\usepackage{epstopdf}
\usepackage{tabularx}

\newtheorem{assumption}[theorem]{Assumption}

\numberwithin{equation}{section}

\usepackage{hyperref}
 \usepackage{multirow}

\usepackage{tikz}
\DeclareFontFamily{OT1}{pzc}{}
\DeclareFontShape{OT1}{pzc}{m}{it}{<-> s * [1.10] pzcmi7t}{}
\DeclareMathAlphabet{\mathpzc}{OT1}{pzc}{m}{it}

\numberwithin{equation}{section}

\newcommand{\rhoT}{\rho_T}
\newcommand{\rhoL}{\rho_T^L}
\newcommand{\rhoLM}{\rho_T^{L,M}}
\newcommand{\brhoL}{\bm{\rho}_T^L}
\newcommand{\brhoT}{\bm{\rho}_T}
\newcommand{\brho}{\bm{\rho}}

\newcommand{\mbf}[1]{\boldsymbol{#1}}

\newcommand{\dbinnerp}[1]{\langle\hspace{-1mm}\langle{#1}\rangle\hspace{-1mm}\rangle}
\newcommand{\abs}[1]{\big| #1 \big|}

\newcommand{\bxm}{\bx^{(m)}}
\newcommand{\dotbxm}{\dot\bx^{(m)}}
\newcommand{\bXm}{\bX^{(m)}}

\newcommand{\nbrm}{r^{(m)}}

\newcommand{\norm}[1]{\left\| #1 \right\|}

\newcommand{\real}{\mathbb{R}}

\newcommand{\br}{\mbf{r}}

\newcommand{\bx}{\mbf{x}}
\newcommand{\bX}{\mbf{X}}

\newcommand{\bz}{\mbf{z}}

\newcommand{\mE}{\mathcal{E}}
\newcommand{\bmE}{\mbf{\mathcal{E}}}

\newcommand{\mK}{\mathcal{K}}

\newcommand{\R}{\real}

\renewcommand{\dim}{d}
\newcommand{\numcl}{K}
\newcommand{\idxcl}{k}
\newcommand{\spaceM}{S}
\newcommand{\cl}{C}
\newcommand{\clof}{\mathpzc{k}}

\newcommand{\bL}{\mbf{L}}

\newcommand{\intkernel}{\phi}
\newcommand{\lintkernel}{\widehat{\intkernel}}
\newcommand{\bintkernel}{{\bm{\phi}}}
\newcommand{\blintkernel}{{\widehat{\bm{\phi}}}}
\newcommand{\intkernelvar}{\varphi}
\newcommand{\bintkernelvar}{{\bm{\varphi}}}

\newcommand{\rhsfo}{\mathbf{f}}
\newcommand{\hypspace}{\mathcal{H}}
\newcommand{\bhypspace}{\mbf{\mathcal{H}}}
\newcommand{\E}{\mathbb{E}}

\newcommand{\probIC}{\mu_0}

\newcommand{\trans}[1]{#1^{\top}}

\newcommand{\argmin}[1]{\underset{#1}{\operatorname{arg}\operatorname{min}}\;}

\newcommand{\supp}[1]{\text{supp}(#1)}

\usepackage{soul}
\DeclareMathAlphabet{\mathpzc}{OT1}{pzc}{m}{it}
\newcommand{\tabincell}[2]{\begin{tabular}{@{}#1@{}}#2\end{tabular}}

\DeclareMathOperator*{\esssup}{ess\,sup}

\jmlrheading{}{}{}{}{}{Learning governing laws in interacting particle systems}

\ShortHeadings{Learning  governing laws in interacting particle systems}{Lu, Maggioni and Tang}
\firstpageno{1}

\begin{document}

\title{Learning interaction kernels in heterogeneous systems of agents from multiple trajectories}

\author{\name Fei Lu$^{1,3}$ \email feilu@math.jhu.edu	 \\
       \AND
       \name Mauro Maggioni$^{1,2,3}$ \email mauromaggionijhu@icloud.com \\
       \AND
       \name Sui Tang$^{1,4}$  \email suitang@math.ucsb.edu\\
       \addr $^1$Departments of Mathematics, $^2$Applied Mathematics and Statistics, \\
       $^3$Johns Hopkins University, 3400 N. Charles Street, Baltimore, MD 21218, USA\\
       $^4$University of California, Santa Barbara, 552 University Rd, Isla Vista, CA 93117, USA
}

\editor{}

\maketitle

\begin{abstract}
Systems of interacting particles or agents have wide applications in many disciplines such as Physics, Chemistry, Biology and Economics.  These systems are governed by interaction laws, which are often unknown: estimating them from observation data is a fundamental task that can provide meaningful insights and accurate predictions of the behaviour of the agents. In this paper, we consider the inverse problem of learning interaction laws given data  from multiple trajectories, in a nonparametric fashion, when the interaction kernels depend on pairwise distances. We establish a condition for learnability of interaction kernels, and construct estimators that are guaranteed to converge in a suitable $L^2$ space, at the optimal min-max rate for 1-dimensional nonparametric regression.  We propose an efficient learning algorithm based on least squares,  which can be implemented in parallel for multiple trajectories and is therefore well-suited for the high dimensional, big data regime.  Numerical simulations on a variety examples, including opinion dynamics, predator-swarm dynamics and heterogeneous particle dynamics, suggest that the learnability condition is satisfied in models used in practice, and the rate of convergence of our estimator is consistent with the theory. These simulations also  suggest that our estimators are robust to noise in the observations, and produce accurate predictions of dynamics in relative large time intervals, even when they are learned from data collected in short time intervals. 

\end{abstract}

\ \\



\acks{FL and MM are grateful for partial support from NSF-1913243, FL for NSF-1821211; MM for NSF-1837991, NSF-1546392 and AFOSR-FA9550-17-1-0280; ST for NSF-1546392 and AFOSR-FA9550-17-1-0280 and AMS Simons travel grant.}


\vskip0.5cm

\section{Introduction}
Systems of interacting particles and agents arise in a wide variety of disciplines including interacting particle systems in Physics (see \cite{DOCBC2006,VCBJCS1995,CDOMBC2007,BDT2017}), predator-swarm systems in Biology (see    \cite{hemelrijk2011some,toner1995long}) , and opinions on interacting networks  in social science (see \cite{olfati2004consensus,MT2014,KMAW2002}).  The interacting system is often modeled by a system of ODEs where the form of the equations is given and the unknowns are the interaction kernels.  Inference of these kernels is useful for modeling and predictions, yet it is a fundamental challenge. In the past, due to the limited amount of data, the estimation of interaction kernels often relied on strong a priori assumptions, which reduced the estimate to a small number of parameters within a given family of possible kernels. With the increased collection and availability of data, computational power, and data storage, it is expected to enable the automatic discovery of interaction laws from data, under minimal assumptions on the form of such laws. We consider the following inverse problem:  given trajectory data of a particle/agent system, collected from  different experiment trials, how to discover the interaction law?  We use tools from statistical and machine learning to propose a learning algorithm guided by rigorous analysis to estimate the interaction kernels, and accurately predict the dynamics of the system using the estimated interaction kernels.

Many governing laws of complex dynamical processes are presented in the form of (ordinary or partial) differential equations. The problem of learning differential equations from data, including those arising from systems of interacting agents,  has attracted continuous attention of researchers from various disciplines.  Pioneering work of learning interaction kernels in system of interacting agents can be found in the work of  \cite{LLEK2010,KTIHC2011, wang2018inferring}, where the true kernels are assumed to lie in the span of  a predetermined  set of template functions and  parametric regression techniques are used to recover the kernels either from  real (see \cite{LLEK2010, KTIHC2011}) or the synthetic (see \cite{wang2018inferring})  trajectory data.  In \cite{BFHM17}, the authors considered learning interaction kernels from data collected from a single trajectory in a nonparametric fashion and a recovery guarantee is proved when the number of agents goes to infinity.  For nonlinear coupled dynamical systems,   symbolic regression techniques are employed to reveal governing laws in various systems from experiment trajectory data sets without any prior knowledge about the underlying dynamics (see \cite{bongard2007automated, schmidt2009distilling}).  Recently, the problem of learning high dimensional nonlinear differential equations from the synthetic trajectory data, where the dynamics are governed by a few numbers of active terms in a prescribed large dictionary, has raised a lot of attention.  Sparse regression approaches include SINDy (\cite{BPK2016, RBPK2017, BKP2017,boninsegna2018sparse}), LASSO (\cite{Schaeffer6634, han2015robust, kang2019ident}), threshold sparse Bayesian regression (\cite{zhang2018robust}),  have been shown to enable effective identification of  the activate terms  from the trajectory data sets.  Recovery guarantees are established under suitable assumptions on the noise in the observation and randomness of data (\cite{TranWardExactRecovery, schaeffer2018extracting}).  There are also approaches using deep learning techniques to learn ODEs (see \cite{raissi2018multistep,  rudy2019deep}) and PDEs (see \cite{raissi2018deep, raissi2018hidden, long2017pde}) from the synthetic trajectory/solution data sets. These techniques make use of approximation power of the deep neural network to construct estimators for response operators  and have shown superior empirical performance. In the Statistics community, parameter estimation for differential equations  from trajectory data has been studied in \cite{varah1982spline,brunel2008parameter,liang2008parameter,cao2011robust,ramsay2007parameter}, among others, and references therein. The form of the differential equations is given, and the unknown is a parameter $\mbf{\theta}$ in Euclidean space.  Approaches include the trajectory matching method, that chooses $\mbf\theta$ so as to maximize agreement with the trajectory data, the gradient matching method that seeks $\mbf\theta$ to fit the right-hand side of the ODEs to the velocities of the trajectory data,  and the parameter cascading method that combines the virtues of previous two methods, but avoids the heavy computational overhead of trajectory matching method and is applicable to partial/indirect observations which the gradient matching method currently can not handle. A review of the parameter estimation problem in system of ODEs may be found in \cite{ramsay2005functional}. The identifiability of $\mbf{\theta}$ in general systems of nonlinear  ODEs from trajectory data is challenging and a topic of current research. It is often assumed that the parameters are identifiable from the trajectory data, i.e, different $\mbf{\theta}$ would yield different trajectories during the observation period. There is no easy way to check this assumption from data, and characterizations of identifiability exist for some special cases (e.g., see \cite{dattner2015optimal} for systems of ordinary differential equations with a linear dependence on (known functions of) the parameters. We refer the interested reader to \cite{miao2011identifiability} and references therein for a comprehensive survey of this topic.
 
 In this paper, we restrict our attention to learning governing laws in first order particle-/agent-based systems with pairwise interactions, whose magnitudes only depend on mutual distances. We consider an $N$-agent system with $K$ types of agents in the Euclidean space $\mathbb{R}^d$ and denote by $\{\cl_\idxcl\}_{\idxcl=1}^\numcl$ the partition of the set of indices $\{1,\dots,N\}$ of the agents corresponding to their type. Suppose there are different interaction kernels for each ordered pair of types and that the agents evolve according to the system of ODEs :
\begin{align}
\label{e:firstordersystem}
  \dot{\bx}_i(t) &= \sum_{i' = 1}^{N}\frac{1}{N_{\clof_{i'}}}\intkernel_{\clof_i\clof_{i'}}(||\bx_{i'}(t)-\bx_i(t)||)(\bx_{i'}(t)-\bx_i(t))\,, i=1,\cdots,N\,.
\end{align}
where $\dot\bx_i(t):=\frac{d}{dt}\bx_i(t)$;$\clof_i$ is the index of the type of agent $i$, i.e. $i\in C_{\clof_i}$; $N_{\clof_{i'}}$ is the number of agents of type $C_{\clof_{i'}}$. The {\em{interaction kernel}} $\intkernel_{\clof_i\clof_{i'}}:\R_+\rightarrow\R$ governs how agents in type $\cl_{\clof_i}$ influence agents in type $\cl_{\clof_{i'}}$. In this system, the velocity of each agent is obtained by superimposing the interactions with all the other agents, with each interaction being in the direction to the other agent, weighted by the interaction kernel evaluated at the distance to the other agent.  We will let $r_{ii'}:= \norm{\bx_{i'} - \bx_i}$, $\mbf{r}_{ii'} := \bx_{i'} - \bx_i$. We assume that the interaction kernels $\smash{\intkernel_{\clof_i\clof_{i'}}}$ are the only unknown factors in \eqref{e:firstordersystem}; in particular, the sets $\cl_k$'s  are pre-specified. { Below we summarize the notation used in first  order models. 
\begin{table}[H]
\centering
\small{
\small{\begin{tabular}{ c | c }
\hline
Variable                    & Definition \\
\hline\hline
$\bx_i(t)\in \R^d$ & state vector (position, opinion, etc.) of agent $i$ at time $t$ \\
\hline
$\|\cdot\|$ & Euclidean norm in $\mathbb{R}^d$ \\
\hline
$\br_{ii'}(t), \br_{ii''}(t)\in \R^d$ & $\bx_{i'}(t)-\bx_{i}(t)$, $\bx_{i''}(t)-\bx_{i}(t)$  \\
\hline
$r_{ii'}(t), r_{ii''}(t)\in \R^+$ & $r_{ii'}(t)=\|\br_{ii'}(t) \|,  r_{ii''}(t)=\|\br_{ii'}(t)\| $ \\
\hline
$N$                         & number of agents \\
\hline
$K$                         & number of types \\

\hline
$N_{\idxcl}$                & number of agents in type $\idxcl$ \\
\hline
$\clof_i$                   & type of agent $i$ \\
\hline
$C_k$                         & the set of indices of the agents of type $k$   \\
\hline
$\intkernel_{kk'} $ & interaction kernel for the influence of agents of type $k$ on agents of type $k'$ \\
\hline
\end{tabular}}  
}
\caption{Notation for first-order models}
\label{tab:1stOrder_def} 
\end{table}
}

We let $\bX(t):=(\bx_i(t))_{i=1}^N$ in the state space $\R^{\dim N}$ be the vector describing the state of the system; we let $\bintkernel:=(\intkernel_{\idxcl\idxcl'})_{\idxcl,\idxcl'=1}^{\numcl,\numcl}$ denote the interaction kernels,  and $\rhsfo_{\bintkernel}(\bX)\in\R^{dN}$  be the vectorization of the right hand side of \eqref{e:firstordersystem}. We can then rewrite the equations in \eqref{e:firstordersystem} as 
\begin{align}\label{vectorform}
  \dot{\bX}(t)= \rhsfo_{\bintkernel}(\bX(t)).
\end{align}

The observational data   $\bX_{\text{traj},M,\mu_0}:=\{\bX^{(m)}(t_l), \dot{\bX}^{(m)}(t_l)\}_{l, m = 1,1}^{L, M}$ consist of the positions and velocities of all agents observed at time $0=t_1<\dots<t_L=T$; $m$ indexes trajectories corresponding to different initial conditions $\{\bX^{(m)}(0)\}_{m=1}^{M}$, which are assumed sampled i.i.d from a probability measure $\probIC$ on the state space $\mathbb{R}^{dN}$.  We will consider the case where the times $t_l$ are equi-spaced, but minor and straightforward modifications would yield the general case. 

The goal is to learn the interaction kernels $(\intkernel_{kk'})_{k,k'=1,1}^{K,K}$  from  $\bX_{\text{traj},M,\mu_0}$ and predict the dynamics given a new initial condition drawn from $\mu_0$.  The case of $K=1$ and $M=1$ has been considered in \cite{BFHM17} in the mean field regime $(N \rightarrow \infty)$; here we consider the case when the number of agents  $N$ is fixed and the number of  observed trajectories $M$  is large, and the more general case of multi-type agents (for some instances of which the mean-field theory is fraught with difficulties).{ In summary, we consider the regime where $L$ fixed, and $M\rightarrow \infty$. No ergodicity  requirement  on the dynamical system is needed}.

Inspired by the work in \cite{BFHM17}, we introduced a risk functional {in our recent work \cite{lu2019nonparametric}} that exploits the structure of the system \eqref{e:firstordersystem}, and minimize it over a  hypothesis function class  $\bhypspace_{M} $, to obtain estimators of the true kernels $\bintkernel$:
\begin{equation}
 \begin{aligned}\label{firstorder:empircalerrorfunctional}
 \mbf{\mE}_{M}(\bintkernelvar):&= \frac{1}{ML}\sum_{m = 1, l = 1}^{M, L}\norm{\dot{\bX}^{(m)}(t_l) - \rhsfo_{\bintkernelvar}(\bX^{(m)}(t_l))}^2_{\mathcal{S}}\,,\\
 \blintkernel_{M,\bhypspace_{M}}: &= \argmin{\bintkernelvar\in\bhypspace_{M}} \mbf{\mE}_{M}(\bintkernelvar)
  \end{aligned} 
  \end{equation}
  where $\bintkernelvar =(\intkernelvar_{\idxcl\idxcl'})_{\idxcl,\idxcl'=1}^{\numcl,\numcl} \in \bhypspace_{M}$,  $\rhsfo_{\bintkernelvar}$ denotes  the right hand side of \eqref{vectorform} with the interaction kernels $\bintkernelvar$, and $\norm{\cdot}_{\mathcal{S}}$ is chosen to equalize the contributions across types accounting for possibly different $N_{\idxcl}$'s:
\begin{equation}\label{firstorder: classnorm}
\norm{\bX}^2_{\mathcal{S}} = \sum_{i = 1}^N\frac{1}{{N_{\clof_i}}}\norm{\bx_i}^2\,.
\end{equation}

\begin{figure}[tbp]
\centering     
 \includegraphics[scale=0.44]{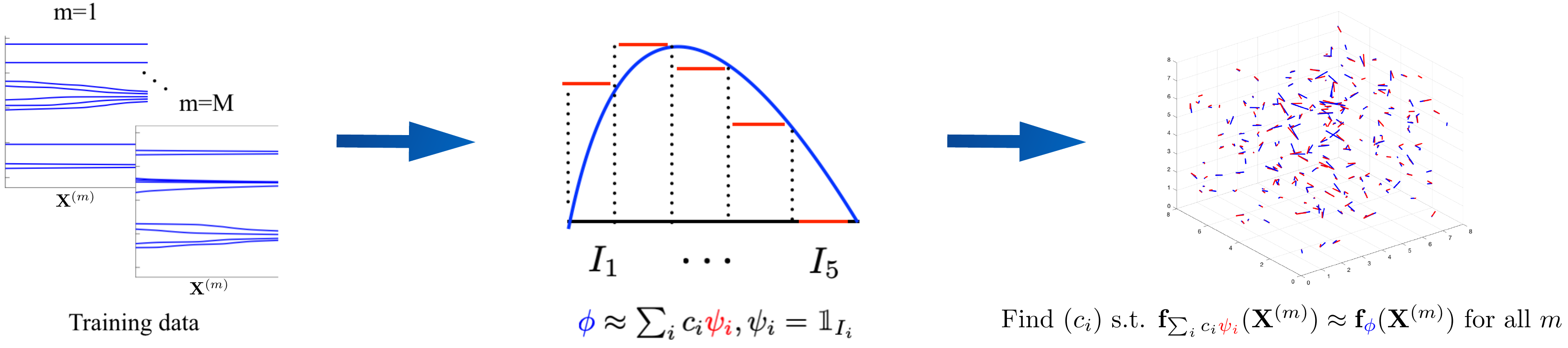}
\caption{ \textmd{{ Overview of the learning approach for homogeneous systems. Left: training data. Middle: $\{\psi_{i}\}$: indicator functions. Right: the blue lines represent the velocity field generated  by the true kernel $\intkernel$. The red lines represent the velocity field generated by the estimator. }} }\label{f:overview}
\end{figure}

The weights play a role in balancing the risk functional across different types of agents, and it is particularly important in cases where types have dramatically different cardinalities. For example, in a predator-swarm system with a large number of preys but only a small number of predators, the unweighted error functional would pay most attention to learning the interaction kernels corresponding to the preys, and mostly disregarding the learning of the interaction kernels corresponding to predators.

\begin{table}[H]
\centering
\small{
\small{\begin{tabular}{ c | c }
\hline
Notation                    & Definition \\
\hline\hline
$M$   & number of trajectories \\
\hline
$L$ & number of time instances in $[0,T]$ \\
\hline
$\mu_0$ & initial conditions sampled from $\mu_0$\\
\hline
$\bX_{\text{traj},M,\mu_0}$ & empirical observations \\
\hline
 $\| \bX \|_{\mathcal{S}}$   &    $\sum_{i = 1}^N\frac{1}{{N_{\clof_i}}}\norm{\bx_i}^2$  \\
\hline
$\mbf{\mathcal{E}}_{M}(\cdot)$ &  empirical error functional\\
\hline
$\bintkernel=(\intkernel_{kk'}), k, k'=1,\cdots, K$ &  true interaction kernels  \\
\hline
$\hypspace_{kk'}$ &  the hypothesis spaces for $\intkernel_{kk'}$\\
\hline
$\{\psi_{kk',p}\}_{p=1}^{n_{kk'}}$ &basis for  $\hypspace_{kk'}$\\
\hline
$\bhypspace=\oplus_{kk'} \hypspace_{kk'}, k, k'=1,\cdots, K$ &  the hypothesis spaces for $\bintkernel$\\
\hline
$\bintkernelvar=(\intkernelvar_{kk'}), k,k'=1,\cdots,K$ &  $\bintkernelvar \in \bhypspace$ with $\intkernelvar_{kk'} \in \mathcal{H}_{kk'}$ \\
\hline
$\widehat \bintkernel_{M,\bhypspace}$ & $\mathrm{argmin}_{\bintkernelvar \in \bhypspace} \mathcal{E}_{M}(\bintkernelvar)$  \\
\hline\hline
\end{tabular}}  
}
\caption{Notation used in learning approach}
\label{tab:1stOrderlearning_def} 
\end{table}

The learning problem addressed in this paper can be informally summarized as:  
\begin{itemize}
\item[(Q)]For which choice of the hypothesis spaces $\{\bhypspace_M\}$ does $\widehat \bintkernel_{M,\bhypspace_M} \rightarrow \bintkernel$ for $M\rightarrow \infty$? For such choices, in which norm does the convergence hold, and what is the rate of convergence? 
\end{itemize}

Our learning problem is  closely related to a classical nonparametric regression problem considered by the statistics and machine learning community:  given  samples   $\{(\bz_i,g(\bz_i)+\mbf{\epsilon}_i)\}_{i=1}^{M}$ with the $(\bz_i,\mbf{\epsilon}_i )$'s  drawn  i.i.d    from an unknown joint distribution $\brho$ defined on the sample space,  and the noise term satisfies $\mathbb{E}[\mbf{\epsilon}_i]=0$, the goal is to learn an unknown function $g:\mathbb{R}^{D} \rightarrow \mathbb{R}$ with prior assumptions on its regularity (e.g, $g$ is s-H\"older). A common approach (see \cite{CS02,Gyorfi06}) is to choose an hypothesis class $\mathcal{F}_{M}$ depending on the sample size and the regularity of $g$,  and then define $\widehat g_{M,\mathcal{F}_M}$ as the minimizer of the empirical risk functional
$$\widehat g_{M,\mathcal{F}_M}:=\argmin{f \in \mathcal{F}_{M} } \frac{1}{M}\sum_{i=1}^{M}(g(\bz_i)+\mbf{\epsilon}_i-f(\bz_i))^2.$$

If we let $\bz=\bX, \mbf{\epsilon}=0$ and $g=\rhsfo_{\bintkernel}(\bX)$ in the regression setting above,  our trajectory data is of the type needed for regression. However,  our data are correlated in time due to the underlying dynamics. Even if we ignore the the lack of independence (e.g., consider $L=1$),  the  application of the existing regression approaches to learning $\rhsfo_{\bintkernel}(\bX):\mathbb{R}^{dN} \rightarrow \mathbb{R}^{dN}$ with noisy observations would lead at best to the optimal min-max  rate $O(M^{-\frac{1}{dN}})$, showing the effect of the curse of dimensionality of the state space. This significantly restricts their usability as soon as, say, $dN \geq16$.


We  proceed in a different direction: we take the structure of \eqref{e:firstordersystem} into account and move our regression target to the interaction kernels $\bintkernel=(\intkernel_{kk'})_{k,k'=1,1}^{K,K}$, to take advantage of the fact that each interaction kernel $\intkernel_{kk'}$ is defined on $\mathbb{R}^{+}$.  This becomes an inverse problem. The observed variables  for each interaction kernel $\intkernel_{kk'}$ are the pairwise distances $\{r_{ii'}\}_{i \in C_k, i'\in C_{k'}}$, and even in the case of $L=1$, their samples are correlated (e.g. $r_{ii'}$ and $r_{ii''}$). For $L>1$ the nonlinear forward map of the dynamical system creates complicated dependences. This is in contrast with the i.i.d assumption on the samples of observation variables in a classical regression setting.   Furthermore, the values of $\bintkernel$ at the observed variables are not measured directly,  but rather linear combinations thereof are observed. Constraining upon such observations leads to a system of equations, that is typically underdetermined given the trajectory data  (see analysis in section \ref{measuresandfunctionspaces}). This may cause severe ill-posedness in our inverse problem. 

{ In this paper, we shall consider the regime where either the velocity is observed, or $\frac{T}{L}$ is sufficiently small. In our numerical section \ref{numericalimplementation},  we use the finite difference method to obtain accurate estimates of the velocity data from position data.  However, our learning theory framework is valid as long as we have accurate velocity data. For simplicity, in the theory part, we assume the velocity data are observed.  }

\subsection{{Contribution} and Main results}

{ The main focus of this work is  to provide a theoretical foundation for the nonparametric learning of the interaction kernel in agent-based systems from multiple trajectories. This work is built on our recent paper \cite{lu2019nonparametric},  where we introduced the problem, and  demonstrated by numerical experiments the effectiveness of the proposed approach on a wide variety of agent-based systems, with partial theoretical results  for the special case $K=1$ (i.e., homogeneous system) outlined without proof. The new contributions of this work are: (i)  The theoretical framework is generalized with full details to cover the more general and widely used \emph{heterogeneous} systems (with $K$ different types of agents), including new analysis of learnability and consistency for multiple kernels that are not presented in \cite{lu2019nonparametric}.  (ii) We add numerical validation of the learning theory on three representative heterogeneous systems that are used in practice, including learnability of kernels,  the consistency of the estimator, near optimal convergence rate of our estimators, and  the decay rate of trajectories prediction errors. (iii) We also test the robustness of the learning approach with respect to multiple type of noise in the observations. }

We establish the learning theory by leveraging classical nonparametric regression  (e.g. \cite{FGlocalPolynomial,CS02,binev2005universal,Gyorfi06}), by  using the coercivity condition to ensure learnability of the interaction kernels, and by introducing a dynamics-adapted probability measure $\brhoL$ on the pairwise distance space. We use  $\bL^2(\brhoL)$ as the function space for learning; the performance of the estimator $\widehat \bintkernel_{ M,\bhypspace_M}$ is evaluated by studying its convergence either in probability or in the expectation as number of observed trajectories $M$ increases, by providing bounds on
\begin{align}\label{metric}
P_{\mu_0}\{ \|\widehat \bintkernel_{M,\bhypspace_M}(r)r-\bintkernel(r)r\|_{\bL^2(\brhoL)}\geq \epsilon\}\quad \text{ and }\quad \mathbb{E}_{\mu_0}[ \|\widehat \bintkernel_{M,\bhypspace_M}(r)r-\bintkernel(r)r\|_{\bL^2(\brhoL)}] \,,
\end{align}
where both the probability and expectation are taken with respect to $\mu_0$, the distribution of the initial conditions.  
Under the coercivity condition, the estimators $\{\widehat \bintkernel_{M,\bhypspace_M}\}_{M=1}^{\infty}$ obtained in \eqref{firstorder:empircalerrorfunctional}  are strongly consistent  and  converge at an optimal min-max rate to the true kernels $\bintkernel$ in terms of $M$, as if we were in the easier (both in terms of dimension and observed quantities) $1$-dimensional nonparametric regression setting with noisy observations.  Furthermore, in the case of $L=1$ and that $\mu_0$ is exchangeable Gaussian, we prove that the coercivity condition holds, and show that even the constants in the error bound can be independent of $N$, making the bounds essentially dimension-free. Numerical simulations suggest that the coercivity condition holds for even larger classes of interaction kernels and initial distributions, and for different values of $L$ as long as $\brhoL$ is not degenerate.

We exhibit an efficient algorithm to compute the estimators based on the regularized least-squares problem \eqref{firstorder:empircalerrorfunctional}, and demonstrate the learnability of interaction kernels on various systems, including opinion dynamics, predator-swarm dynamics and heterogeneous particle dynamics. {Our theory results holds for rather general hypothesis function spaces, with a wide variety of choices. In our numerical section we shall use local basis functions consisting of piece-wise polynomials due to their simplicity and ease of efficient computation.} The numerical results are consistent with the convergence rate from the learning theory and demonstrate its applicability.  In particular, the convergence rate has no dependency on the dimension of the state space of the system, and therefore avoids the curse of dimensionality and makes these estimators well-suited for the high dimensional data regime. The numerical results also suggest that our estimators are robust to noise and predict the true dynamics faithfully,  in particular, the collective behaviour of agents, in a large time interval,  even though they are learned from trajectory data collected in a very short time interval.

{ \subsection{Discussion and future work}
 
 The regime we considered in this manuscript is that the data is collected from $M$ independent short trajectories, and the convergence rate of our estimators is with respect to $M$. In such a regime, we require neither that the system to be ergodic nor the trajectory to be long, nor the number of observations along a trajectory to be large. It is different from regime where  the data is collected from a  long trajectory of an ergodic system. There is of course a natural connection between these two regimes, though,  when the system is ergodic. In that case we may view the long trajectory as many short trajectories starting from initial conditions sampled from the invariant measure, immediately obtaining  similar results using our framework. We need the proper initial distributions such that the coercivity condition holds true and therefore the inverse problem is well-posed. In fact,  many types of distributions, particularly the common distributions such as Gaussian or uniform distribution, satisfy the coercivity condition, and therefore ensure the identifiability of the interaction kernel.  Our regime is of interest in many applications either when the model is not ergodic, or when the invariant measure is unavailable, or when observation naturally come in the form of multiple, relatively short trajectories.

Another important issue in our learning problem is the effective sample size (ESS) of the data, particularly  the ESS with respect to the regression measure $\rho_T^L$. 
For non-ergodic systems (which are the focus of this manuscript), our results in \cite{lu2019nonparametric} (see e.g. Fig.S16) suggest that the ESS does not necessarily increase in $L$ because the $L^2(\rho_T^L)$ error does not always decrease in $ML$. For ergodic stochastic systems, we  obtain in \cite{LMT20} optimal convergence rate in $ML$, as suggested by the ergodic theory that the ESS is proportional to $ML$. 

 The learning theory and algorithm developed in this paper could be extended to an even larger variety of agent-based systems, including second-order systems, stochastic interacting agent systems, and discrete-time systems with non-smooth interaction kernels; these cases require different analyses, and will be explored in separate works.   
}

\subsection{Notation} \label{functionspace}
Throughout this paper, we use bold letters to denote the vectors or vector-valued functions.  
Let $\mathcal{K}$ be a compact (or precompact) set of Euclidean space; Lebesgue measure will be assumed unless otherwise specified. We define the following function spaces 
\begin{itemize}
\item $L^{\infty}(\mathcal{K})$: the space of bounded scalar valued functions on $\mathcal{K}$ with the infinity norm 
$$\|g\|_{\infty}=\esssup_{x\in \mathcal{K}}|g(x)|;$$
\item $\bL^{\infty}(\mathcal{K}):=\bigoplus_{ \idxcl,\idxcl'=1,1}^{K,K}L^{\infty}(\mathcal{K})$ with $\|\mbf{f}\|_{\infty}=\max_{k,k'}\|f_{kk'}\|_{\infty}, \forall \mbf{f} \in \bL^{\infty}(\mathcal{K}) $;
\item $C(\mathcal{K}):$ the closed subspace of $L^{\infty}(\mathcal{K})$ consisting of  continuous functions;
\item $C_c(\mathcal{K}):$ the set of functions in $C(\mathcal{K})$ with compact support;

\item $C^{k,\alpha}(\mathcal{K})$ for $k \in \mathbb{N}, 0<\alpha \leq 1$: the space of  $k$ times continuously differentiable functions whose  $k$-th derivative is  H\"older continuous of order $\alpha$.  {In the special case of $k=0$ and $\alpha=1$, $g \in  C^{0,1}(\mathcal{K})$ is called Lipschitz  function space, and denote by $\text{Lip}(\mathcal{K})$. The Lipschitz seminorm of $g \in \text{Lip}(\mathcal{K})$ is defined as
$\text{Lip}[g]:=\sup_{x\neq y} \frac{|g(x)-g(y)|}{\|x-y\|}.$}

\end{itemize}

We use $||\cdot||_\infty$ as the default norm to define the compactness of sets in $\bL^{\infty}(\mathcal{K})$ and its subspaces. The prior on the interaction kernels $\bintkernel$ is that $\bintkernel$ belong to a compact set of $\textbf{Lip}(\mathcal{K}):=\bigoplus_{ \idxcl,\idxcl'=1,1}^{K,K}\text{Lip}(\mathcal{K})$. The regularity condition is presented either by specifying a compact set in a function class (e.g, $C^{k,\alpha}(\mathcal{K})$) or by quantifying the rate of approximation by a chosen sequence of linear spaces.  We will restrict the estimators to a function space that we call the \textit{hypothesis space}; $\bhypspace$ will be  a subset  of $\bL^{\infty}(\mathcal{K})$, where we show that the minimizer of \eqref{firstorder:empircalerrorfunctional} exists.  We will focus on the compact (finite- or infinite-dimensional) subset of $ \bL^{\infty}(\mathcal{K})$ in the theoretical analysis, however in the numerical implementation we will use finite-dimensional linear spaces. While these linear subspaces are not compact, it is shown that the minimizers over the whole space behave essentially in the same way as the minimizers over compact sets of linear subspaces (e.g., see Theorem 11.3 in \cite{Gyorfi06}). We shall therefore assume the compactness of the hypothesis space in the theoretical analysis, following the spirit of \cite{CS02,binev2005universal,devore2004mathematical}.

\subsection{Relevant background on interacting agent systems}
 The interacting agent systems considered in this paper follow the equations of motion \eqref{e:firstordersystem}. These equations may be derived from physical laws for particle dynamics in gradient form: we can think of each agent as a particle and for type-$k$ agents, there is an associated potential energy function depending only on the pairwise distance between agents: 
\begin{align} \label{potentialfunction}
\mathcal U_k(\bX(t)):=\sum_{i \in C_k}\left( \sum_{i'\in C_k}\frac{1}{2N_{k}} \Phi_{kk}(r_{ii'}(t)) +\sum_{i'\not\in C_{k}}\frac{1}{N_{\clof_{i'}}}\Phi_{k\clof_{i'}}(r_{ii'}(t))\right)
\end{align} with  $\intkernel_{k\clof_{i'}}(r)=\Phi_{k\clof_{i'}}'(r)/r$.
The evolution of agents in each type is driven by the minimization of this potential energy function. This relationship justifies the appearance of the functions 
$\{\intkernel_{kk'}(r)r\}_{k,k'=1}^{K,K}$ in \eqref{metric} and in our main results. 

In the special case of $K=1$, the system is said to be \textit{homogeneous}. It is one of the simplest models of interacting agents, yet it can yield complex, emergent dynamics (\cite{kolokolnikov2013emergent}). A prototypical example  is opinion dynamics in social sciences, where the interaction kernel could be an increasing or decreasing positive function, modeling, respectively, heterophilious and homophilous opinion interactions (\cite{MT2014}), or particle systems in Physics where all particles are identical (e.g. a monoatomic metal). In the case of $K\geq 2$, the system is said to be  \textit{heterogeneous}. Prototypical examples include interacting particle systems with different particle types (e.g. composite materials with multiple atomic types) and predator-prey systems, where the interaction kernels may be negative for small $r$, inducing the repulsion, and positive for large $r$, inducing the attraction. 

{There has been a line of research  fitting real data in systems of  form (1.1) across various disciplines. We refer readers the application in chemistry using Lennard Jones potential  to \cite{cisneros2016modeling}, application  in the exploratory data analysis for animal movement  to \cite{LLEK2010, brillinger2011modelling, brillinger2012use}. When $K=1$ and  the interaction kernel is an indicator function, the corresponding system is the well-known  Hegselmann-Krause Model ( also called flocking model in some literatures of computational social science),we refer readers to \cite{de2014learning,abebe2018opinion} for details. Recently, the systems  have also been applied to learn  the celestial dynamics from Jet Propulsion Lab's data \cite{zhong2020data, MMZ2020}, or learning the cell dynamics from microscopy videos (personal communication)}

 We consider $\bintkernel \in \bL^{\infty}([0,R])$ with the radius $R$ representing the maximum range of interaction between agents. We further assume that $\bintkernel$ lies in the \emph{admissible set}
\begin{align}\label{firstorder:admissibleset}
 \mbf{\mathcal{K}}_{R,\spaceM }: = \{\bintkernelvar=(\intkernelvar_{\idxcl\idxcl'})_{k,k'=1}^{K,K}:   \ \intkernelvar_{\idxcl\idxcl'} \in C_c^{0,1} ([0,R]), \|\intkernelvar_{\idxcl\idxcl'}\|_{\infty} + \text{Lip}[\intkernelvar_{\idxcl\idxcl'}]  \leq \spaceM \}
 \end{align}  
for some $\spaceM >0$. For $\bintkernel \in \mbf{\mathcal{K}}_{R,\spaceM }$,  the system \eqref{e:firstordersystem} is well-posed for any given initial condition (i.e. there is a unique solution, that can be extended to all times) and it is expected to converge for $t \rightarrow \infty$ to configurations of  points whose mutual distances are close to local minimizers of the potential energy function in \eqref{potentialfunction}, corresponding  to steady states of evolution.  We refer to  \cite{kolokolnikov2013emergent,chen2014multi} for the qualitative analysis of this type of systems.

\subsection{Outline and organization}
The remainder of the paper is organized as follows.  In section \ref{learningtheory}, we present the learning theory that establishes a theoretical framework for analyzing the performance of the proposed learning algorithm.  We then discuss the numerical implementation of the learning algorithm in section \ref{sec_finitedifference} and \ref{numericalexamples}, and its performance in various numerical examples in section \ref{main:numericalexamples}. Section \ref{coercivity} presents some theoretical results for the coercivity condition, a key condition for achieving the optimal convergence rate of interaction kernels. Finally,  we present the proof of the main Theorems in the Appendix.  

\section{Learning theory}
\label{learningtheory}

\subsection{Measures and function spaces adapted to the dynamics} \label{measuresandfunctionspaces}
To measure the accuracy of the estimators, we introduce a probability measure,  dependent on the  distribution of the initial condition $\mu_0$ and the underlying dynamical system, and then define the function space for learning. 
We start with a heuristic argument.  In the case  $K=1$,  the interaction kernel $\intkernel$ depends only on  one variable, but it is observed through a collection of non-independent linear measurements with values $\dot{\bx}_i$, the l.h.s. of \eqref{e:firstordersystem}, at locations $r_{ii'}=\|\bx_{i'}-\bx_i\|$, with coefficients $\br_{ii'}=\bx_{i'}-\bx_i$.  One could attempt to recover $\{\intkernel(r_{ii'}) \}_{i,i'}$ from the  equations of $\dot{\bx_i}$'s by solving  the corresponding linear system. Unfortunately, this linear system is usually underdetermined as $dN$ (number of known quantities) $\leq \frac{N(N-1)}{2}$ (number of unknowns) and in general one will not be able to recover the values of $\intkernel$ at locations $\{r_{ii'}\}_{i,i'}$.  
We take a different route, to leverage observations through time: we note that  the pairwise distances $\{r_{ii'}\}_{i,i'}$  are ``equally'' important in a homogeneous system, and introduce a probability density $\rhoL$ on $\R_+$ 
\begin{equation}\label{e:rhoL}
\rhoL(r)  = \frac{1}{\binom N2 L}\sum_{l= 1}^{L}\E_{\mu_0}\sum_{i,i'=1,\\ i< i' }^N\delta_{r_{ii'}(t_l)}(r) \,,
\end{equation} where the expectation in \eqref{e:rhoL} is with respect to the distribution $\probIC$ of the initial condition. This density is the limit of the empirical measure of pairwise distance
\begin{equation}
\label{e:rhoLM}
\rhoLM(r)  =\frac{1}{\binom N2 L M}\sum_{l,m= 1}^{L,M}\sum_{i,i'=1,\\ i< i' }^N \delta_{\nbrm_{ii'}(t_l)}(r)\,,
\end{equation} 
as $M\to \infty$, by the law of large numbers.  

The measure $\rhoL$ is intrinsic to the dynamical system and  independent of the observations. It can be thought of as  an ``occupancy'' measure, in the sense that for any interval $I \subset\R_+$, $\rhoL(I)$ is the probability of seeing a pair of agents at a distance between them equal to a value in $I$, averaged over the observation time.  It  measures how much regions of $\R_+$ on average (over the observed times and with respect to the distribution $\probIC$ of the initial conditions) are explored by the dynamical system. Highly explored regions are where the learning process ought to be successful, since these are the areas with enough samples from the dynamics to enable the reconstruction of the interaction kernel. Therefore, a natural metric to measure the regression error is the mean square distance in  $L^2(\rhoL)$: for an estimator $\lintkernel_{M,\hypspace}$, we let  
 \begin{equation}
 \label{errornorm}
\text{dist}(\lintkernel_{M,\hypspace},\intkernel) =\|\lintkernel_{M,\hypspace}(\cdot)\cdot-\intkernel(\cdot)\cdot\|_{L^2(\rhoL)}=\bigg( \int_{r = 0}^{\infty} \abs{\lintkernel_{M,\hypspace}(r)r-\intkernel(r)r}^2 \,\rhoL(dr)\bigg)^{\frac{1}{2}}.
 \end{equation} 
 If trajectories were observed  continuously in time, we could consider
\begin{equation}\label{e:rhoT}
\rhoT(r)  = \frac{1}{\binom N2T}\int_{t = 0}^T\E_{\mu_0}\sum_{i,i'=1,\\ i< i' }^N\delta_{r_{ii'}(t)}(r) \, dt\,.
\end{equation}

The natural generalizations of $\rhoL$ and $\rhoT$, defined in \eqref{e:rhoL} and \eqref{e:rhoT}, to the heterogeneous case, for each $k,k'=1,\dots,K$, are the probability measures on $\R_+$ (in discrete and continuous time respectively)
\begin{align}
\rho_T^{L, \idxcl\idxcl'}(r)  &= \displaystyle \frac{1}{LN_{\idxcl\idxcl'}}\sum_{l= 1}^L\!  \mathbb{E}_{\mu_0} \sum_{\substack{i \in C_{\idxcl}, i' \in C_{\idxcl'} \\ i\neq i'}} \delta_{r_{ii'}(t_l)}(r)  \label{e:rhot_norm_mc1} \\
\rhoT^{\idxcl\idxcl'}(r)  &= \displaystyle \frac{1}{TN_{\idxcl\idxcl'}}\int_{t = 0}^T\! \mathbb{E}_{\mu_0} \sum_{\substack{i \in C_{\idxcl}, i' \in C_{\idxcl'} \\ i\neq i'}}\delta_{r_{ii'}(t)}(r) dt,
\label{e:rhot_norm_mc2}
\end{align}
where $N_{\idxcl\idxcl'} = N_{\idxcl}N_{\idxcl'}$ when $\idxcl \neq \idxcl'$ and $N_{\idxcl\idxcl'} ={{N_{\idxcl}}\choose 2}$ when $\idxcl = \idxcl'$.  The error of an estimator $\lintkernel_{\idxcl\idxcl'}$ will be measured by $||\lintkernel_{\idxcl\idxcl'}(\cdot)\cdot - \intkernel_{\idxcl\idxcl'}(\cdot)\cdot||_{L^2(\rho_T^{L, \idxcl\idxcl'})}$ as in \eqref{errornorm}.   For simplicity of notation, we write
\begin{align}\label{e:brhoL}
\brhoL=\bigoplus_{\idxcl,\idxcl'=1,1}^{K,K} \rho_T^{L, \idxcl\idxcl'}\,, {\brhoT=\bigoplus_{\idxcl,\idxcl'=1,1}^{K,K} \rho_T^{\idxcl\idxcl'}\, ,}\quad \bL^2(\brhoL)=\bigoplus\limits_{ \idxcl,\idxcl'=1,1}^{K,K}L^2(\rho_T^{L, \idxcl\idxcl'})\,, \quad
\end{align} 
with  $\|\bintkernelvar\|_{\bL^2(\brhoL)}^2=\sum_{ \idxcl\idxcl'} \|\bintkernelvar_{\idxcl\idxcl'}\|_{\rho_T^{L, \idxcl\idxcl'}}^2$ for any $\bintkernelvar\in \bL^2(\brhoL)$.

\subsection*{Well-posedness and properties of measures}

The probability measures $\{\rho_T^{L, \idxcl\idxcl'}\}_{ \idxcl,\idxcl'=1,1}^{K,K}$ and their continuous counterpart are well-defined, thanks to the following lemma:

\begin{lemma}\label{averagemeasure}
Suppose $\bintkernel \in \mbf{\mathcal{K}}_{R,S}$ (see the admissible set defined in\eqref{firstorder:admissibleset}). Then for each $(\idxcl,\idxcl')$,  the measures $\rho_T^{L, \idxcl\idxcl'}$ and $\rhoT^{\idxcl\idxcl'}$, defined in \eqref{e:rhot_norm_mc1} and  \eqref{e:rhot_norm_mc2},  are regular Borel probability measures on $\R_+$. They are absolutely continuous with respect to the Lebesgue measure  on $\R$ provided that $\probIC$ is absolutely continuous  with respect to the Lebesgue measure on $\R^{dN}$. 
\end{lemma}

We emphasize that the measures $\rho_T^{L, \idxcl\idxcl'}$ and $\rho_T^{\idxcl\idxcl'}$ are both averaged-in-time measures of the pairwise distances, and they have the same properties. The only difference is that they correspond to discrete-time and continuous-time observations, respectively.  In the following, we analyze only the discrete-time observation case using $\rho_T^{L, \idxcl\idxcl'}$, and all the arguments can be extended directly to the continuous-time observation case using $\rho_T^{\idxcl\idxcl'}$.

The  measures $\rho_T^{\idxcl\idxcl'}$ and $\rho_T^{L, \idxcl\idxcl'}$ are compactly supported provided that $\probIC$ is:  
\begin{proposition}\label{compactmeasure} Suppose the distribution  $\probIC$ of the initial condition is compactly supported. Then there exists $R_0>0$,  such that for each $(\idxcl, \idxcl')$, the support of the measure $\rho_T^{\idxcl\idxcl'}$(and therefore $\rho_T^{L, \idxcl\idxcl'}$), is contained in $[0, R_0]$ with $R_0=2C_0+2K\|\bintkernel\|_{\infty}RT$ where $C_0$ only depends on $\supp \probIC$.  
\end{proposition}

The proofs of these results are postponed to sec.\ref{s:proofsPropertiesOfMeasures}.

\subsection{Learnability:  a coercivity condition} 

A fundamental question is the well-posedness of the inverse problem of learning the interaction kernels. Since the least square estimator always  exists for compact  sets  in $\bL^\infty([0,R])$, learnability is equivalent to the convergence of the estimators  to the true interaction kernels as the sample size increases (i.e. $M\to \infty$) and as the compact sets contain better and better approximations to the true kernels $\bintkernel$. 
To ensure such a convergence, one would naturally wish:  (i) that  the true kernel $\bintkernel$ is the unique minimizer of the expectation of the error functional (by the law of large numbers) 
\begin{align}\label{e_errorFtnl_infty2}
\mbf{\mE}_{\infty}(\bintkernelvar) &\coloneqq \lim_{M\to\infty}\mbf{\mE}_{M}(\bintkernelvar) =  \frac{1}{L}\sum_{l = 1}^{L}\mathbb{E}_{\mu_0}\left[\|\dot{\bX}(t_l) - \rhsfo_{\bintkernelvar}(\bX(t_l))\|^2_{\mathcal{S}}\right]\,;
\end{align}
(ii) that the  error of the estimator, say $\widehat\bintkernel$,   is small once  $\mbf{\mE}_{\infty}(\widehat\bintkernel)$ is small since $\mbf{\mE}_{\infty}(\bintkernel)=0$.

Note that $\mbf{\mE}_{\infty}(\bintkernelvar)$ is a quadratic functional of $\bintkernelvar -\bintkernel$; by Jensen's inequality, we have
\begin{align*}
 \mbf{\mE}_{\infty}(\bintkernelvar) < K^2\norm{\bintkernelvar(\cdot)\cdot- \bintkernel(\cdot)\cdot}^2_{L^2(\rhoL)}\,. 
\end{align*}
 If we  bound $\mbf{\mE}_{\infty}(\bintkernelvar)$ from below by $\norm{\bintkernelvar(\cdot)\cdot - \bintkernel(\cdot)\cdot}^2_{L^2(\brhoL)}$, we can conclude (i) and (ii) above. 
This suggests the following coercivity condition:
 
\begin{definition}[Coercivity condition for first-order systems] \label{def_coercivity_firstorder}
Consider the dynamical system defined in \eqref{e:firstordersystem} at time instants $0=t_1<t_2<\dots <t_L=T$, with random initial condition distributed according to the probability measure $\probIC$ on $\R^{dN}$. We say that it satisfies the coercivity condition on a hypothesis space $\bhypspace $ with a constant $c_{L,N,\bhypspace}$ if
\begin{align}\label{firstorder:gencoer}
 c_{L,N,\bhypspace}  := \inf_{ \bintkernelvar \in \bhypspace \backslash \{\mbf 0 \}} \,\frac{ \frac{1}{L}\sum_{l=1}^{L}\E_{\mu_0} \bigg[ \big\|  \rhsfo_{\bintkernelvar}(\bX(t_l)) \big\|_{\mathcal{S}}^2\bigg] }{\|\bintkernelvar(\cdot)\cdot\|_{\bL^2(\brhoL)}^2} >0 \,.
 \end{align} 
A similar definition holds for continuous observations on the time interval $[0,T]$, upon replacing the sum over observations with an integral over $[0,T]$.
\end{definition}

The coercivity condition plays a key role in the learning of the kernel. It ensures learnability by ensuring the uniqueness of minimizer of  the expectation of the error functional, and by guaranteeing convergence of estimator. To see this, apply the coercivity inequality to $\bintkernelvar - \bintkernel$ and suppose $\bintkernelvar - \bintkernel$ lies in $\bhypspace$, we obtain
\begin{align}
c_{L,N,\bhypspace} \norm{\bintkernelvar(\cdot)\cdot  - \bintkernel(\cdot)\cdot }^2_{\bL^2(\brhoL)} \leq  \mbf{\mE}_{\infty}(\bintkernelvar).
\end{align}
From the facts that $\mbf{\mE}_{\infty}(\bintkernelvar) \geq 0$ for any $\bintkernelvar$ and that $\mbf{\mE}_{\infty}(\bintkernel) =0$, we conclude that the true kernel $\bintkernel$ is the unique minimizer of the  $\mbf{\mE}_{\infty}(\bintkernelvar)$. Furthermore, the coercivity condition enables us to control the error of the estimator by the discrepancy  between the error functional {and its expectation} (see Proposition \ref{firstordersystem:convexity}), therefore guaranteeing convergence of the estimator. Finally, the coercivity constant controls the condition number of the inverse problem, guaranteeing numerical stability. We study the consistency and rate of convergence in the next section.

\subsection{Consistency and rate of convergence}
We start from a concentration estimate, in which the coercivity condition plays a fundamental role. 
 
 \begin{theorem}\label{firstorder:maintheorem}
Suppose that $\bintkernel\in \mbf{\mathcal{K}}_{R, \spaceM}$. Let  $\bhypspace_M \subset \mbf{L}^{\infty}([0,R])$ be a compact (with respect to the $\infty$-norm) convex set bounded above by $\spaceM_0\geq \spaceM$.  Assume that the coercivity condition \eqref{firstorder:gencoer} holds true {on $\bhypspace_M$}. Then for any $\epsilon >0$, the estimate 
 \begin{align}\label{firstorder:mainestimate1}
& c_{L,N,\bhypspace_M}\| \blintkernel_{M,\bhypspace_M}(\cdot)\cdot- \bintkernel(\cdot)\cdot\|^2_{\bL^2(\brhoL)}  
 \leq   2\inf_{\bintkernelvar \in \bhypspace_M}\|\bintkernelvar(\cdot)\cdot-\bintkernel(\cdot)\cdot\|^2_{\bL^2(\brhoL)} +2\epsilon
\end{align} 
holds true with probability at least $1-\delta$, provided that 
$M \geq \frac{1152\spaceM_0^2R^2K^4}{c_{L,N,\bhypspace_M}\epsilon}\big(\log (\mathcal{N}(\bhypspace_M,\frac{\epsilon}{48\spaceM_0R^2K^4} ))+\log(\frac{1}{\delta})  \big)$,
where $\mathcal{N}(\bhypspace_M,\frac{\epsilon}{48\spaceM_0R^2K^4} )$ is the covering number of $\bhypspace_M$ with respect to  the $\infty$-norm.  
\end{theorem} 

If we choose a family of compact convex hypothesis space $\bhypspace_M$ that contain better and better approximations to the true interaction kernels $\bintkernel$; then the concentration estimate \eqref{firstorder:mainestimate1} yields the following consistency result:

\begin{theorem}[Consistency of estimators]\label{main:consistency}  Suppose that 
$\{\bhypspace_M\}_{M=1}^{\infty} \subset \bL^{\infty}([0,R])$ is a family of compact convex subsets such that 
$$\inf_{\mbf{f}\in \bhypspace_M}\|\mbf{f}-\bintkernel\|_{\infty} \xrightarrow{M\rightarrow \infty} 0$$
 Suppose $ \cup_{M}\bhypspace_M$ is compact  in  $\bL^{\infty}([0,R])$ and the coercivity condition holds true on $ \cup_{M}\bhypspace_M$. Then $$\lim_{M\rightarrow \infty}\|\widehat \bintkernel_{M,\bhypspace_M}(\cdot)\cdot-\bintkernel(\cdot)\cdot\|_{\bL^2(\brhoL)} =0 \text{ with probability one.}$$



\end{theorem}

Given data collected from $M$ trajectories, we would like to choose the best $\bhypspace_M$ to maximize the accuracy of the estimator. Theorem \ref{firstorder:maintheorem} highlights two competing issues. On one hand, we would like the hypothesis space $\bhypspace_M$ to be large so that the bias $\inf_{\bintkernelvar \in \bhypspace_M}\|\bintkernelvar(\cdot)\cdot-\bintkernel(\cdot)\cdot\|^2_{\bL^2(\brhoL)}$ (or $\inf_{\bintkernelvar \in \bhypspace}\|\bintkernelvar-\bintkernel\|^2_{\infty}$) is small. On the other hand,  we would like $\bhypspace_M$  to be small so that the covering number $\mathcal{N}(\bhypspace_M,\frac{\epsilon}{48\spaceM_0R^2K^3})$ is small. This is the classical bias-variance trade-off in statistical estimation. As in the standard regression theory, the rate of convergence depends on the regularity condition of the true interaction kernels and the hypothesis space, as is demonstrated in the following proposition. We show that the optimal min-max rate of convergence for $1$-dimensional regression is achieved by choosing suitable hypothesis spaces dependent on the sample size.

\begin{theorem}  Let $\widehat\bintkernel_{M,\bhypspace_M}$ be  a minimizer of the empirical error functional defined in \eqref{firstorder:empircalerrorfunctional}  over  the hypothesis space $\bhypspace_M$. \\
(a)  If we choose $\bhypspace_M\equiv\mbf{\mathcal{K}}_{R, \spaceM}$, assume that the coercivity condition holds true on $\mbf{\mathcal{K}}_{R, \spaceM}$,  then there exists a constant $C=C(K,\spaceM,R)$ such that 
 \[ \mathbb{E}_{\mu_0}[\| \widehat\bintkernel_{M,\bhypspace_M}(\cdot)\cdot-\bintkernel(\cdot)\cdot\|_{\bL^2(\brhoL)} ]\leq \frac{C}{c_{L,N,\bhypspace_M}} M^{-\frac{1}{4}}.\]
(b) Assume that $\{\mbf{\mathcal{L}}_n\}_{n=1}^{\infty}$  is a sequence of linear subspaces of $\bL^{\infty}([0,R])$,  such that 
 \begin{align}\label{assumptions}
\text{dim}({}\mbf{\mathcal{L}}_n) \leq c_0K^2n\,,\quad \inf_{\bintkernelvar \in \mbf{\mathcal{L}}_n}\|\bintkernelvar-\bintkernel\|_{\infty}  \leq c_1 n^{-s}
\end{align}
for some  constants $c_0,c_1>0, {s \geq 1}$.  {Such a sequence exists, for example,  when $\bintkernel  \in C^{k,\alpha}$ with $s=k+\alpha$,  it is approximated by $\mbf{\mathcal{L}}_n$: piecewise polynomials of degree at least $\lfloor s-1\rfloor$, defined on $n$ uniform subintervals of $[0,R]$.}  Suppose the coercivity condition holds true on the set $\mbf{\mathcal{L}}:=\cup_n\mbf{\mathcal{L}}_n$.  
 Define $\mbf{\mathcal{B}}_n$ to be the central ball of $\mbf{\mathcal{L}}_{n}$ with the radius $(c_1+S)$.  Then by choosing $\bhypspace_M=\mbf{\mathcal{B}}_{n(M)}$, with $n(M)\asymp (\frac{M}{\log M})^{\frac{1}{2s+1}}$, there exists a constant $C=C(K, S,R, c_0,c_1)$ such that
\begin{align}\label{rate}
 \mathbb{E}_{\mu_0}[\| \widehat\bintkernel_{M,\bhypspace_{M}}(\cdot)\cdot-\bintkernel(\cdot)\cdot\|_{\bL^2(\brhoL)} ] \leq \frac{C}{c_{L,N,\mbf{\mathcal{L}}}} \left(\frac{\log M}{M}\right)^{\frac{s}{2s+1}}\,. \end{align}
\label{t:firstordersystem:thm_optRate}
\end{theorem}

{We remark that $C$ increases (polynomially) with $K$, the number of agent types, consistently with the expectation that the multi-type estimation problem is harder than the single-type problem. We do not expect, however, the dependency of $C$ on $K$ to be sharp. A tighter bound would take into account, for example, how similar the interaction kernels between different types are. This is an interesting direction of future research.}

\begin{proof}
For part (i), denote $\bhypspace=\mbf{\mathcal{K}}_{R,S}$, and recall that for $\epsilon>0$, the covering number of $\bhypspace$ (with respect to the $\infty$-norm) satisfies $$\mathcal{N}(\bhypspace, \epsilon) \leq e^{C_1K^2 \epsilon^{-1}}$$ where $C_1$ is an absolute constant (see e.g. \cite[Proposition 6]{CS02}), and  $ \inf_{\bintkernelvar \in \bhypspace}\|\bintkernelvar-\bintkernel\|^2_{\infty}=0$. Then estimate \eqref{firstorder:mainestimate1} gives 
\begin{align}\label{est1}
P_{\mu_0}\{\| \widehat \bintkernel_{L, M, \bhypspace}(\cdot)\cdot-\bintkernel(\cdot)\cdot\|_{\bL^2(\brhoL)}  > \epsilon\} &\leq  \mathcal{N}(\bhypspace, \frac{\epsilon^2c_{L,N,\bhypspace}}{48\spaceM R^2K^4}) e^{\frac{-c_{L,N,\bhypspace}^2M\epsilon^2 }{1152\spaceM^2K^4}} \nonumber\\& \leq e^{ \frac{48\spaceM R^2 K^6 C_1}{c_{L,N,\bhypspace}}  \epsilon^{-2}- \frac{c_{L,N,\bhypspace}^2M\epsilon^2 }{1152\spaceM^2R^2K^6}} .
\end{align}
Define $g(\epsilon) := \frac{48\spaceM R^2K^6 C_1}{c_{L,N,\bhypspace}}  \epsilon^{-2}- \frac{c_{L,N,\bhypspace}^2M\epsilon^2 }{2304\spaceM^2R^2K^4}$, which is a  decreasing function of $\epsilon$. By direct calculation, $g(\epsilon)=0$ if $\epsilon=\epsilon_{M}=(\frac{C_2}{M})^{\frac{1}{4}}$, where $C_2=\frac{11092\spaceM^3K^{10}C_1}{c_{L,N,\bhypspace}^3}$.  Thus, we obtain
$$P_{\mu_0}\{ \| \widehat\bintkernel_{M,\bhypspace}(\cdot)\cdot-\bintkernel(\cdot)\cdot\|_{\bL^2( \brhoL)}  > \epsilon\}\leq  \begin{cases} e^{\frac{-c_{L,N,\bhypspace}^2M\epsilon^2 }{2304\spaceM^2 K^4}}, &  \epsilon \geq \epsilon_M \\ 1,  &  \epsilon \leq \epsilon_M \end{cases}$$

Integrating over $\epsilon\in(0,+\infty)$ gives
$$ \mathbb{E}_{\mu_0}[\| \widehat\bintkernel_{M, \bhypspace}(\cdot)\cdot-\bintkernel(\cdot)\cdot\|_{\bL^2(\brhoL)}]  \le \frac{C_3}{c_{L,N,\bhypspace}} M^{-\frac{1}{4}}, $$ where $C_3 =C(K, \spaceM, R)$ is an absolute constant only depends on $K, \spaceM$ and $R$. 

For part (ii), recall that for $\epsilon>0$, the covering number of $\mbf{\mathcal{B}}_n$  by $\epsilon$-balls is bounded above by $ \left({4(c_1+\spaceM)}/{\epsilon}\right)^{c_0K^2n}$(see \cite[Proposition 5]{CS02}). 
From estimate \eqref{firstorder:mainestimate1}, we obtain 
\begin{equation}
 \begin{aligned}
P_{\mu_0}\{ \| \widehat\bintkernel_{M,\mbf{\mathcal{B}}_n}(\cdot)\cdot-\bintkernel(\cdot)\cdot\|_{\bL^2(\brhoL)}\geq \epsilon+ c_2n^{-s}  \} &\leq (\frac{c_3}{\epsilon^2})^{c_0K^2n} e^{-c_4M\epsilon^2}\\&=e^{c_0K^2n\log(c_3)+2c_0K^2n|\log(\epsilon)|-c_4M\epsilon^2}, 
\end{aligned}
\label{est2}
\end{equation}
where $c_2=\sqrt{\frac{1}{c_{L,N,\cup_n\mbf{\mathcal{L}}_n}}}c_1$, $c_3=\frac{192 (\spaceM+c_1)^2R^2K^4}{c_{L,N,\cup_n\mbf{\mathcal{L}}_n}}$, and $c_4=\frac{c_{L,N,\cup_n\mbf{\mathcal{L}}_n}^2}{1152(\spaceM+c_1)^2R^2K^6}$ are absolute constants independent of $M$.  Define 
$$g(\epsilon) :=c_0nK^2\log(c_3)+2c_0nK^2|\log(\epsilon)|-\frac{c_4}{2}M\epsilon^2.$$ Set $n_*=(\frac{M}{\log M})^{\frac{1}{2s+1}}$, 
and consider 
$g(c \epsilon_M) $ with $\epsilon_M=(\frac{\log M}{M})^{\frac{s}{2s+1}}=n_{*}^{-s}$ as a function of $c$.  By calculation, $g(c \epsilon_M) $  is a decreasing function of $c$. We have $\lim_{c \rightarrow 0^+}g(c \epsilon_M)=\infty$ and 
$\lim_{c \rightarrow \infty}g(c \epsilon_M)=-\infty$. Therefore,  there exists a constant $c_5$ depending only on $K,c_0,c_3,c_4$ such that $g(c_5 \epsilon_M)=0$.   This gives 
$$P_{\mu_0}\{ \| \widehat\bintkernel_{\infty,\mbf{\mathcal{B}}_{n_*}}(\cdot)\cdot-\bintkernel(\cdot)\cdot\|_{\bL^2(\brhoL)}  > \epsilon\}\leq  \begin{cases} e^{\frac{-c_4}{2}M\epsilon^2}, &  \epsilon \geq  c_5\epsilon_M \\ 1,  &  \epsilon \leq c_5\epsilon_M \end{cases}.$$ Therefore, with $\bhypspace_M=\mbf{\mathcal{B}}_{n_*}$, 
$$ \mathbb{E}_{\mu_0}[\| \widehat\bintkernel_{M,\bhypspace_{M}}(\cdot)\cdot-\bintkernel(\cdot)\cdot\|_{\bL^2(\brhoL)}] \le \frac{c_6}{c_{L,N,\cup_n\bhypspace_n}} \left(\frac{\log M}{M}\right)^{\frac{s}{2s+1}}, $$ where $c_6$ is an absolute constant only depending on $K,\spaceM$, $c_0, c_1$. 
\end{proof}

{The  convergence rate $\frac{s}{2s+1}$ coincides with the convergence rate for $1$-dimensional regression, where one can observe directly noisy values of the target function at sample points drawn i.i.d from $\brhoL$, for the set of functions satisfying the approximation property \eqref{assumptions}. It is the optimal min-max rate for functions $C^{k,\alpha}$ with $s=k+\alpha$}. Obtaining this optimal rate is satisfactory, because we do not observe the values $\{\intkernel_{kk'}(\|\bx_{i'}^{(m)}(t_l) - \bx_i^{(m)}(t_l)\|)\}_{l=1, i, i' =1,\ m=1, k,k'=1}^{L,N,N,M,K,K}$ from the observations of the trajectories of the states. The only randomness is in the $M$ samples, via the random initial condition. It is perhaps a shortcoming of our result that there is no dependence on $L$ nor $N$ in our upper bounds, especially since numerical examples in \cite{lu2019nonparametric} suggest that the error does decrease with $L$.
In the case of $K=1$ and $N$ large, the results in \cite{BFHM17} suggest rates no better than $N^{-\frac1d}$, i.e. they are cursed by the dimensionality of the space in which the agents move, albeit recent work by some of the authors of \cite{BFHM17} suggest better results, with rates similar to ours but in the case of $N\rightarrow+\infty$ may be possible (personal communication).

\subsection{Accuracy of  trajectory prediction}
Once an estimator $\widehat\bintkernel_{L, M,\bhypspace}$ is obtained, a natural question is the accuracy of trajectory prediction based on the estimated kernel. The next proposition shows that the error in prediction is (i) bounded trajectory-wise by a continuous-time version of the error functional, and (ii) bounded {on average} by the $\bL^2(\brhoT)$ error of the estimator. This further validates the effectiveness of our error functional and $\bL^2(\brhoT)$-metric to assess the quality of the estimator.


\begin{proposition}\label{firstordersystem:Trajdiff}
Suppose $\widehat\bintkernel \in \mbf{\mathcal{K}}_{R,S}$. Denote by $\widehat{\bX}(t)$ and $\bX(t)$ the solutions of the systems with kernels $\widehat\bintkernel=(\widehat\intkernel_{kk'})_{k,k'=1}^{K,K}$ and $\bintkernel$ respectively, starting from the same initial condition. Then for each trajectory we have  
 \begin{align*}
\sup_{t\in[0,T]}\!\! \|\widehat{\bX}(t)- \bX(t)\|_{\mathcal{S}}^2
\leq   2T\exp(8T^2K^2S^2)\int_0^T  \|\dot\bX(s)-\rhsfo_{\widehat\bintkernel}( \bX(s)) \|_{\mathcal{S}}^2 dt\,,
\end{align*}and on average with respect to the initial distribution $\mu_0$
\[
\E_{\probIC}[\sup_{t\in[0,T]} \|\widehat{\bX}(t)- \bX(t)\|_{\mathcal{S}}^2] \leq  2T^2K^2\exp(8T^2K^2S^2) \|\widehat{\bintkernel}(\cdot)\cdot-\bintkernel(\cdot)\cdot\|_{\bL^2(\brhoT)}^2,
\]
 where the measure $\brhoT$ is defined by \eqref{e:rhoT}.
\end{proposition}

\begin{proof} Recall that $ \br_{ii'}:= \bx_{i'}- \bx_i$ and $\widehat \br_{ii'} := \widehat \bx_{i'}-\widehat \bx_i$. To simplify the notation, we introduce the function $F_{[\intkernelvar]}(\bz):=\intkernelvar(\|\bz\|)\bz$, defined on $\mathbb{R}^d$  for $\intkernelvar \in L^{\infty}([0,R])$. Since $\widehat\bintkernel:=(\widehat \intkernel_{kk'})_{k,k'=1,1}^{K,K} \in \mbf{\mathcal{K}}_{R,S}$, we obtain $\text{Lip}[F_{[\widehat\intkernel_{kk'}]}]\leq S$ for each pair $(k,k')$.  
For every $t \in [0,T]$, we have 
\begin{align*}
\|\bX(t)-\widehat \bX(t)\|_{\mathcal{S}}^2 
&=\sum_{j=1}^{K}\sum_{i \in C_j}  \frac{1}{N_j}\left\| \int_{0}^{t} (\dot{\bx}_i(s)-\dot{\widehat {\bx}}_i(s)) ds\right\|^2\!\!\! 
\leq  t\sum_{j=1}^{K}\sum_{i \in C_j} \frac{1}{N_j}  \int_{0}^{t} \left\| \dot{\bx}_i(s)-\dot{\widehat {\bx}}_i(s)\right\|^2\!\!\!ds \\
& \leq   2T\int_0^t  \left\|\dot{\bX}(s) - \rhsfo_{\widehat\bintkernel}(\bX(s))\right\|_{\mathcal{S}}^2 ds + 2T {\int_{0}^{s}}I(s)ds,
\end{align*}
where $$I(s)= \left\|  \sum_{j'=1}^{K}\sum_{i' \in C_{j'}} \frac{1}{N_{j'}}  \left(F_{\widehat\intkernel_{jj'}}(\br_{ii'}(s)) -F_{\widehat\intkernel_{jj'}}(\widehat\br_{ii'}(s))\right)\right\|_{\mathcal{S}}^2.$$
By the triangle inequality, we have $I(s) \leq I_1(s)+I_2(s)$, where
\begin{align*}
I_1(s)&=\left\|\sum_{j'=1}^{K}\sum_{i' \in C_{j'}} \frac{1}{N_{j'}}  \left(F_{\widehat\intkernel_{jj'}}(\br_{ii'}(s)) -F_{\widehat\intkernel_{jj'}}(\bx_{i}(s)-\hat{\bx}_{i'}(s))\right)\right\|_{\mathcal{S}}^2, \\
I_2(s)&=\left\|  \sum_{j'=1}^{K}\sum_{i' \in C_{j'}} \frac{1}{N_{j'}}  \left(F_{\widehat\intkernel_{jj'}}(\widehat \br_{ii'}(s)) -F_{\widehat\intkernel_{jj'}}(\bx_{i}(s)-\hat{\bx}_{i'}(s))\right)\right\|_{\mathcal{S}}^2.
\end{align*}

Estimating by Jensen or H\"older inequalities, we obtain 
\begin{align*}
I_1(s) &\leq \sum_{j=1}^{K}\sum_{i \in C_j} \frac{1}{N_j}  \left |  \sum_{j'=1}^{K}\sum_{i' \in C_{j'}} \frac{1}{N_{j'}}  \text{Lip}[F_{\widehat\intkernel_{jj'}}]\|\bx_{i'}(s)-\hat{\bx}_{i'}(s)\| \right |^2\\ &\leq  K \sum_{j=1}^{K}\sum_{i \in C_j} \frac{1}{N_j} \sum_{j'=1}^{K}  \sum_{i'\in C_j'} \frac{ \text{Lip}^2[F_{\widehat\intkernel_{jj'}}]}{N_{j'}} \|\bx_{i'}(s)-\hat{\bx}_{i'}(s)\|^2 \\&\leq  K(\sum_{j=1}^{K}\max_{j'}\text{Lip}^2[F_{\widehat\intkernel_{jj'}}])\|\bX(s)-\hat \bX(s)\|^2_{\mathcal{S}}\\&\leq K^2S^2\|\bX(s)-\hat \bX(s)\|^2_{\mathcal{S}}.
\end{align*}
Similarly, 
\begin{align*}
I_2(s)& \leq K\max_{j}(\sum_{j'=1}^{K}\text{Lip}^2[F_{\widehat\intkernel_{jj'}}])\|\bX(s)-\hat \bX(s)\|^2_{\mathcal{S}}\leq K^2S^2\|\bX(s)-\hat \bX(s)\|^2_{\mathcal{S}}. 
\end{align*}

Combining above inequalities with Gronwall's inequality yields the first inequality in the proposition. The second inequality follows by combining the above with Proposition  \ref{firstordersystem:empiricalerrorfunctional}, which implies
 \begin{align*}
\frac{1}{T}\int_{0}^{T}  \E \left\| \dot{\bX}^{}(s) - \rhsfo_{\widehat\bintkernel}(\bX(s)) \right\|_{\mathcal{S}}^2ds < K^2 \norm{\widehat\bintkernel(\cdot)\cdot - \bintkernel(\cdot)\cdot}^2_{L^2(\brhoT)}.
\end{align*}

\end{proof}

\section{Algorithm}
\label{sec_finitedifference}
Recall that our goal is to learn the interaction kernels $\bintkernel$ from the observational data $$\{\bxm_i(t_l),\dot{\bx}^{(m)}_i(t_l)\}_{i = 1, m = 1, l = 1}^{N, M, L},$$  consisting of the positions and velocities of agents observed at equidistant time instances $0 = t_1 < t_2 < \cdots < t_{L} = T$ with $M$ i.i.d initial conditions drawn from a probability measure $\mu_0$ on $\mathbb{R}^{dN}$.  Our estimator $\blintkernel_{M,\bhypspace}$ is obtained by minimizing the empirical error functional 
\begin{align}\label{eq:errFnl}
\mbf{\mE}_{M}(\bintkernelvar)= \frac{1}{LM}\sum_{l = 1, m = 1, i = 1}^{L, M, N}\frac{1}{N_{\clof_i}}\norm{\dotbxm_i(t_l) - \sum_{i' =1}^N\frac{1}{N_{\clof_{i'}}}\intkernelvar_{\clof_i\clof_{i'}}(r_{ii'}^{(m)}(t_l))\mbf{r}_{ii'}^{(m)}(t_l)}^2,
\end{align}
over all possible $\bintkernelvar = \{\intkernelvar_{\idxcl\idxcl'}\}_{\idxcl,\idxcl'=1}^{\numcl}$ in a suitable hypothesis space $\mbf{\hypspace}=\bigoplus_{k,k'=1,1}^{K,K} \mathcal{H}_{kk'}$. 
In section \ref{learningtheory}, we analyzed the performance of estimators over compact convex subsets of $\bL^{\infty}([0,R])$.  However, to compute these estimators numerically, one has to solve a constrained quadratic minimization problem, which is computationally demanding. Fortunately, as in the standard nonparametric setting references, such a costly constrained optimization is unnecessary, and one can simply compute the minimizer by least-squares over the linear finite-dimensional hypothesis spaces because one can prove by standard truncation arguments that the learning theory is still applicable to the truncation of these estimators obtained by the unconstrained optimization (see chapter 11 in \cite{Gyorfi06}). In the following, we solve the minimization problem by choosing a suitable set of basis functions for $\bintkernelvar$ and compute the minimizer by regularized (i.e. constrained to $\bintkernelvar$) least squares in a fashion that is amenable to efficient parallel implementation.

\subsection{Numerical implementation}\label{numericalimplementation}
\subsubsection{Choice of the hypothesis spaces and their basis}  We use local basis functions to capture  local features of the interaction kernels, such as the sharp jumps: each hypothesis space $\mathcal{H}_{kk'}$ is an $n_{kk'}$-dimensional space spanned by $\{\psi_{kk',p}\}_{p=1}^{n_{kk'}} $, a set of piecewise polynomial functions of degree $s$, with $s$ being the order of local differentiability of the true kernel. The dimension $n_{kk'}$ is chosen to be a scalar multiple of  the optimal dimension $n_*=(\frac{M}{\log M})^{\frac{1}{2s+1}}$ of the hypothesis space, as in Theorem \ref{t:firstordersystem:thm_optRate}.  For simplicity, we set these piecewise polynomials to be supported on a uniform partition of the interval $[0, R]$, where the radius $R$ is the largest observed pairwise distance.  

\subsubsection{Velocity data of agents}
When only the position data are available, the velocity data may be approximated numerically. In our numerical experiments,  $\dotbxm_i(t_l)$ is approximated by backward differences:
\[
\dotbxm_i(t_l) \approx \Delta{\bx}_i^{m}(t_l) = \frac{\bxm_i(t_{l+1}) - \bxm_i(t_{l })}{t_{l+1} - t_{l }}, \quad \text{for $1 \le l \le L$}.
\] 
The error of the backward difference approximation is of order $O(T/L)$, leading to a comparable bias in the estimator, as we shall see in \ref{sec_condN}. Hereafter we assume that $T/L$ is sufficiently small so that the error is negligible relative to the statistical error. 

\subsubsection{The numerical implementation}
With these basis functions, denoting $\intkernelvar_{\idxcl\idxcl'}(r) =\sum_{p = 1}^{n_{\idxcl\idxcl'}}a_{\idxcl\idxcl',p}\psi_{\idxcl\idxcl', p}(r) \in \mathcal{H}_{kk'}$  for some constant coefficients $(a_{\idxcl\idxcl',p})_{p=1}^{n_{\idxcl\idxcl'}}$, we can rewrite the error functional in \eqref{eq:errFnl} as 
 \begin{align*}
 \, \mbf{\mE}_{M}(\bintkernelvar) 
  =&\frac{1}{LM}\sum_{l = 1, m = 1, i = 1}^{L, M, N}\frac{1}{N_{\clof_i}} \bigg\| \sum_{i'=1 }^N \frac{1}{N_{\clof_{i'}}}(\sum_{p=1}^{n_{\clof_i\clof_{i'}}}a_{\clof_i\clof_{i'}, p} \psi_{\clof_i\clof_{i'}, p})(r_{ii'}^{(m)}(t_l))\mbf{r}_{i, i'}^{(m)}(t_l)- \dot{\bx}^{(m)}_i(t_l)\bigg\|^2,
 \end{align*} 
 {which is a quadratic functional with respect to the coefficient vector $\vec{a}$ of $\bintkernelvar$: }
  \begin{align*}
 \frac{1}{LM} \sum_{m=1}^{M} \|{\Psi}_{\bhypspace}^{(m)}\vec{a}-\vec{d}_L^{(m)}\|_2^2.
  \end{align*}
  Here the vectors $\vec{a}$ and $\vec{d}_L^{(m)}$ are
  \[
\vec{a} = \begin{pmatrix} a_{11, 1} \\ \vdots \\ a_{11, n_{11}} \\ \vdots \\ a_{\numcl\numcl, 1} \\ \vdots \\ a_{\numcl\numcl, n_{\numcl\numcl}} \end{pmatrix}\in\R^{{\tiny{\sum_{\idxcl, \idxcl' = 1}^\numcl n_{\idxcl \idxcl'}}}}, \quad \vec{d}_L^{(m)}=\begin{pmatrix} (1/N_{\clof_1})^{1/2}\dot{\bx}^{(m)}_1(t_1)\\ \vdots \\  (1/N_{\clof_N})^{1/2}\dot{\bx}^{(m)}_N(t_1)\\ \vdots \\  (1/N_{\clof_1})^{1/2}\dot{\bx}_1^{(m)}(t_L)\\ \vdots \\ (1/N_{\clof_N})^{1/2}\dot{\bx}^{(m)}_N(t_L)\end{pmatrix}\in \real^{LNd},
\]and  the learning matrix $\Psi_{\bhypspace}^{(m)} \in \mathbb{R}^{LNd\times \sum_{\idxcl, \idxcl' = 1}^\numcl n_{\idxcl \idxcl'}}$ is defined as follows:  partition the columns into $\numcl^2$ regions with each region indexed by the pair $(\idxcl,\idxcl')$, with $\idxcl,\idxcl' = 1, \cdots, \numcl$; the usual lexicographic partial ordering is placed on these pairs, namely
$(\idxcl_1,\idxcl'_1) < (\idxcl_2,\idxcl'_2)$ iff $\idxcl_1 < \idxcl_2$ or $\idxcl_1 = \idxcl_2$ and $\idxcl'_1 < \idxcl'_2$; then in the region of the columns of $\Psi_{\bhypspace}^{(m)}$ corresponding to $(\idxcl,\idxcl')$, the entries of the learning matrix are
\[
\Psi_{\bhypspace}^{(m)}(li, \tilde{n}_{kk'} + p) =  \sqrt{\frac{1}{N_{\clof_i}}}\sum_{i' \in \cl_{\idxcl'}}\frac{1}{N_{\idxcl'}}\psi_{\idxcl\idxcl', p}(r_{ii'}^{(m)}(t_l))\mbf{r}_{ii'}^{(m)}(t_l) \in \R^{d},
\]
for $i \in \cl_{\idxcl}$ and $1 \le l \le L$, and $\tilde{n}_{kk'} = \sum_{(\idxcl_1,\idxcl'_1) < (\idxcl,\idxcl')}n_{\idxcl_1\idxcl'_1}$. 

Then we solve the least squares problem
$ \argmin{\vec{a}} \frac{1}{LM} \sum_{m=1}^{M} \|{\Psi}_{\bhypspace}^{(m)}\vec{a}-\vec{d}_L^{(m)}\|_2^2$
by the normal equation
\begin{equation}
\underbrace{\frac{1}{M} \sum_{m=1}^{M}  A_{\bhypspace}^{(m)} }_{A_{M,\bhypspace}}\vec{a} =\frac{1}{M} \sum_{m=1}^{M}  b_{\bhypspace}^{(m)} \,,
\label{e:ALM}
\end{equation}
where the trajectory-wise regression matrices are 
\[A_{\bhypspace}^{(m)} := \frac{1}{LN}(\Psi^{(m)}_{\bhypspace})^T\Psi^{(m)}_{\bhypspace},\quad  b_{\bhypspace}^{(m)} := \frac{1}{LN}(\Psi^{(m)}_{\bhypspace})^T\vec{d}_{\bhypspace}^{(m)}.\]

Note that  the matrices $A_{\bhypspace}^{(m)}$ and $b_{\bhypspace}^{(m)}$ for different trajectories may be computed in parallel.  The size of the matrices $\smash{A_{\bhypspace}^{(m)}}$ is $(\sum_{kk'}n_{kk'})\times (\sum_{kk'}n_{kk'})$, and there is no need to read and store all the data at once, thereby dramatically reducing memory usage. These are the main reasons why we solve the normal equations instead of the linear system directly associated to the least  squares problem; the disadvantage of this approach is that the condition number of $\smash{A_{\bhypspace}^{(m)}}$ is the square of the condition number of the matrix of the linear system, and in situations where the latter is large, passing to the normal equations is not advised. We summarize the learning algorithm in the following table. 

 \begin{algorithm}[H]
  \caption{Learning interaction kernels from observations}\label{algo:prony}
  \begin{algorithmic}[1]
    \REQUIRE Data $\{\bX^{(m)}(t_l)\}_{l=1,m=1}^{L+1,M}$  and the interval $[0, R]$\myfootnotemark.

             \STATE Use the backward finite difference method to compute the velocity data (if not given) to obtain $\{\bX^{(m)}(t_l),\dot{\bX}^{(m)}(t_l)\}_{l=1,m=1}^{L,M}$

        \STATE  For each pair $(k,k')$, partition the interval $[0,R]$ into $\frac{n_{kk'}}{s+1}$ uniform sub-intervals and construct the piecewise polynomial functions of degree $s$, $\{\psi_{kk',p}\}_{p=1}^{n_{kk'}}$, over the partition.  
     \STATE Assemble (in parallel) the normal equation  as in \eqref{e:ALM}

$$ {A}_{M,\bhypspace}\vec{a}=\vec{b}_{M,\bhypspace}$$

\STATE  Solve for $\vec{a}$
\begin{align}\label{leastsquarematrix}
\vec{a}={A}_{M,\bhypspace}^{\dagger}\vec{b}_{M,\hypspace},
\end{align} where ${A}_{M,\bhypspace}^{\dagger}$ is the pseudo-inverse of ${A}_{M,\bhypspace}=\trans{\Psi_{M,\bhypspace}}\Psi_{M,\bhypspace}$ and $\vec{b}_{M,\bhypspace}=\trans{\Psi_{M,\bhypspace}}\vec{d}_{M}.$

\ENSURE  The estimator $\blintkernel=(\sum_{p=1}^{n_{kk'}}a_{kk',p}\psi_{kk',p})_{k, k'=1}^{K}$. 

\end{algorithmic}
\end{algorithm}
  \myfootnotetext{$R$ is assumed know; if not, we could estimate it using $R_{\max,M}:=\max_{i,i',l,m}\|\bx_{i}^{(m)}(t_l)-\bx_{i'}^{(m)}(t_l)\|$.
 }

{The total computational cost of constructing estimators, given $P$ CPU's, is $O(ML\frac{N^2d}Pn^2+n^3)$.  This becomes  $O((L\frac{N^2d}P+C)M^{1+\frac{2}{2s+1}})$ when $n$ is chosen optimally according to  our Theorem and $\phi$ is at least Lipschitz (corresponding to the index of regularity $s\geq1$ in the theorem.}

\subsection{Well-conditioning from coercivity} \label{sec_condN}

We could choose different bases of $\bhypspace$ to construct the regression matrix $A_{M,\bhypspace}$  in \eqref{e:ALM}; although the minimizer in $\bhypspace$ is of course independent of the choice of basis, the condition number of $A_{L, M,\bhypspace}$ does depend on the choice of basis, affecting the numerical performance. The question is, how do we choose the basis functions so that the matrix $A_{L, M,\bhypspace}$ is well-conditioned?  
We show that if the basis is orthonormal in $\bL^2(\brhoL)$, the coercivity constant provides a lower bound on the smallest singular value of $A_{L, M,\bhypspace}$, therefore providing control on the condition number of $A_{L, M,\bhypspace}$.

To simplify the notation, we introduce a bilinear functional $\dbinnerp {\cdot, \cdot}$ on $\bhypspace \times \bhypspace$ 
\begin{equation}
\label{eq:bilinearFn}
\dbinnerp {\bintkernelvar_1, \bintkernelvar_2}:=\frac{1}{L}\sum_{l,i=1}^{L,N}\frac{1}{N_{\clof_{i}}} \mathbb{E}_{\mu_0}\bigg[\bigg\langle\sum_{i'=1}^{N} \frac{1}{N_{\clof_{i'}}}\intkernelvar_{1,\clof_i\clof_{i'}}(r_{ii'}(t))\mbf{r}_{ii'}(t)
, \sum_{i'=1}^{N}\frac{1}{N_{\clof_{i'}}}\intkernelvar_{2,\clof_i\clof_{i'}}(r_{ii'}(t))\mbf{r}_{ii'}(t) \bigg\rangle \bigg] 
\end{equation} for any $\bintkernelvar_1=(\intkernelvar_{1,kk'})_{k,k'}$, and 
$\bintkernelvar_2=(\intkernelvar_{2,kk'})_{k,k'} \in \bhypspace$. 
For each pair $(k,k')$ with $1\leq k,k'\leq K$, let $\{\psi_{kk',1},\cdots, \psi_{kk',n_{kk'}}\}$ be a basis of $\hypspace_{{kk'}}\subset L^\infty([0,R])$ such that \begin{equation}\label{onb}\langle \psi_{kk',p}(\cdot)\cdot,\psi_{kk',p'}(\cdot)\cdot \rangle_{L^2{(\rho_T^{L,kk'})}}=\delta_{p,p'}, \|\psi_{kk',p}\|_{\infty} \leq S_0. \end{equation} 
For each $\psi_{kk',n_{kk'}} \in\mathcal{H}_{kk'}$, we denote by $\mbf{\psi}_{kk',n_{kk'}}$ its canonical embedding in $\bhypspace$.  Adopting the usual lexicographic partial ordering on pairs $(k,k')$, we reorder the basis $\{\mbf{\psi}_{kk',1},\cdots, \mbf{\psi}_{kk',n_{kk'}}\}$ to be $\{\mbf{\psi}_{1+\tilde{n}_{kk'}},\cdots,\mbf{\psi}_{n_{kk'}+\tilde{n}_{kk'}}\}$, where $\tilde{n}_{kk'} = \sum_{(\idxcl_1,\idxcl'_1) < (\idxcl,\idxcl')}n_{\idxcl_1\idxcl'_1}$. Set $n=\sum_{k,k'}n_{kk'}=dim(\bhypspace)$; then for any function $\bintkernelvar \in \bhypspace$, we can write $\bintkernelvar=\sum_{p=1}^{n} a_{p}\mbf{\psi}_{p}$. We have:

\begin{proposition}
\label{firstordersystem:coercivityconstant}
Define $A_{\infty,\bhypspace}=\big(\dbinnerp{\mbf{\psi}_{p},\mbf{\psi}_{p'}}\big)_{p, p'} \in \mathbb{R}^{n \times n}$. Then the coercivity constant  of $\bhypspace= \mathrm{span}\{ \mbf{\psi}_{p}\}_{p=1}^n$ is the smallest singular value of $A_{\infty,\bhypspace}$, i.e.~
\begin{align*}
\sigma_{\min}(A_{\infty,\bhypspace}) =c_{L,N,\bhypspace}\,
\end{align*}
with $c_{L,N,\bhypspace}$ defined in \eqref{firstorder:gencoer}. Moreover,  
for large $M$, the smallest singular value of $A_{M,\bhypspace}$  
\begin{align*}
\sigma_{\min}(A_{M,\bhypspace}) \geq  0.9 c_{L,N,\bhypspace}
\end{align*} 
with  probability at least $1-2n\exp(-\frac{c_{L,N,\bhypspace}^2M}{200n^2c_1^2+\frac{10c_{L,N,\bhypspace}c_1}{3}n} )$ with $c_1=K^4R^2S_0^2+1$.
\label{p:sigmaminAL}
\end{proposition}

\begin{proof}
Due to properties \eqref{onb},  the set of functions $\{\bintkernelvar_{p}\}_{p=1}^{n} \subset \bhypspace$ is   orthonormal  in the sense 
$\langle \mbf{\psi}_{p}(\cdot)\cdot, \mbf{\psi}_{p'}(\cdot)\cdot\rangle_{L^2(\brhoL)}=\delta_{pp'}$. Then for any $\bintkernelvar = \sum_{p=1}^n a_p\bintkernelvar_{p} \in \bhypspace$,
\begin{align*}
\vec{a}^T A_{\infty,\bhypspace} \vec{a}& =\dbinnerp {\sum_{p=1}^{n} a_{p}\mbf{\psi}_{p} , \sum_{p=1}^{n} a_{p} \mbf{\psi}_{p}} =  \!\!\frac{1}{L}\sum_{l=1}^{L}\E_{\mu_0} \bigg[ \big\|  \rhsfo_{\sum_{p=1}^{n} a_{p} \mbf{\psi}_{p}}(\bX(t_l)) \big\|_{\mathcal{S}}^2\bigg]
 \\&\geq  \sigma_{\min}(A_{M,\bhypspace})\|\vec{a} \|^2 = \sigma_{\min}(A_{M,\bhypspace}) \big\| \bintkernelvar(\cdot)\cdot\|_{\mbf{L}^2(\brhoL)}^2\,.
\end{align*} 
Thus, the coercivity constant $c_{L,N,\bhypspace}$ is $\sigma_{\min}(A_{M,\bhypspace}),$ {since this lower bound is in fact realized by the eigenvector corresponding to the singular value $\sigma_{\min}(A_{M,\bhypspace})$}.

From the definition of $A_{M,\bhypspace}$ in \eqref{e:ALM} and the fact that $A_{\infty,\bhypspace}=\mathbb{E}[A_{\bhypspace}^{(m)}]$,
we have $\lim_{M\rightarrow\infty} A_{M,\bhypspace}=A_{\infty,\bhypspace}$ by the Law of Large Numbers. Using the matrix Bernstein inequality (Theorem 6.1.1 in \cite{tropp2015introduction}) to control the smallest singular value of $A_{M,\bhypspace}$, we obtain that 
 $\sigma_{min}(A_{M,\bhypspace})$ is  bounded below by $0.9 c_{L,N,\bhypspace}$ with the desired probability. 
\end{proof}

Proposition \ref{p:sigmaminAL} also implies that the  $O(T/L)$ error in the finite difference approximations  leads to a bias of order $O(T/L)$ in the estimator (with high probability). This can be derived from equation \eqref{e:ALM}: the error in the finite difference approximation leads to a bias of order $O(T/L)$ on the vector $b_{\bhypspace}^{(m)} $, which is whence passed to the estimator linearly, as $\smash{A_{M,\bhypspace} ^{-1} \frac{1}{M} \sum_{m=1}^{M}  b_{\bhypspace}^{(m)}}$. With high probability, the bias is of the same order as the finite difference error since the smallest singular value of the regression matrix $A_{M,\bhypspace}$ is bounded below by the coercivity constant.

From Proposition \ref{p:sigmaminAL} we see that, for each hypothesis space $\bhypspace_{kk'}$, it is important to choose a basis that is well-conditioned in $\bL^2(\brhoL)$,  instead of in $\bL^\infty([0,R])$, for otherwise the matrix $A_{\infty,\bhypspace}$ in the normal equations may be ill-conditioned or even singular.  This issue can deteriorate in practice when the unknown $\smash{\rhoL}$ is replaced by the empirical measure $\smash{\brho_{T}^{L, M}}$. It is therefore advisable to either use piecewise polynomials on a partition of the support of $\smash{\brho_T^{L, M}}$ or use the pseudo-inverse to avoid the artificial singularity.

\section{Numerical  examples}
\label{main:numericalexamples}
We report in this section the learning results of three widely used examples of first-order interacting agent systems:  opinion dynamics from social sciences, predator-swarm dynamics from Biology and a heterogeneous particle system inspired by particle Physics. Numerical results demonstrate that our learning algorithm can produce an accurate estimation of the true interaction kernels from observations made in a very short time, and can predict the dynamics,   and even collective behaviour of agents,  in a larger time interval. We also demonstrate numerically that as the number of observed trajectories $M$ increases, the errors in the estimation of the interaction kernel and in the predicted trajectories decay at rates agreeing with the theory in Section \ref{learningtheory}.

\subsection{Numerical setup}
\label{numericalexamples}
We begin by specifying in detail the setup for the numerical simulations in the examples that follow.
We use a large number $\smash{M_{\rhoL}}$ of independent trajectories, not to be used elsewhere, to obtain an accurate approximation of the unknown probability measure $\smash{\rhoL}$ in \eqref{e:rhoL}. In what follows, to keep the notation from becoming cumbersome, we denote by $\rhoL$ this empirical approximation. We run the dynamics over the observation time $[t_1, t_{L}]$ with $M$ different initial conditions (drawn from the dynamics-specific probability measure $\probIC$), and the observations consist of the state vector, with no velocity information, at $L$ equidistant time samples in the time interval $[t_1, t_{L}]$.  All ODE systems are evolved using \textrm{ode$15$s} in MATLAB\textsuperscript{\textregistered} with a relative tolerance at $10^{-5}$ and absolute tolerance at $10^{-6}$. 
 
In the numerical experiments, we shall use piecewise constant or piecewise linear functions to estimate the interaction kernels and then use the estimators to predict the dynamics. In order to reduce the stiffness of the differential equations with estimated interaction kernels, we choose a fine grid to linearly interpolate the estimator (and exterpolate it with a constant). This results in Lipschitz continuous estimators.

We report the relative  $\smash{\bL^2(\rhoL)}$ error of our estimators. In the spirit of Theorem \ref{firstordersystem:Trajdiff}, we also report  the error on trajectories $\bX(t)$ and $\widehat\bX(t)$ generated by the system with the true interaction kernel and with the learned interaction kernel, respectively, on both the training time interval $[t_1,t_L]$ and on a future time interval $[t_L, t_f]$, with both the same initial conditions as those used for training, and on new initial conditions (sampled according to $\probIC$), where the max-in-time trajectory prediction error over time interval $[T_0, T_1]$ is defined as 
\begin{equation}\label{e:tm_norm}
\norm{\bX - \hat{\bX}}_{\text{TM($[T_0,T_1]$)}} = \sup_{t \in [T_0, T_1]}\norm{\bX(t) - \hat{\bX}(t)}_{\mathcal{S}}.
\end{equation} The trajectory error will be estimated using $M$ trajectories (we report the mean of the error).  We run a total of $10$ independent learning trials and compute the mean of the corresponding estimators, their errors, and the trajectory errors just discussed.

Finally, for each example, we also consider the case of noisy observations of the positions. With noise added in the position, the finite difference method used to estimate the velocities will amplify the noise: the error in the estimated velocity will scale as $\frac{\text{std(noise)}}{t_{l+1}-t_l}$ (\cite{wagner2015regularised}). This issue is treated in the topic of numerical differentiation of noisy data and several approaches have been developed, include the total variation regularization approach (\cite{chartrand2011numerical}) used in \cite{BPK2016, zhang2018robust}, high order finite difference method used in \cite{TranWardExactRecovery}, global and local smoothing techniques (see \cite{knowles2014methods,wagner2015regularised, cleveland1988locally}) used in \cite{wu2019numerical,kang2019ident}, but no technique robust enough to work across a wide variety of examples seems to exist.  We have tried to use these techniques in our examples: a combination of position data smoothing techniques and total variation regularization approach worked well in the opinion dynamics but no technique worked well in the Predator-Swarm Dynamics and particle dynamics, likely due to the large Lipchitz constant of the response functions in these two systems.  This is an important and difficult problem, and we leave it the development of robust techniques and their theoretical analysis to future work.  
Instead, to investigate the effect of noise in learning kernels from data,  we consider the simpler setting where we assume the velocities are observed, but both position and velocities are corrupted by noise. In the case of additive noise, the observations are $$\smash{\{(\bXm(t_l)+\eta_{1,l,m},\dot\bX^{(m)}(t_l))+\eta_{2,l,m}\}_{l=1,m=1}^{L,M}},$$ while in the  case of multiplicative noise they are $$\smash{\{(\bXm(t_l)\cdot(1+\eta_{1,l,m})}, \smash{\dot\bX^{(m)}(t_l))\cdot(1+\eta_{2,l,m})\}_{l=1,m=1}^{L,M}},$$ where in both cases $\eta_{1,l,m}$ and $\eta_{2,l,m}$ are i.i.d. samples from  Unif.$([-\sigma,\sigma])$.

For each example, 6 plots display and summarized our numerical results:
\begin{itemize}
\item in the first plot,  we compare the estimated interaction kernels (after smoothing) to the true interaction kernel(s) , for different values of $M$, with mean and standard deviation computed over a number of learning trials.  In the background we compare $\smash{\brhoL}$ (computed on $\smash{M_{\brhoL}}$ trajectories, as described above) and $\smash{\brho_T^{L, M}}$ (generated from the observed data consisting of $M$ trajectories). 
\item The second plot 
compares the true trajectories (evolved using the true interaction law(s)) and predicted trajectories (evolved using the learned interaction law(s) from a small number $M$ of trajectories) over two different set of initial conditions -- one taken from the training data, and one new, randomly generated from $\mu_0$. It also includes the comparison between the true trajectories and the trajectories generated with the learned interaction kernels, but for a different system with the number of agents $N_{\text{new}} = 4N$, over one set of randomly chosen initial conditions. 
\item The third plot displays the learning rate of mean trajectory error with respect to $M$,  both on the training time interval and the future time interval, in which we also compare them with the learning rate of estimated interaction kernels (those used to produce the predicted trajectories). 
\item  The fourth plot displays the coercivity constant over the hypothesis spaces used in the experiments (see Algorithm \ref{algo:coercivity}) and the learning rate of interaction kernels with and without the observation of true velocities.  To validate the applicability of the main Theorems, we report the relative $\bL^2(\brhoL)$ errors of the piecewise polynomial estimators (without smoothing), just as in the main Theorem \ref{t:firstordersystem:thm_optRate}. 

\item The fifth plot compares the estimators learned from the noisy observations with the true kernels, and their performance in trajectory prediction. 

\item The last plot shows the convergence rate of our estimators and their smoothed ones when the observations are contaminated by noises.

\end{itemize}
  
\begin{algorithm}[H]
  \caption{Approximate the coercivity constant over the hypothesis space $\bhypspace$ }\label{algo:coercivity}
  \begin{algorithmic}[1]
    \REQUIRE A set  of  basis functions $\{\bintkernelvar_{p}\}_{p=1}^{n} \subset \bhypspace$ 
\STATE   Generate the position data $\{\bX^{(m)}(t_l)\}_{l=1,m=1}^{L,M}$ with $M=10^5$.
\STATE Use the data in step 1 to compute an empirical measure  $\tilde{\mbf{\rho}}_T^L$.
\STATE Apply the Gram-Schmidt process  on the  $\{\bintkernelvar_{p}\}_{p=1}^{n}$ to get a new set of basis functions $\{\tilde\bintkernelvar_{p}\}_{p=1}^{n}$ that satisfy 
$$\langle \tilde\bintkernelvar_{p}(\cdot)\cdot, \tilde\bintkernelvar_{p'}(\cdot)\cdot\rangle_{\bL^2(\tilde{\mbf{\rho}}_T^L)}=\delta_{pp'}$$
     \STATE Use the data in step 1 and $\{\tilde\bintkernelvar_{p}\}_{p=1}^{n}$ to assemble the  matrix $A_{L,M,\bhypspace}$ (see equation \ref{e:ALM}) and compute its minimal eigenvalue. 

\ENSURE  $\sigma_{min}(A_{L,M,\bhypspace})$. 
\end{algorithmic}
\end{algorithm}

\subsection{Opinion dynamics}
One successful application of first order systems is  opinion dynamics in social sciences (see \cite{Krause2000SI, BHT2009,MT2014, BT2015, CKFL2005SI} and references therein). The interaction function $\intkernel$ models how the opinions of pairs of people influence each other. We consider a homogeneous case with  interaction kernel defined as 
\[
\intkernel(r) = \left\{
        \begin{array}{ll}
           0.4,    & \quad 0                          \le r < \frac{1}{\sqrt{2}}-0.05, \\
           -0.3\cos(10\pi(r-\frac{1}{\sqrt{2}}+0.05)) + 0.7, & \quad \frac{1}{\sqrt{2}}-0.05 \le r <  \frac{1}{\sqrt{2}}+0.05, \\
           1,    & \quad  \frac{1}{\sqrt{2}}+0.05    \le r <0.95,\\
             0.5  \cos(10\pi (r -0.95)) + 0.5, & \quad 0.95 \le r <  1.05\\
             0,&\quad 1.05 \le r
  \end{array}
    \right.
\] This kernel $\intkernel$ is compactly supported and Lipchitz continuous with Lipchitz constant  $5\pi$. It models heterophilious opinion interactions (see \cite{MT2014}) in a homogeneous group of people: each agent is more influenced by its further neighbours than by its closest neighbours.  It is shown in \cite{MT2014} that heterophilious dynamics enhances consensus: the opinions of agents merge into clusters, with the number of clusters significantly smaller than the number of agents, perhaps contradicting the intuition that would suggest that the tendency to bond more with those who are different rather than with those who are similar would break connections and prevent clusters of consensi.


Suppose the prior information is that $\intkernel$ is Lipschitz and compactly supported on $[0,10]$ (so $R=10$).  Let $\mathcal{H}_n$ be the function space consisting of piecewise constant functions on uniform partitions of [0,10] with $n$ intervals.  It is well-known in  approximation theory (see the survey \cite{devore1992wavelets}) that 
$\inf_{\varphi\in \mathcal{H}_n}\|\varphi-\intkernel\|_{\infty} \leq \text{Lip}[\intkernel]n^{-1}$, therefore the conditions in Theorem \ref{t:firstordersystem:thm_optRate} are satisfied with $s=1$. 
Our theory suggests that a choice of dimension $n$  proportional to $(\frac{M}{\log M})^{\frac{1}{3}}$ will yield an optimal learning rate $M^{-\frac{1}{3}}$ up to a logarithmic factor. We choose  $n=60(\frac{M}{\log M})^{\frac{1}{3}}$. 
Table \ref{t:OD_params} summarizes the system and learning parameters.


\begin{table}[H]
\centering
\begin{tabular}{| c | c | c | c | c | c | c |c|c|}
\hline 
 $d$ &$N$ &$M_{\rhoL}$& $L$      & $[t_1;t_L;t_f]$   & $\probIC$          & deg($\psi$) &$n$ \\ 
\hline 
 $1$ &10 &$10^5$ & $51$ & $[0;0.5;20]$ & $\mathcal{U}([0, 8])$ & 0&$ 60(\frac{M}{\log M})^{\frac{1}{3}}$ \\
\hline
\end{tabular}
\caption{\textmd{\footnotesize{(OD) Parameters for the system} }}
\label{t:OD_params}
\end{table}

Figure \ref{t:ODH1_kernel} shows that as the number of trajectories increases, we obtain more faithful approximations to the true interaction kernel, including near the locations with large derivatives and the support of $\phi$. The estimators also perform well near $0$, notwithstanding that information of $\intkernel(0)$ is lost due to the structure of the equations, that have terms of the form $\phi(0)\vec{0} = \vec{0}$.

\begin{figure}[!htbp]
\centering     

\subfigure{\label{figOD:1}\includegraphics[width=0.48\textwidth]{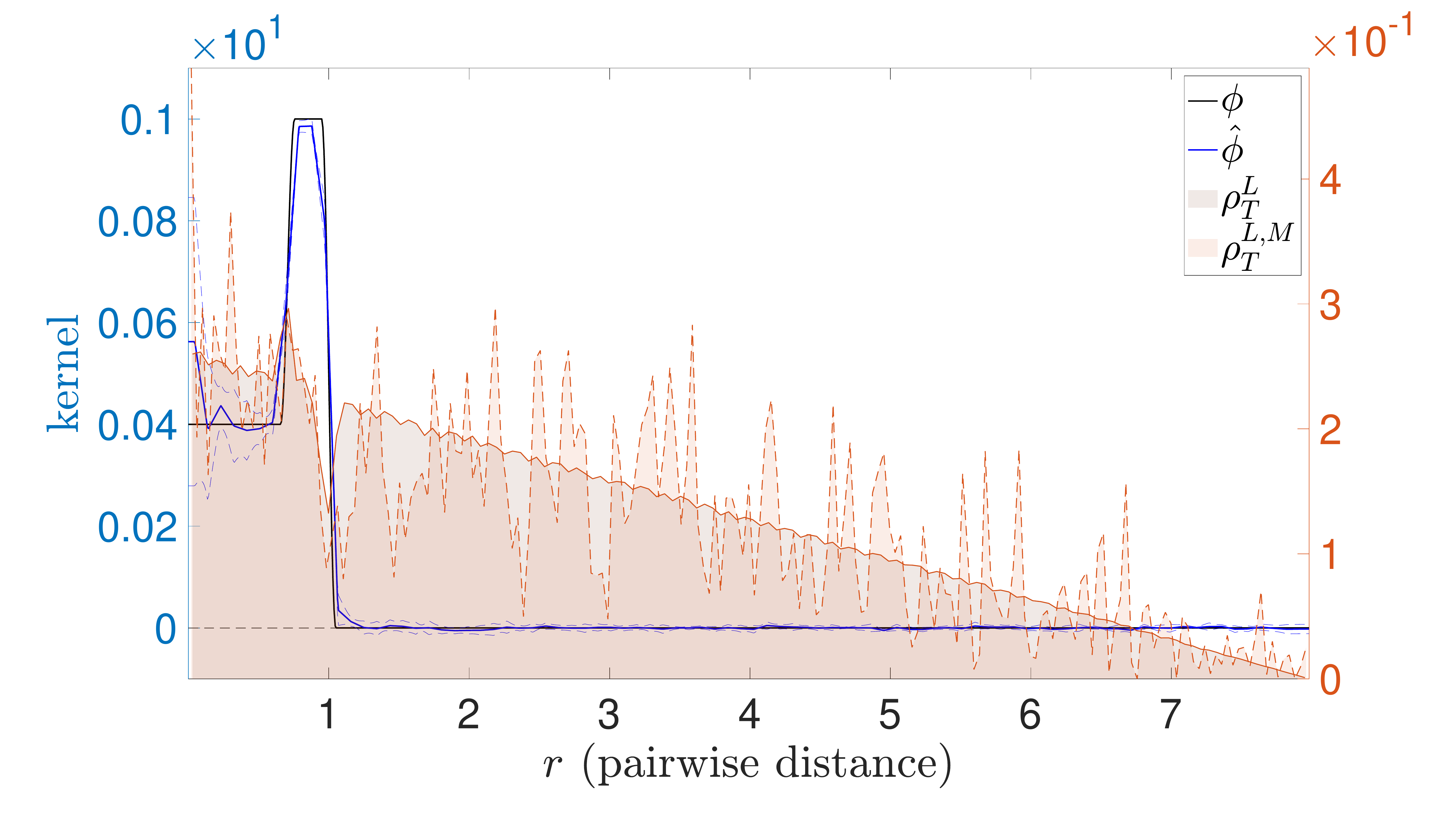}}
\subfigure{\label{figOD:2}\includegraphics[width=0.48\textwidth]{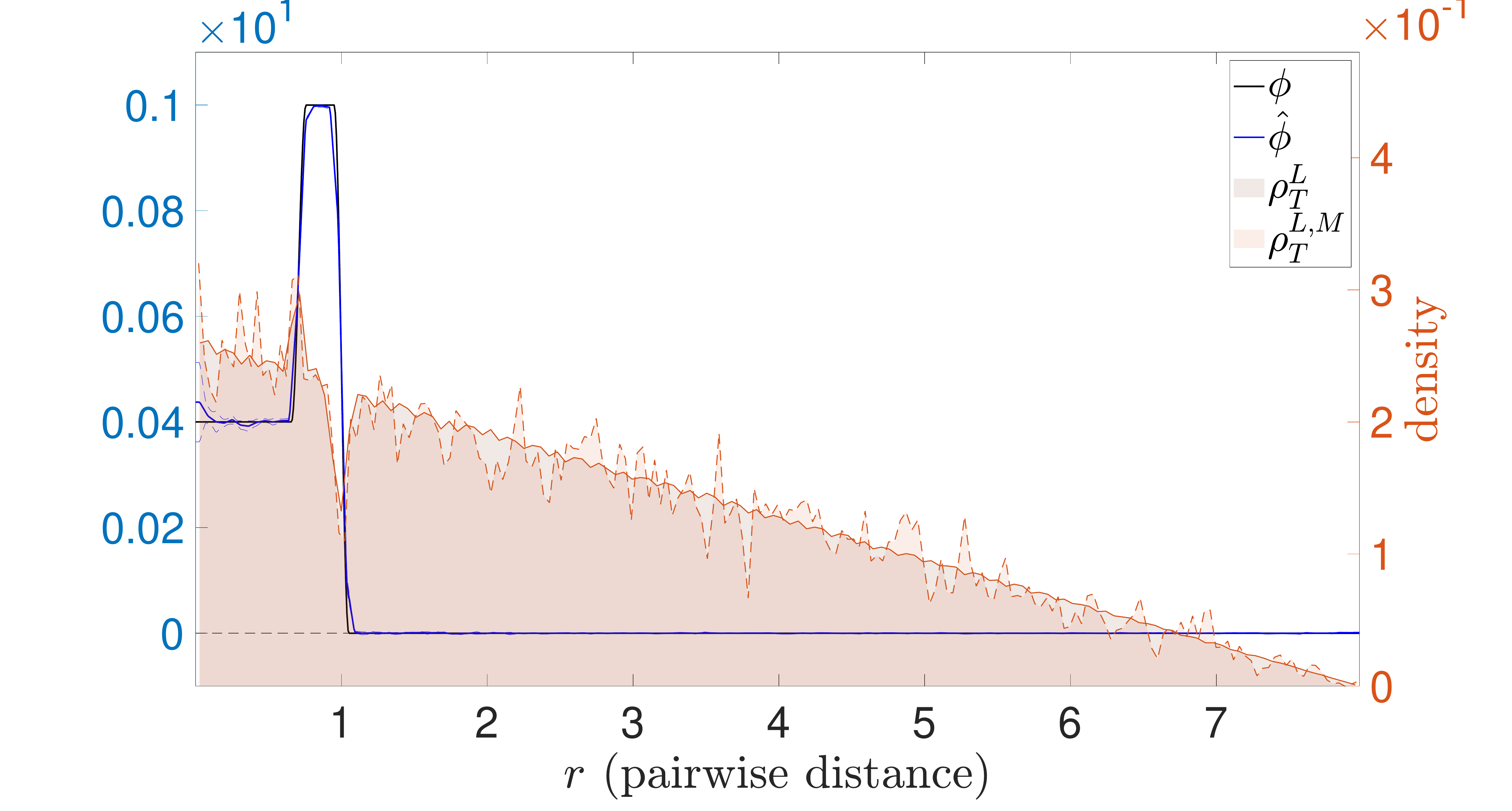}}
\subfigure{\label{figOD:3}\includegraphics[width=0.48\textwidth]{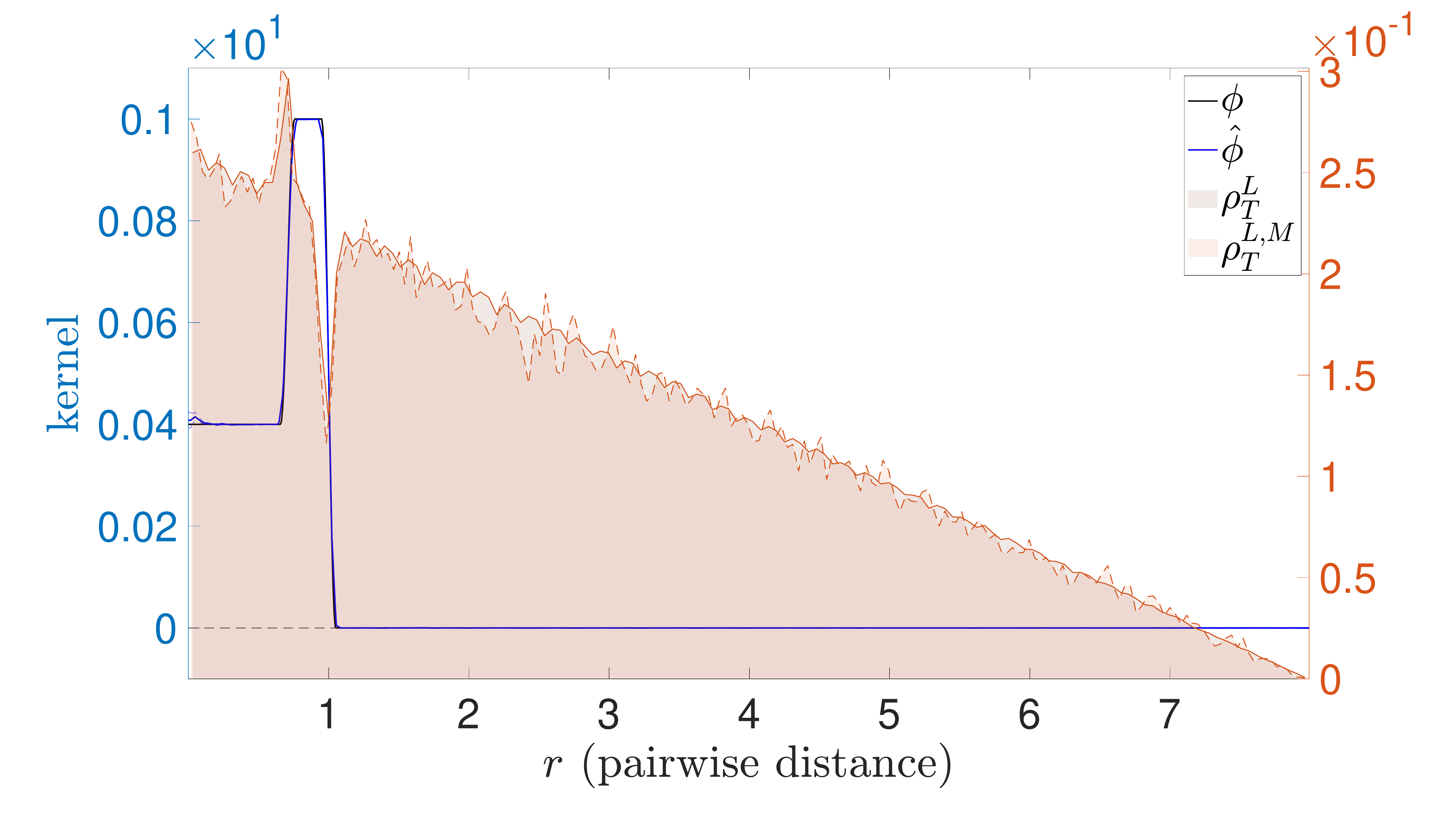}}
\subfigure{\label{figOD:4}\includegraphics[width=0.48\textwidth]{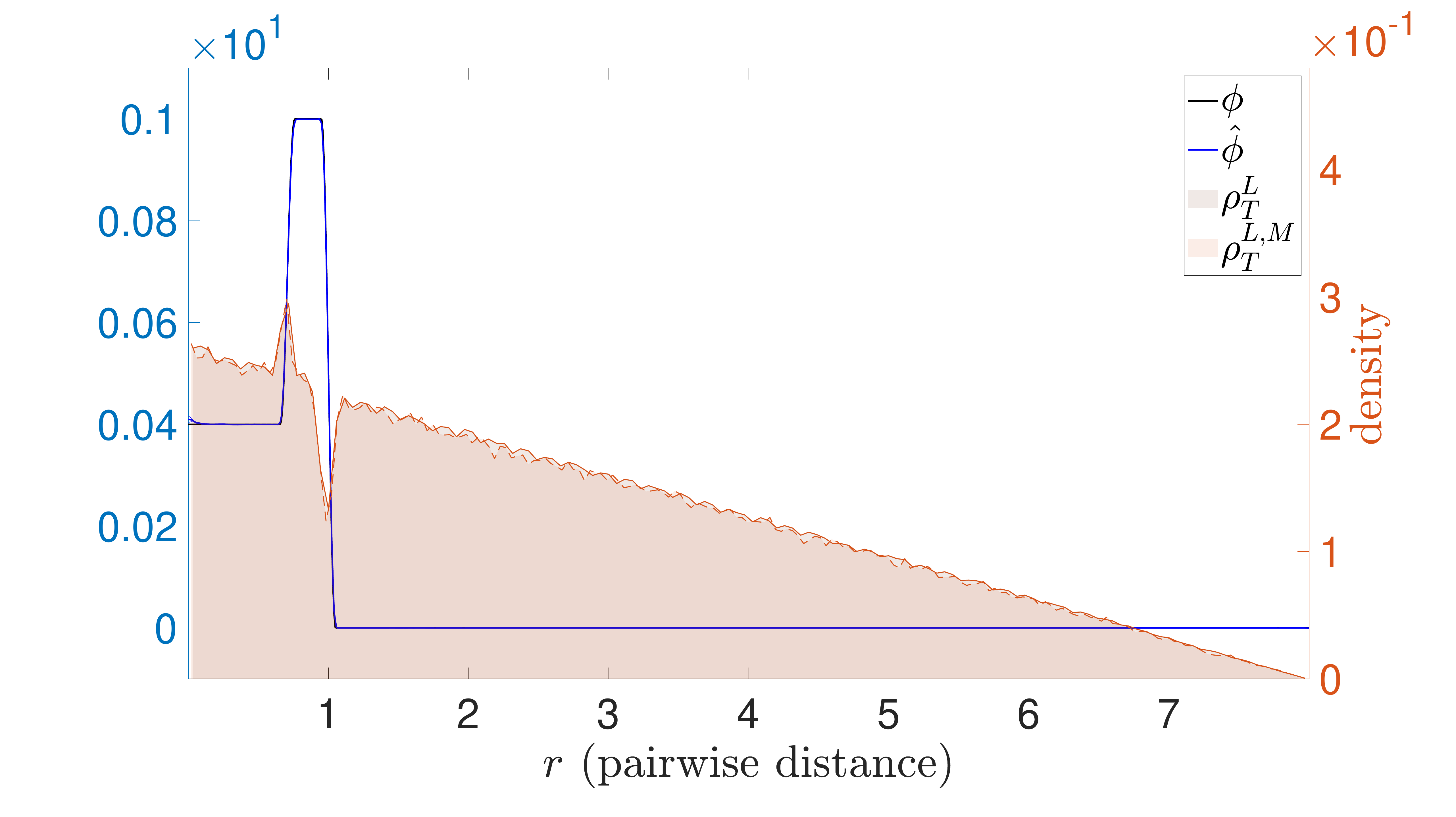}}
\caption{(Opinion Dynamics) Comparison between true and estimated interaction kernels with different values of $M$, together with histograms (shaded regions) for $\rhoL$ and $\rho_T^{L,M}$. In black: the true interaction kernel. In blue: the piecewise constant estimator smoothed by the linear interpolation.  From left-top to right-bottom: learning from $M=2^4, 2^7,2^{10},2^{13}$ trajectories. The standard deviation bars on the estimated interaction kernels become smaller and barely visible. The estimators converge to the true interaction kernel, as also indicated by the relative errors: $(1.5 \pm  0.08) \cdot 10^{-1}$, $(6 \pm 0.5)\cdot 10^{-2}$, $(2.5\pm 0.03) \cdot 10^{-2}$ and  $(8.9\pm 0.1)\cdot 10^{-3}$.}
\label{t:ODH1_kernel}
\end{figure}

We then use the learned interaction kernels $\hat\phi$ to predict the dynamics, and summarize the results in Figure \ref{t:ODH1_trajM16_err}. Even with $M=16$, our estimator produces very accurate approximations of the true trajectories. Figure \ref{t:ODH1_traj_err} displays the accuracy of trajectory predictions.  As $M$ increases, the mean trajectory prediction errors decay at the same rate as the learning rate of the interaction kernel, not only in the training time interval $[0,0.5]$ (consistently with Theorem \ref{firstordersystem:Trajdiff}), but also in the future time interval $[0.5, 20]$ (suggesting the bounds in Theorem \ref{firstordersystem:Trajdiff} may be sometimes overly pessimistic).

\begin{figure}[!htbp]
\centering     
 \includegraphics[width=0.96\textwidth]{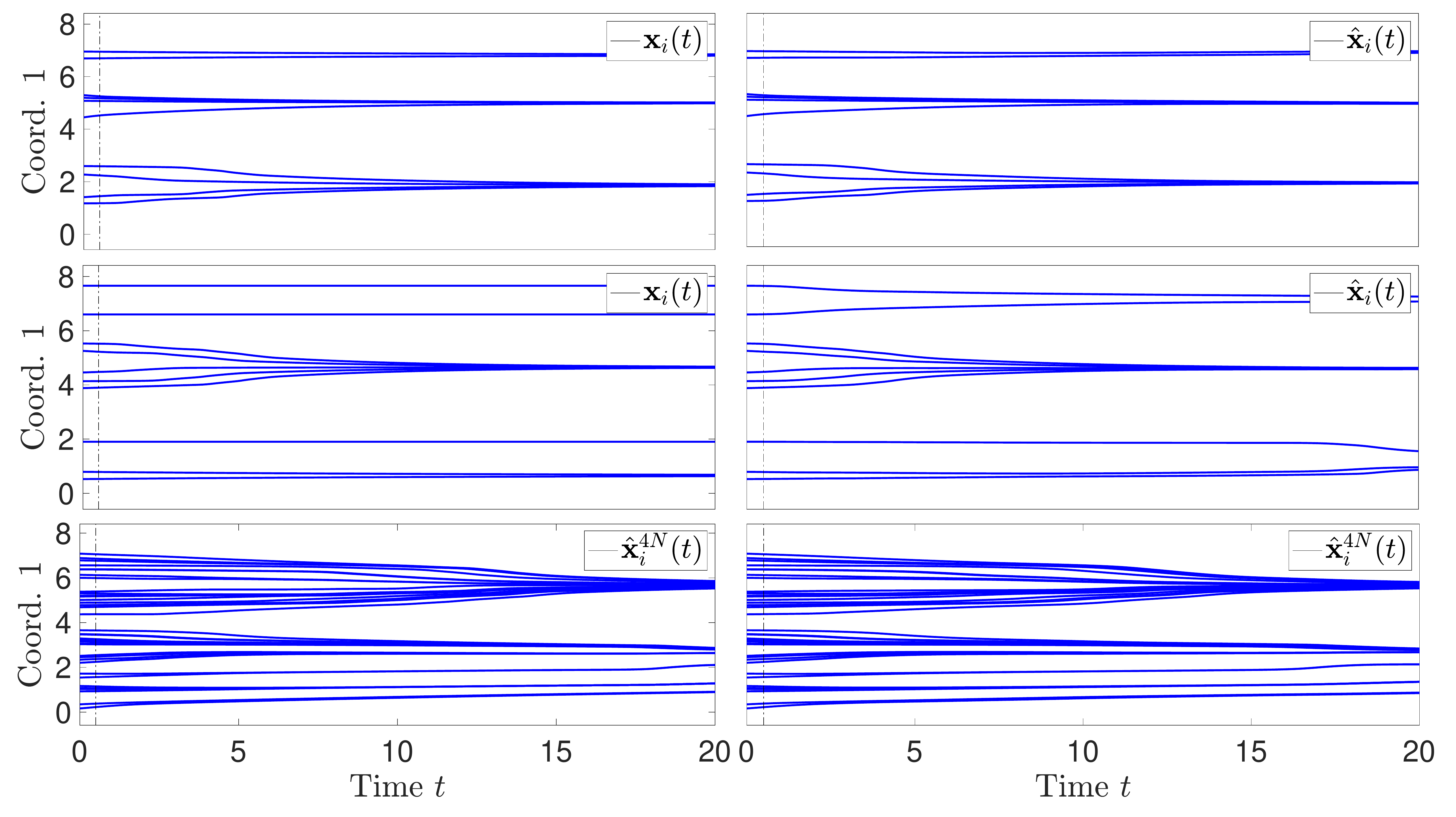}
\caption{ \textmd{{ (Opinion Dynamics)  $\bX(t)$ (Left column) and $\hat \bX(t)$ (Right column) obtained with  the true kernel $\intkernel$ and the estimated interaction kernel $\hat \intkernel$   from $M=16$ trajectories, for an initial condition in the training data (Top row) and an initial condition randomly chosen (Middle row). The black dashed vertical line at $t = 0.5$ divides the  ``training" interval $[0, 0.5]$ from the ``prediction" interval [0.5, 20].  Bottom row:  $\bX(t)$ and $\hat \bX(t)$ obtained with $\intkernel$ and $\hat \intkernel$, for dynamics with larger $N_{new} = 4N$, over one set of initial conditions. We achieve small error in all cases, in particular predicting the number and location of clusters for large time. The mean of max-in-time trajectory errors over $10$ learning trials can be found in Figure \ref{t:ODH1_traj_err}.}}}\label{t:ODH1_trajM16_err}
\end{figure}

\begin{figure}[!htbp]
\centering     

\subfigure{\label{figODTraj:1}\includegraphics[width=0.49\textwidth]{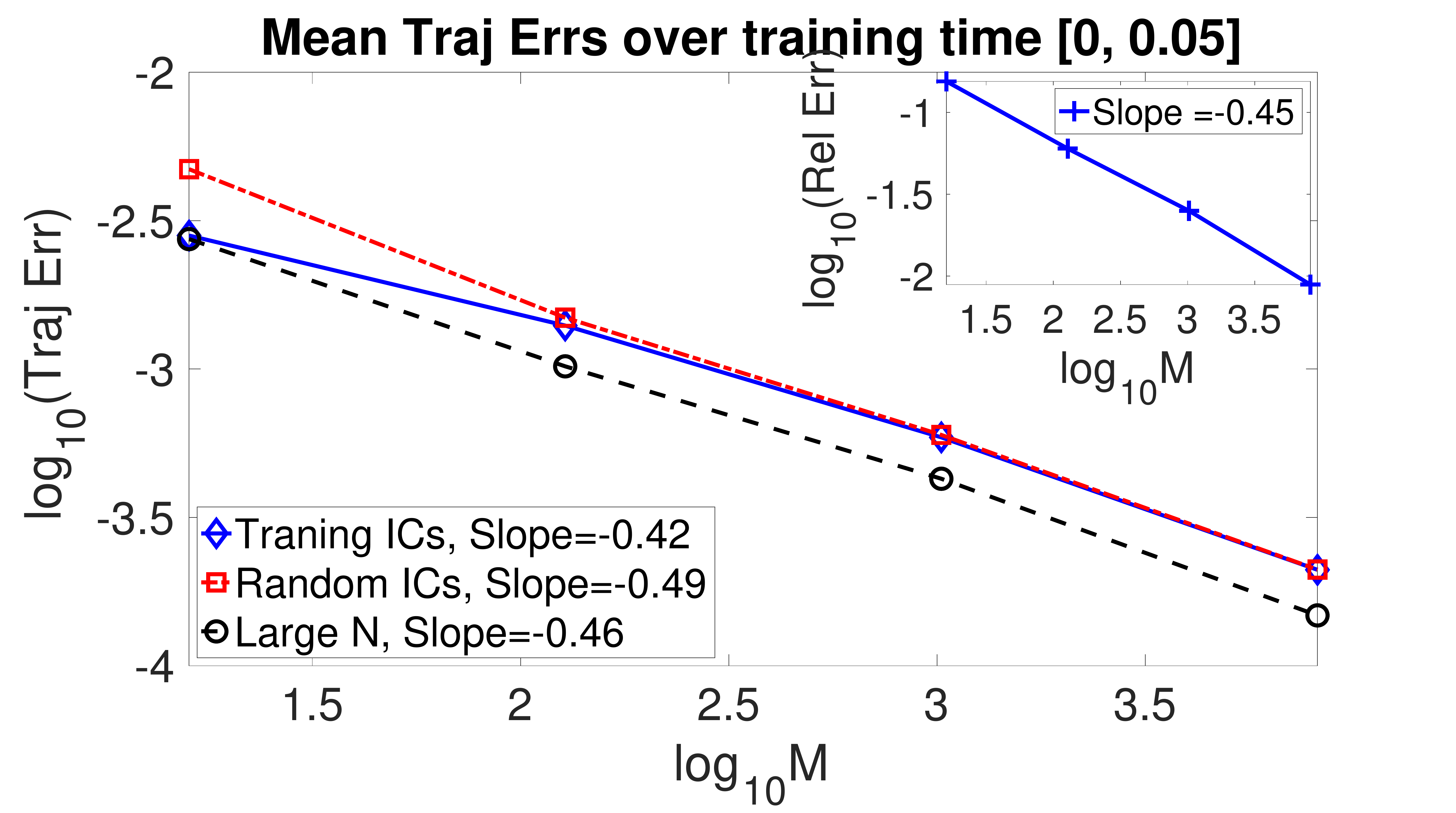}}
\subfigure{\label{figODTraj:2}\includegraphics[width=0.49\textwidth]{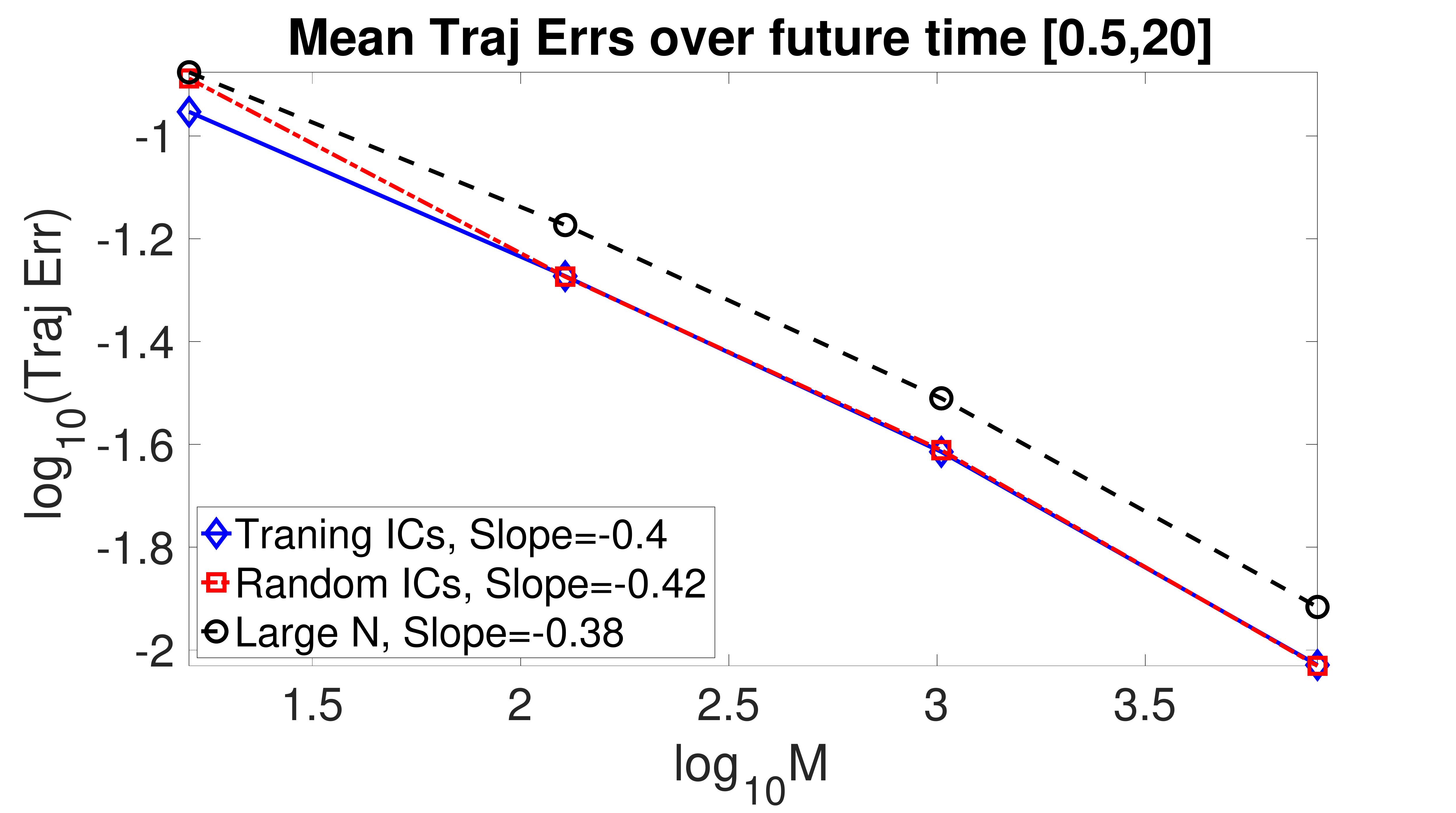}}
\caption{ \textmd{{ (Opinion Dynamics) Mean errors in trajectory prediction over 10 learning trails using estimated interaction kernels obtained with different values of $M$: for initial conditions  in the training set (Training ICs), randomly drawn from $\mu_0$ (Random  ICs), and for  a system with $4N$ agents (Large $N$). {Left}: Errors over the training time interval [0,0.5]. {Right}: Errors over the prediction time interval [0.5,20]. {Right upper corner inside the left figure}:  the learning rate of the relative $L^2(\rhoL)$ errors of  smoothed estimated interaction kernels. The mean trajectory errors  decay at a rate close to the learning rate of interaction kernels, in agreement with Theorem \ref{firstordersystem:Trajdiff}. }}}\label{t:ODH1_traj_err}
\end{figure}

\begin{figure}[!htbp]
\centering
\subfigure{\label{t:ODH1_CoercivityConstant}\includegraphics[width=0.49\textwidth]{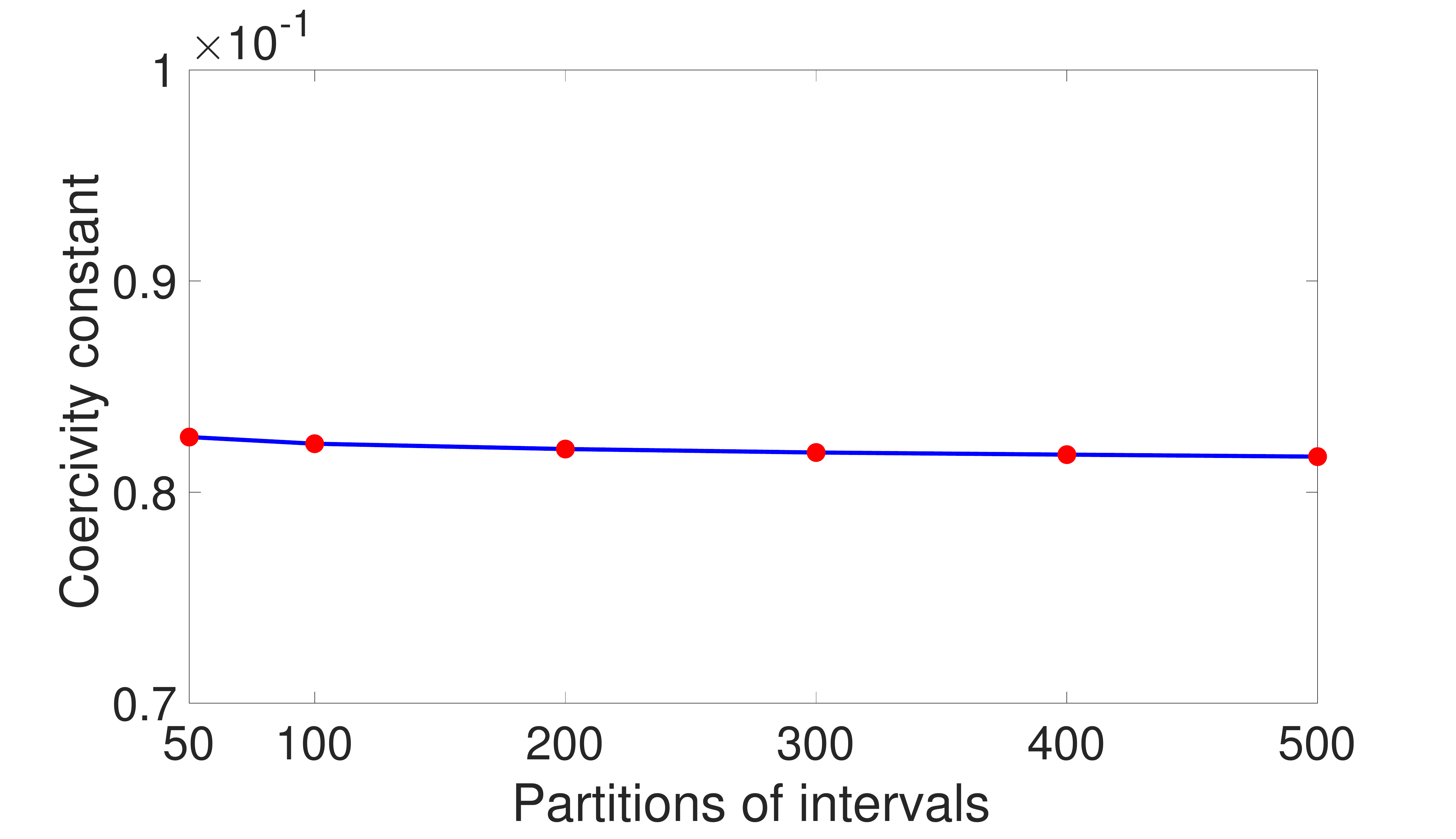}}
\subfigure{\label{t:ODH1_Convergence}\includegraphics[width=0.49\textwidth]{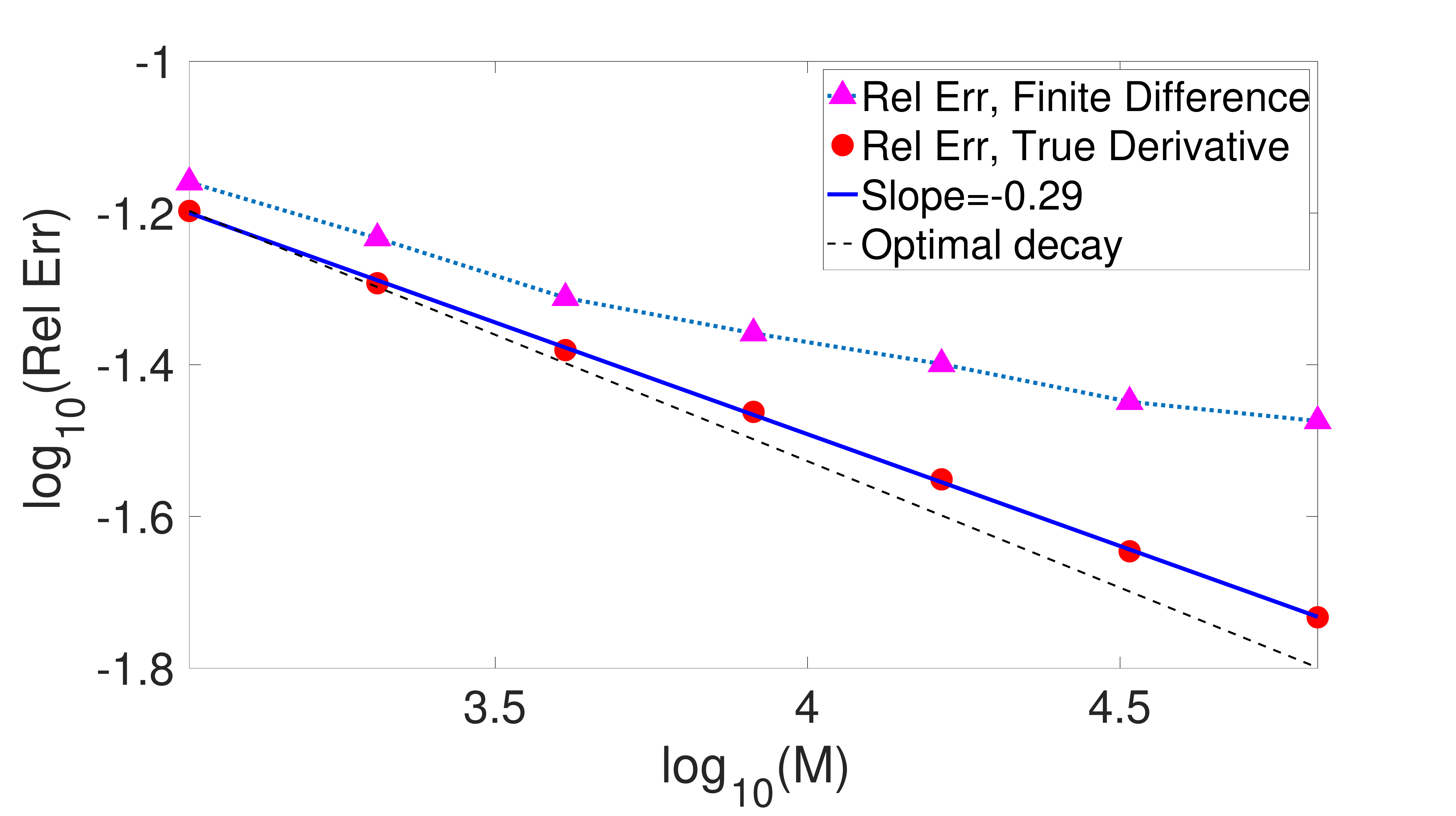}}
\caption{(Opinion Dynamics). {Left}: The approximate coercivity constant on $\mathcal{H}_n$ that consists of piecewise constant functions over $n$ uniform partitions of $[0,10]$, obtained from data of $10^5$ trajectories. {Right}: Given true velocities, the learning rate of the estimators is 0.29, close to the theoretical optimal min-max rate ${1}/{3}$ (shown in the black dot line). Otherwise, for unobserved velocities, the curve of the learning error flattens due to the approximation error of velocities by the finite difference method.}\label{t:ODH1_Convergence_Plot} 
\end{figure}

To verify the learnability of the interaction kernel, we estimate the coercivity constant on the hypothesis spaces used in the experiments: we partition $[0,10]$  into $n$ uniform subintervals and choose a set of basis functions consisting of the indicator functions, and then use Algorithm \ref{algo:coercivity}. We display the estimated coercivity constant of $\mathcal{H}_n$ for different values of $n$ in Figure \ref{t:ODH1_CoercivityConstant}. These numerical results suggest that the coercivity constant, over $L^2([0,10],\rhoL)$, is around $0.08$, close to the conjectured lower bound $0.09$ based on Theorem \ref{firstordersingle:coercivity}. We impute this small difference to the finite sample approximation.  

\begin{figure}[!htbp]
\centering     
\subfigure{\label{figODTraj:3}\includegraphics[width=0.96\textwidth]{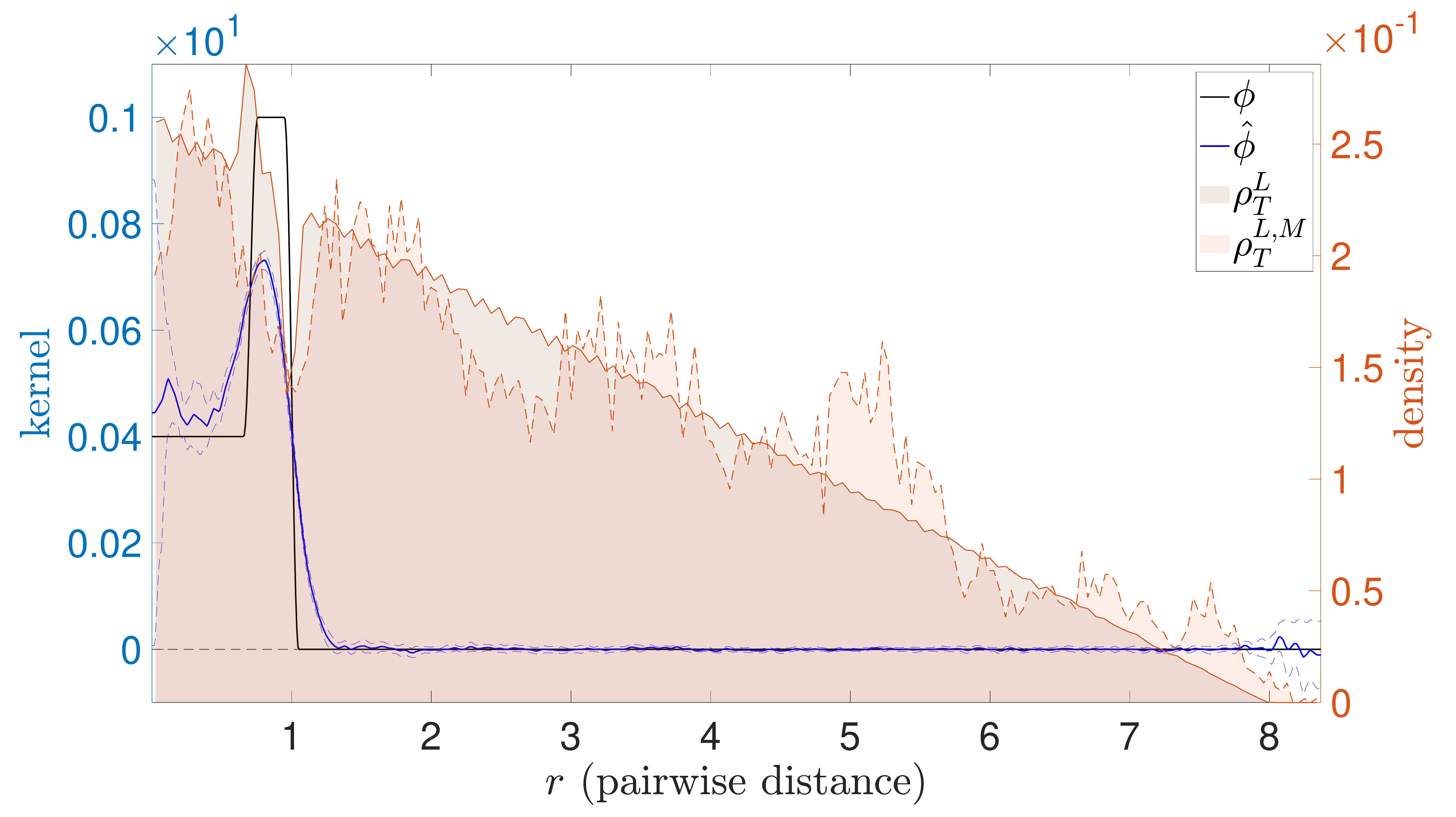}}
\caption{ \textmd{{ (Opinion Dynamics) Interaction kernel learned with Unif.$([-\sigma,\sigma])$ additive noise, for $\sigma=0.15$, in the observed positions {\em{and observed velocities}}; here $M=128$,  with all the other parameters as in Table \ref{t:OD_params}.}}}\label{t:ODH1_traj_err_noise}
\end{figure}

Figure \ref{t:ODH1_Convergence} shows that the learning rate of the interaction kernel is $M^{-0.3}$,  close to the theoretical optimal rate $M^{-{1}/{3}}$ in Theorem \ref{t:firstordersystem:thm_optRate} up to a logarithmic factor. An interesting phenomenon is that smoothed learned interaction kernel exhibits a learning rate of $M^{-0.45}$ (see upper-right corners of plots in Figure \ref{t:ODH1_traj_err}). We explain this phenomenon as follows: the gridded interpolation smoothing techniques make our piecewise constant estimators match well with the true kernel, which is almost piecewise constant, and given the lack of noise, it succeeds in reducing the error in the estimator and yielding an almost parametric learning rate.

\begin{figure}[!htbp]
\centering     
\subfigure{\label{figODTraj:4}\includegraphics[width=0.96\textwidth]{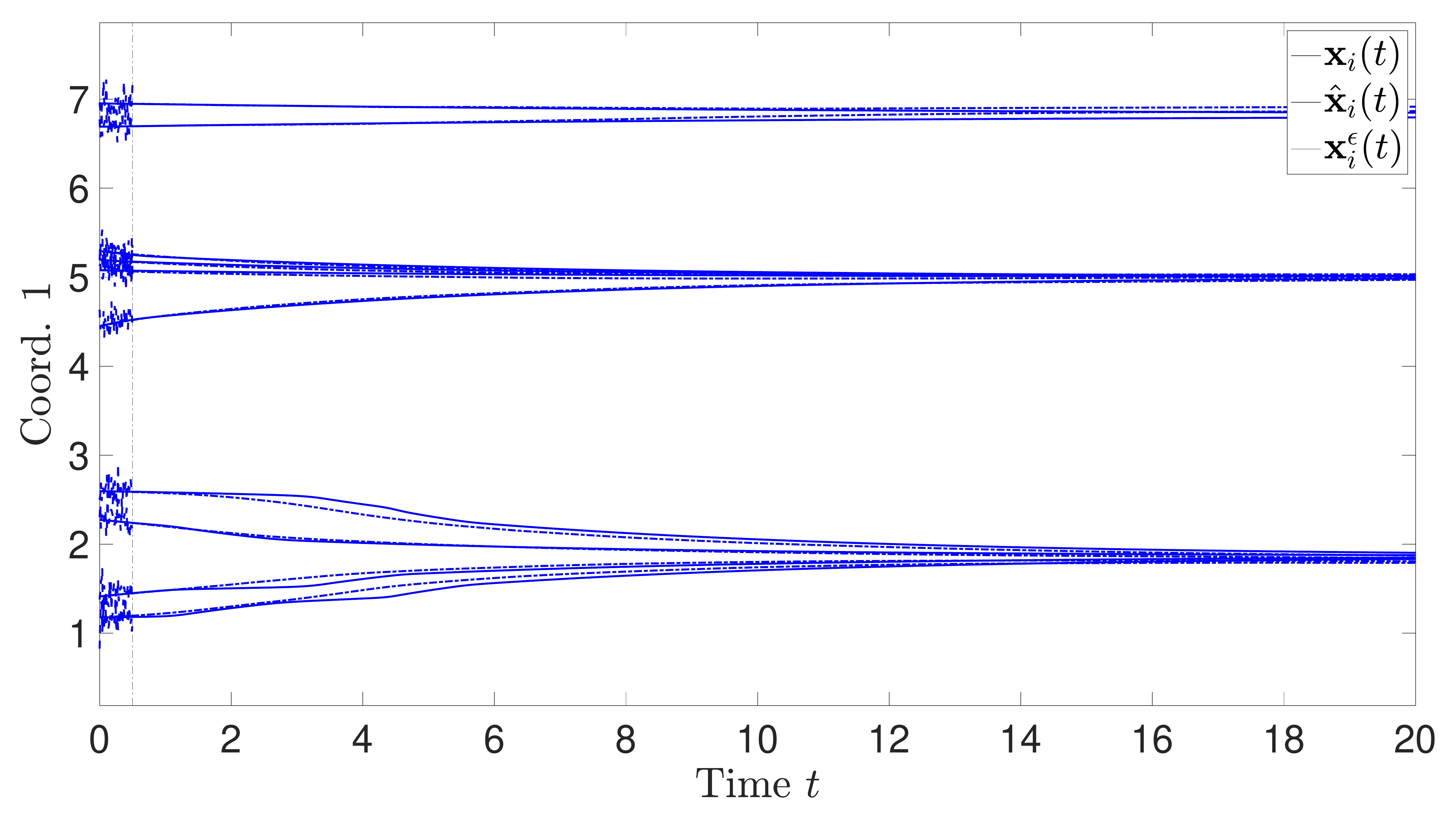}}
\caption{ \textmd{{ (Opinion Dynamics) One of the observed trajectories before and after being perturbed by the additive noise. The solid lines represent the true trajectory; the dashed semi-transparent lines represent the noisy trajectory used as training data (together with noisy observations of the velocity); the dash dotted lines  are the predicted  trajectory learned from the noisy trajectory.}}}\label{t:ODH1_traj_err_noise1}
\end{figure}

\begin{figure}[!htbp]
\centering
\subfigure{\label{t:ODH1_Convergence_noise_left}\includegraphics[width=0.49\textwidth]{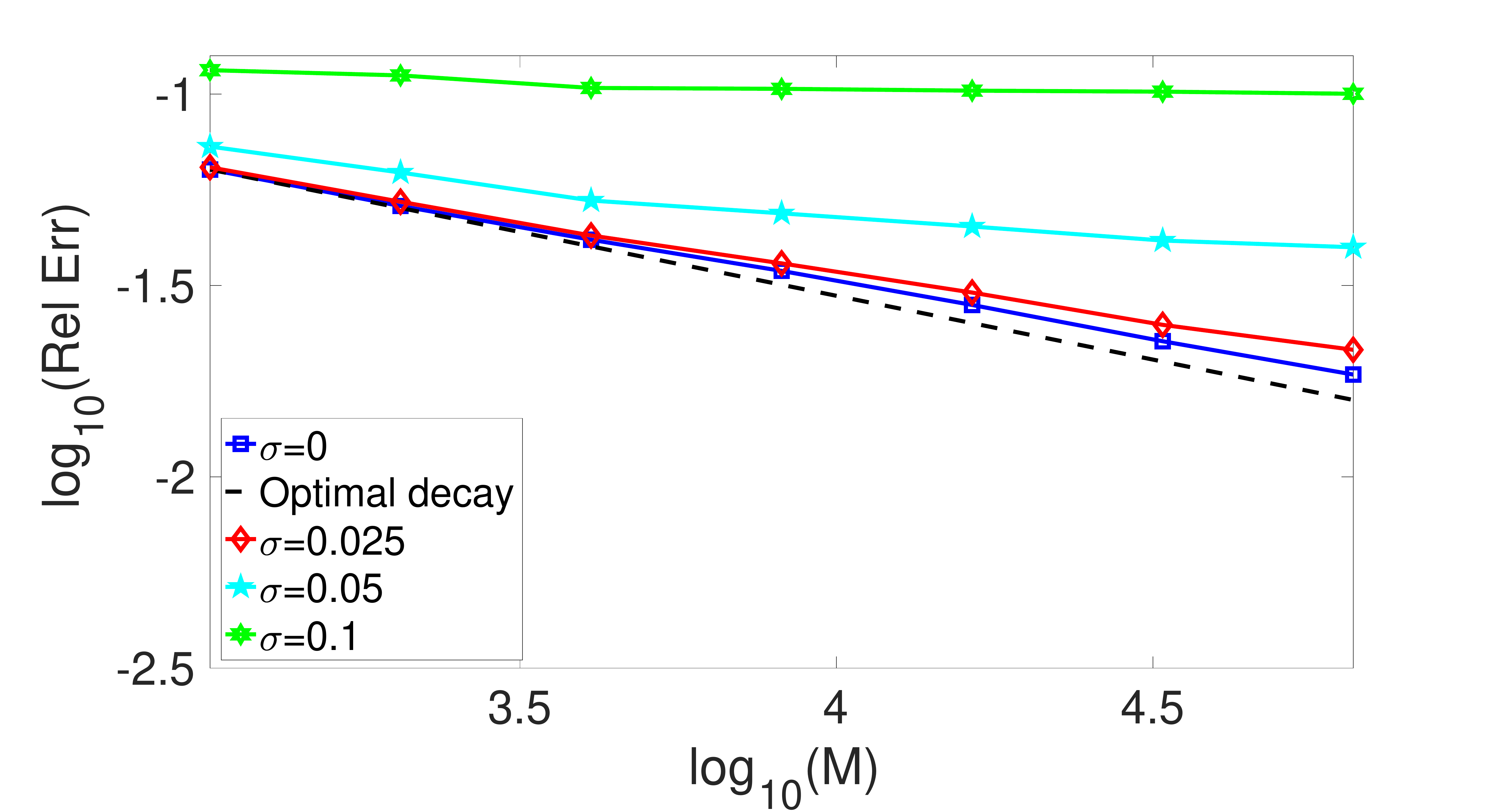}}
\subfigure{\label{t:ODH1_Convergence_noise_right}\includegraphics[width=0.49\textwidth]{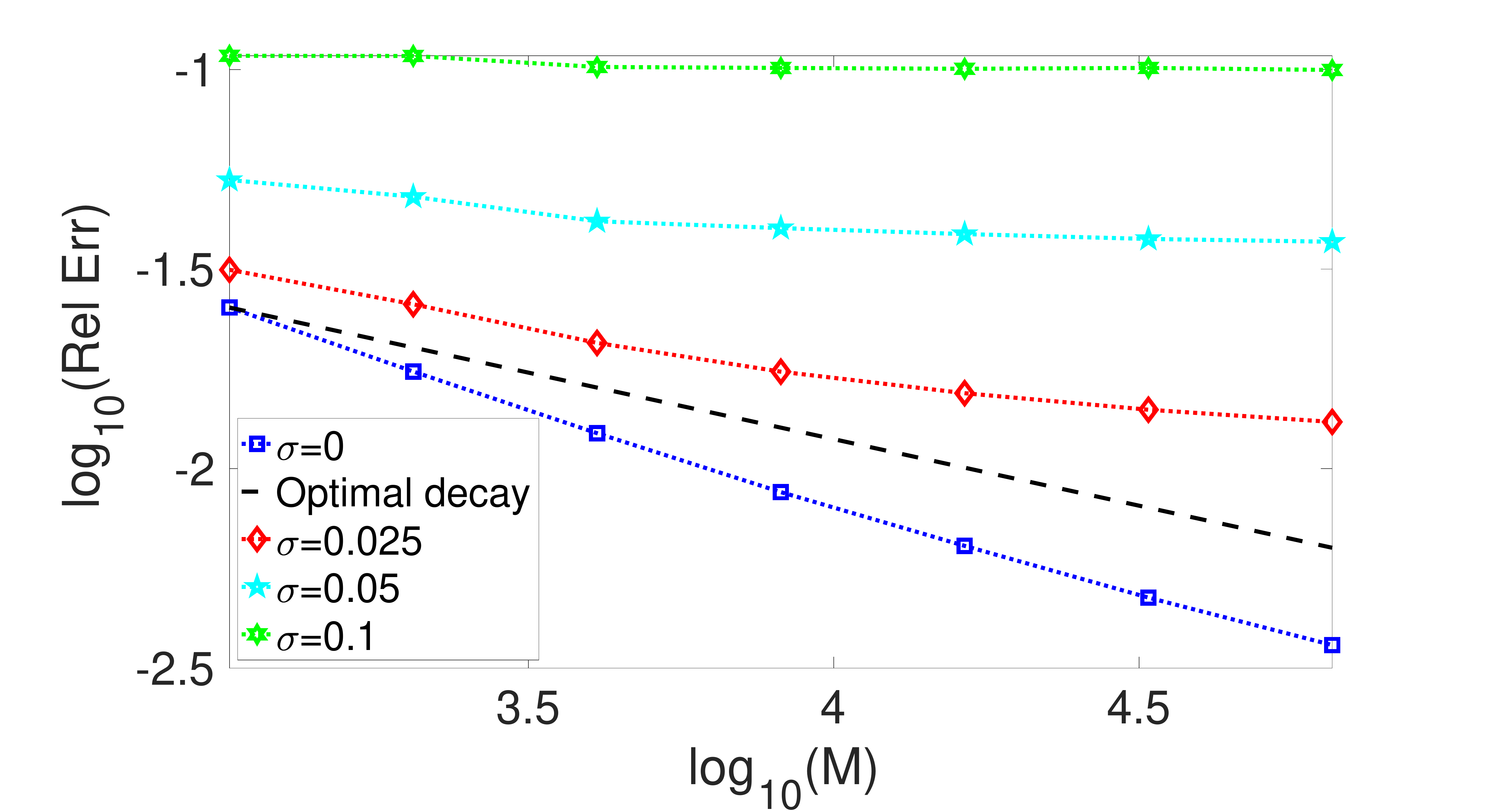}}

\caption{(Opinion Dynamics). The learning rates of estimators with different levels of additive noise drawn from Unif.$([-\sigma,\sigma])$. The noise causes a flattening of the error, with large noise making improvements in the error negligible as the number of observations increases. {Left}: Rates for estimators without smoothing. {Right}: Rates for smoothed estimators.  }\label{t:ODH1_Convergence_Plot_Noise} 
\end{figure}

\subsection {Predator-Swarm dynamics}\label{PS1stOrderDescriptions}
There has been a growing literature on modelling interactions between animals of multiple types for the study of animal motion, see \cite{EMSC2014, PEK1999, CK2009, Nowak2006, FMSP2007}. We consider a first-order Predator-Swarm system, modelling interactions between a group of preys and a single predator.  The prey-prey interactions have both short-range repulsion to prevent collisions, and long-range attraction to keep the preys in a flock.  The preys attract the predator and the predator repels the preys. Since there is only one predator,  there are no predator-predator interactions.  The intensity of interactions between the single predator and group of preys can be tuned with parameters, determining dynamics with various interesting patterns (from confusing the predator with fast preys, to chasing, to catching up to one prey). We use the set $C_1$ for the set of preys, and the set $C_2$ for the single predator.  We consider the interaction kernels  
\[
\intkernel_{1, 1}(r) = 1 - r^{-2}, \quad \intkernel_{1, 2}(r) = -2r^{-2}, \quad \intkernel_{2, 1}(r) = 3.5r^{-3}, \quad \intkernel_{2, 2}(r) \equiv 0.
\]

Since interaction kernels are all singular at $r=0$, we truncate them at $r_{\text{trunc}}$ by connecting it with an exponential function of the form $a\exp(-br)$ so that it has continuous derivative on $\mathbb{R}^+$. The truncation parameters are summarized in Table \ref{t:PS_kernel_params}. 

\begin{table}[H]
\centering
\begin{tabular}{|c | c | }
\hline 
kernels& $r_{\text{trunc}}$    \\ 
\hline 
 $\intkernel_{1,1}$ &0.4 \\
\hline
 $\intkernel_{1,2}$ & $1$\\
\hline
 $\intkernel_{2,1}$&$0.4$ \\
\hline
$\intkernel_{2,2}$&$0$\\
\hline
\end{tabular}
\caption{\textmd{\footnotesize{(PS)  Truncation parameters for the  Prey-Predator kernels} }}
\label{t:PS_kernel_params}
\end{table}


\begin{table}[htbp]
\centering
\begin{tabular}{| c | c | c | c | c | c | c |c|c|}
\hline 
 $d$ &$N_{1}$ & $N_{2}$ &$M_{\rhoL}$& $L$ & $[t_1;t_L;t_f]$   & $\probIC$          & deg($\psi_{kk'}$) &$n_{kk'}$ \\ 
\hline 
 $2$ &9& $1$&$10^5$ & \tabincell{c}{100} & $[0;1;20]$ & \tabincell{c}{Preys: Unif. disk [0,0.5]\\ Predators: Unif. ring [0.8,1]} & 1&$ 100(\frac{M}{\log M})^{\frac{1}{5}}$ \\
\hline
\end{tabular}
\caption{\textmd{\footnotesize{(PS) System and learning Parameters for the Predator-Swarming system} }}
\label{t:PS_system_params}
\end{table}
 
In the numerical experiments, the initial positions of the preys are sampled from the uniform distribution on the disk with radius 0.5,  and the initial position of the predator is sampled from the uniform distribution in the ring with radii between $0.8$ and $1$. The dynamics  mimics the following real situation: preys gather and scatter in a small area; the predator  approaches the preys gradually and  begins to chase the preys within a small distance; although the predator is able to catch up with the swarm as a whole,  the individual prey is able to escape by ``confusing" the predator: the preys form a ring with the predator at the centre. Finally, they form a flocking behaviour, i.e., they all run in the same direction.

{In this example, we assume that the prior information is that each interaction kernel $\bintkernel_{kk'}$ is in the 2-H\"older space, i.e., its derivative  is Lipchitz. Note that the true interaction kernels are not compactly supported. However, our theory is still applicable to this case: due to the compact support of $\mu_0$ and decay of $\bintkernel$ at $\infty$, Grownwall's inequality  implies that,  for a sufficiently large $R$ (depending only on $\supp{\mu_0}$, $\|\bintkernel\|_{\infty}$ and $T$),  $\bintkernel$ and $\bintkernel1_{[0,R]}$ would produce the same dynamics on $[0,T]$ for any initial conditions sampled from $\mu_0$,  but now $\bintkernel=\bintkernel1_{[0,R]}$ is in the function space $\mathcal{K}_{R,S}$. Therefore, we can still assume that  $\bintkernel$  is compactly supported.  Here, we choose $R=10$ and  $\mathcal{H}_{n_{kk'}}$ to be the function space that consists of piecewise linear functions on the uniform partition of [0,10] with $n$ intervals.  It is well-known in approximation theory (e.g. \cite{devore1992wavelets}) that 
$\inf_{\varphi\in \mathcal{H}_n}\|\varphi-\intkernel_{kk'}1_{[0,10]}\|_{\infty} \leq \text{Lip}[\intkernel_{kk'}^{'}]n^{-2}$. Therefore the conditions in Theorem \ref{t:firstordersystem:thm_optRate} are satisfied with $s=2$.  Our theory suggests that any choice of dimension $n$ that is   proportional to $(\frac{M}{\log M})^{{1}/{5}}$  yields an optimal learning rate $M^{-\frac{2}{5}}$, up to a logarithmic factor. We choose  $n=100(\frac{M}{\log M})^{{1}/{5}}$ here.  The system and learning parameters for Predator-Swarm dynamics are summarized in Table  \ref{t:PS_system_params}.} 

\begin{figure}[!htbp]
\centering     

\subfigure{\label{figPS:1}\includegraphics[width=\textwidth]{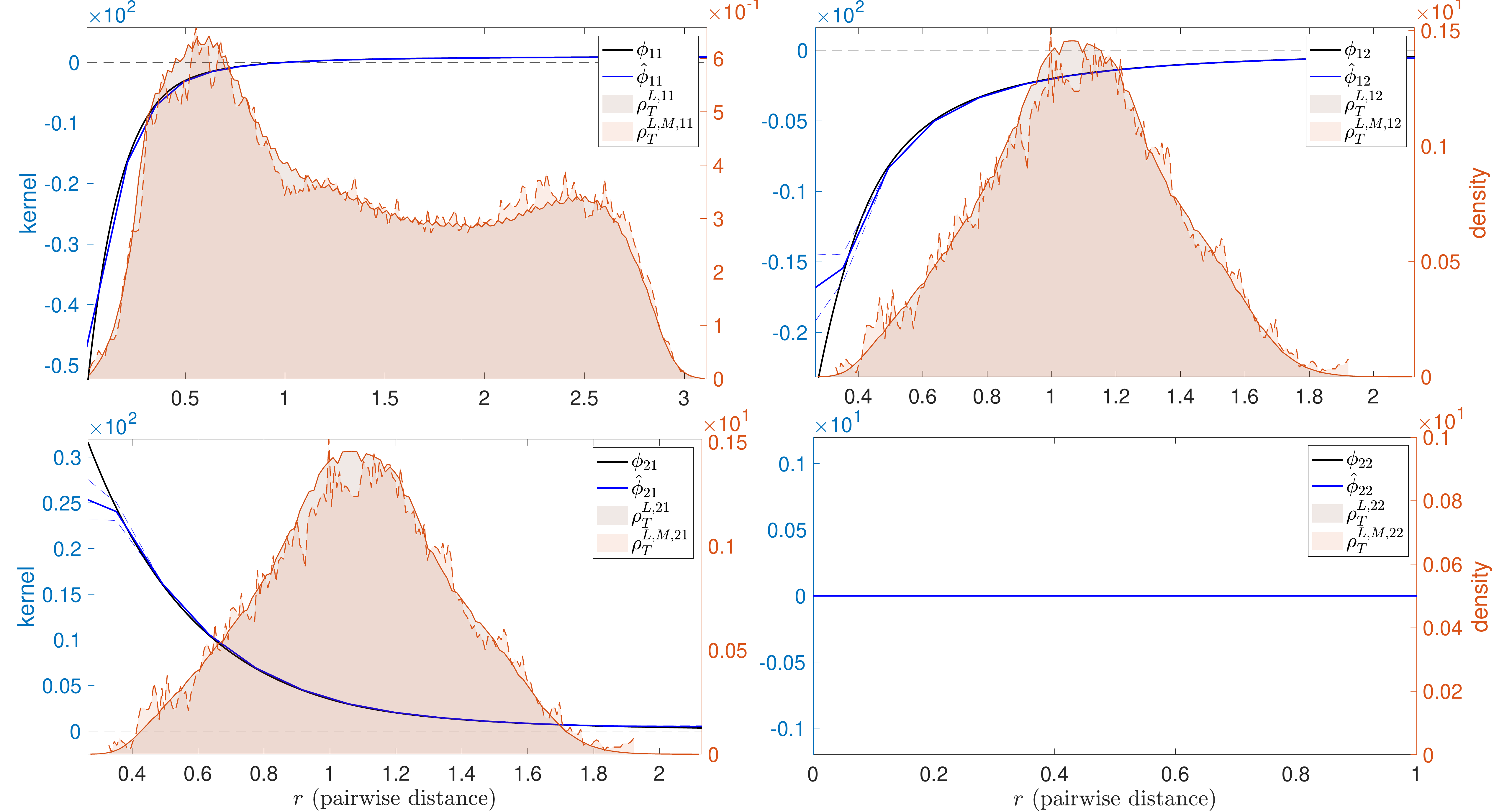}}
\subfigure{\label{figPS:2}\includegraphics[width=\textwidth]{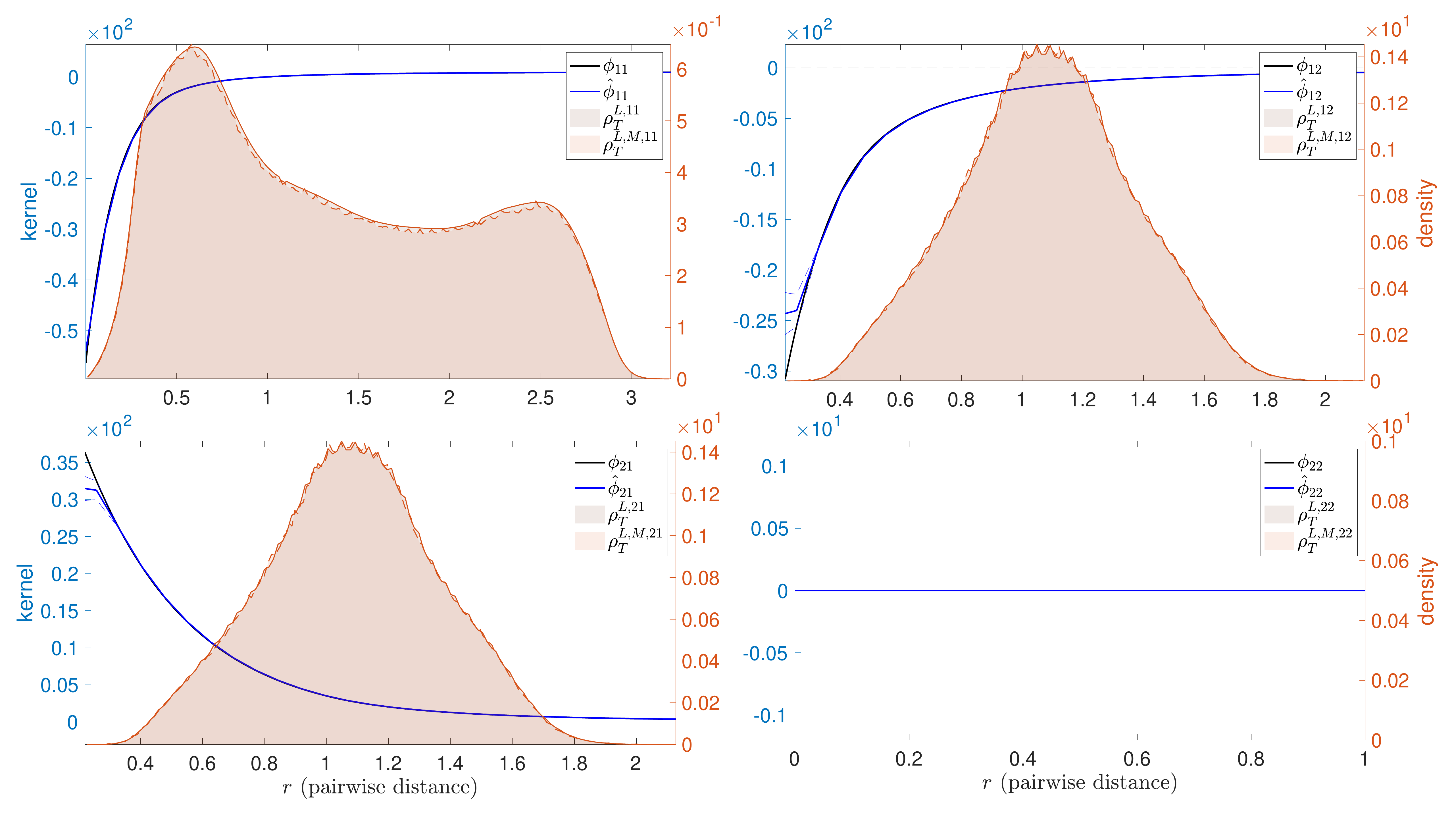}}

\caption{(Predator-Swarm Dynamics) Comparison between true and estimated interaction kernels with $M=16$ (Top) and $1024$ (Bottom). In black: the true interaction kernels. In blue: the learned interaction kernels using piecewise linear functions.  When $M$ increases from 16 to $1024$, the standard deviation bars on the estimated interaction kernels become smaller and less visible. The relative errors  in $\mbf{L}^2(\brhoL)$  for the interaction kernels are $(9\pm 2) \cdot 10^{-3}$  and $(2.5\pm 0.05)\cdot 10^{-3}$.  }
\label{t:PS1_kernel_L100_noderivative}
\end{figure}

Figure \ref{t:PS1_kernel_L100_noderivative} indicates that the estimators match the true interaction kernels extremely well except for a small bias at locations near 0. We impute this error near 0 to two reasons:  (i)  the strong short-range repulsion between agents force the pairwise distances to stay bounded away from $r = 0$, yielding a $\brhoL$ that is nearly singular near 0, so that there are only a few samples to learn the interaction kernels near $0$. We see that as $M$ increases, the error near 0 is  getting smaller, and we expect it to converge to $0$.  (ii) Information of $\bintkernel(0)$ is lost due to the structure of the equations, as we mentioned earlier in the previous example, which may cause the error in the finite difference approximation of velocities to affect the reconstruction of values near $0$. 

\begin{figure}[!htbp]
\centering     
 \includegraphics[width=0.96\textwidth]{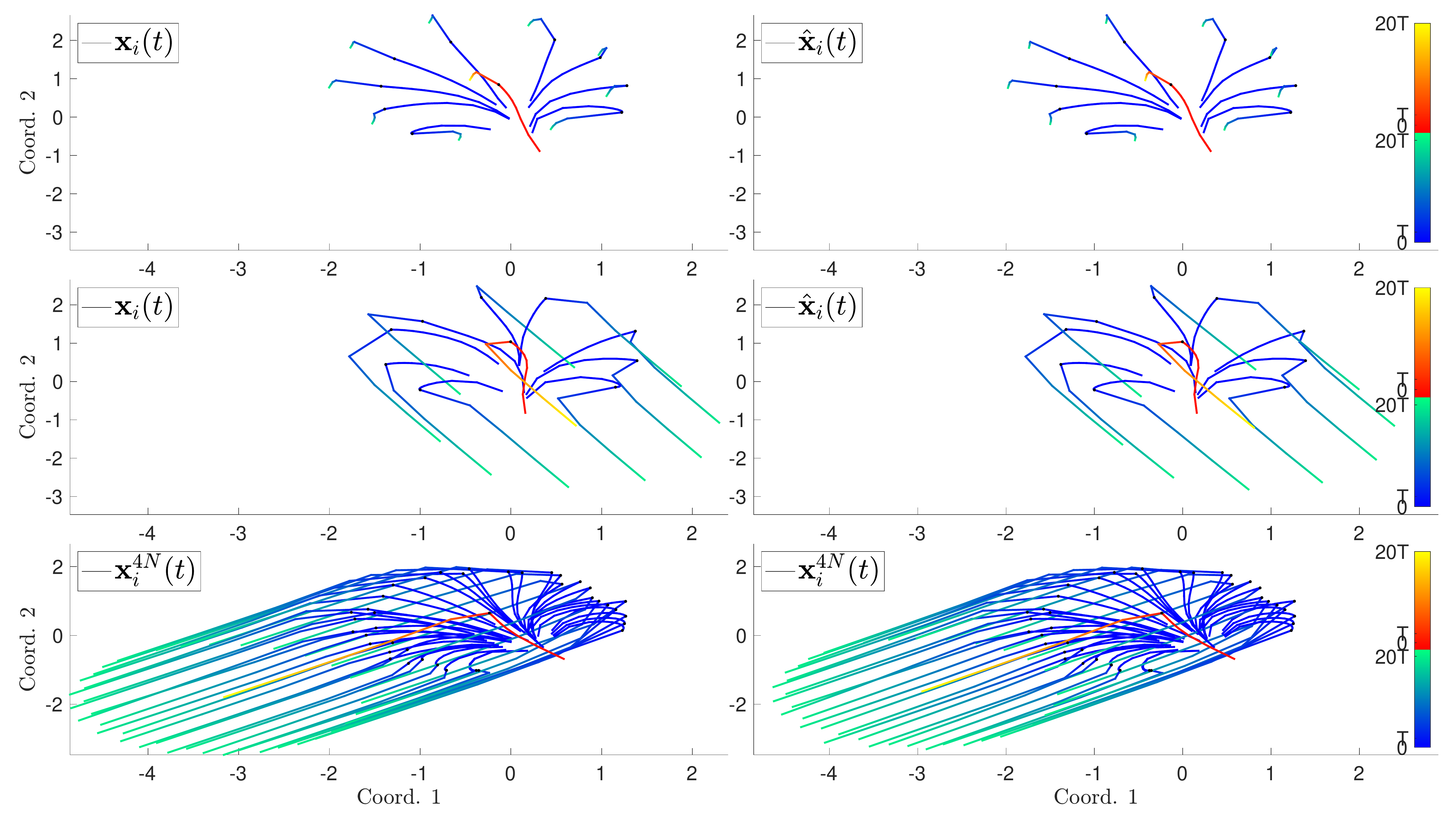}
\caption{ \textmd{{ (PS) $\bX(t)$ and $\hat \bX(t)$ obtained with $\bintkernel$ and $\hat \bintkernel$  learned from $M=16$ trajectories respectively: for an initial condition in the training data ({Top}) and an initial condition randomly chosen ({Middle}). The black dot at $t = 1$ divides the  ``training" interval $[0, 1]$ from the ``prediction" interval [1, 20]. {Bottom}:  $\bX(t)$ and $\hat \bX(t)$ obtained with $\intkernel$ and $\hat \intkernel$  learned from $M=16$ trajectories respectively, for dynamics with larger $N_{new} = 4N$, over  a set of initial conditions. We achieve small errors in all cases, in particular we predict successfully the flocking time and direction. The means of trajectory errors can be found in Figure \ref{t:PS1_traj_err}.}}}\label{figPSTrajM16}
\end{figure}

Figure \ref{figPSTrajM16} shows that with a rather small $M$,  the learned interaction kernels  not only produce an accurate approximation of the transient behaviour of the agents over the training time interval $[t_1,t_L]$, but also of the flocking behaviour over the large time interval $[t_L,t_f]$ including the time  of  formation and the direction of a flocking, which is perhaps beyond expectations.

 Figure \ref{figPSTraj:1} shows that the mean trajectory errors over 10 learning trials decay with $M$ at a rate 0.32  on the training time interval $[0,1]$, matching the learning rate of smoothed kernels, even in the case of a new system with $4N$ agents.  This agrees with  Theorem \ref{firstordersystem:Trajdiff} on the learning rate of trajectories over the training time. For the prediction time interval $[1, 20]$,  our learned interaction kernels also produced very accurate approximations of true trajectories in all cases, see Figure \ref{figPSTraj:2}. 

Next, we study the learnability of the estimated interaction kernels in this system. As demonstrated by  Proposition  \ref{firstordersystem:coercivityconstant}, the coercivity constant is the minimal eigenvalue of $A_{L,\infty,\bhypspace}$,  which in our cases is blocked diagonal:   one block for learning prey-prey and prey-predator interactions from velocities of preys, and the other block for learning predator-prey interaction from velocities of the predator.   We display the minimal eigenvalues for each block in Figure \ref{t:PS_Coercivityconstant}.  We see that the minimal eigenvalue of the prey-prey block matrix stays around $2\cdot10^{-2}$ and the predator-predator matrix stays around $0.7\cdot10^{-2}$  as partitions get finer.  We therefore conjecture that the coercivity constant over $\bL^2([0,10],\brhoL)$ is about $0.7\cdot10^{-2}$. 

When true velocities are observed,  we obtain a learning rate for $\|\hat \bintkernel(\cdot)\cdot-\bintkernel(\cdot)\cdot\|_{L^2(\brhoL)}$ around $M^{-0.35}$ ($\log M/M)^{-0.4}$) (see Figure \ref{t:PS_Convergencerate}), which matches our theoretical results and is close to the optimal min-max rate $M^{-{2}/{5}}$ for regression with noisy observations up to a logarithmic factor.  If the velocities were not observed, the learning rate would be affected. In the right upper corner of Figure \ref{figPSTraj:1}, we see that the learning rate of the smoothed estimators is around $M^{-0.32}$ if we use the finite difference method to estimate the unobserved velocities, leading to an error in the velocities of size $O(10^{-2})$.

\begin{figure}[!htbp]
\centering     

\subfigure{\label{figPSTraj:1}\includegraphics[width=0.48\textwidth]{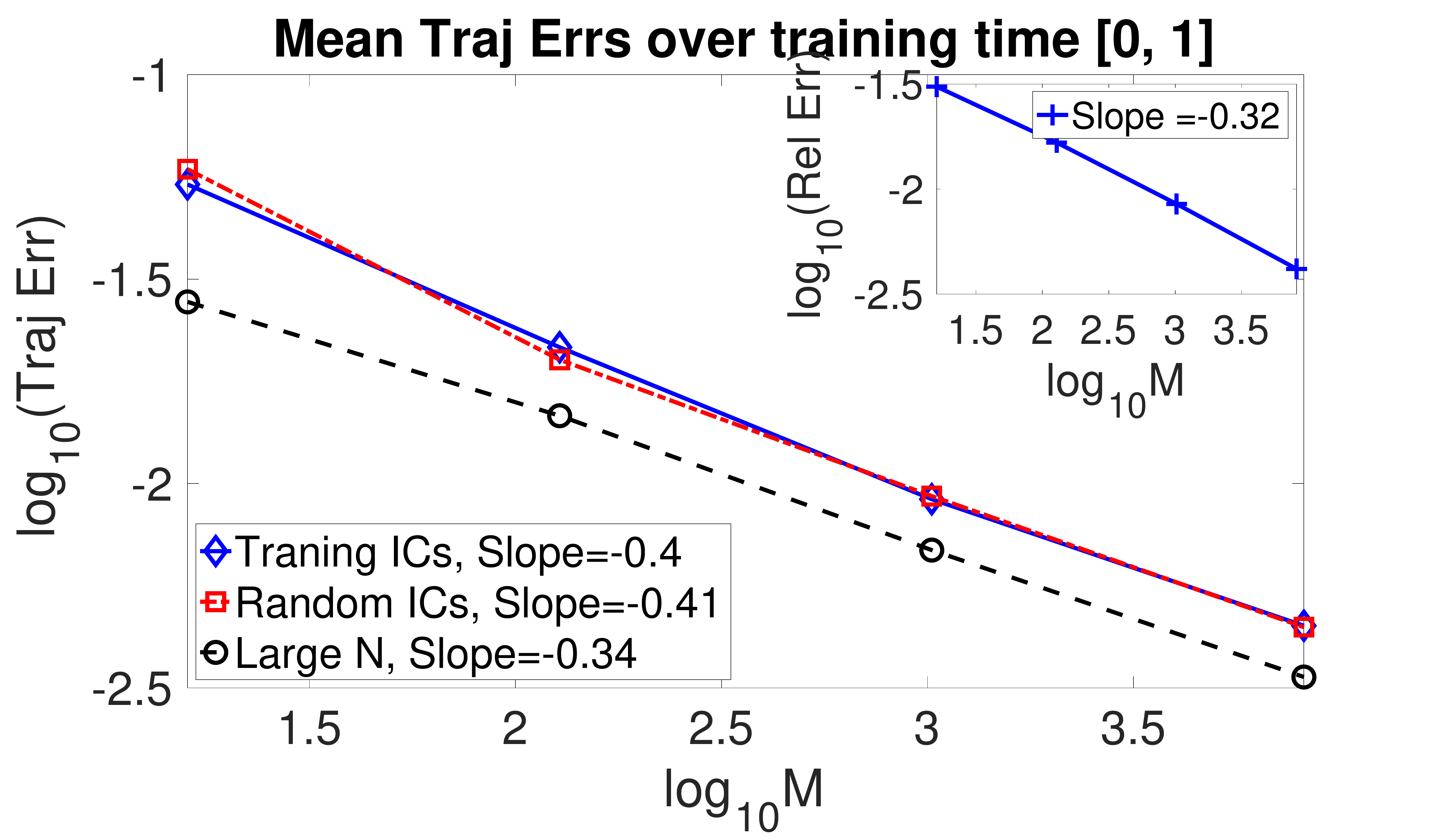}}
\subfigure{\label{figPSTraj:2}\includegraphics[width=0.48\textwidth]{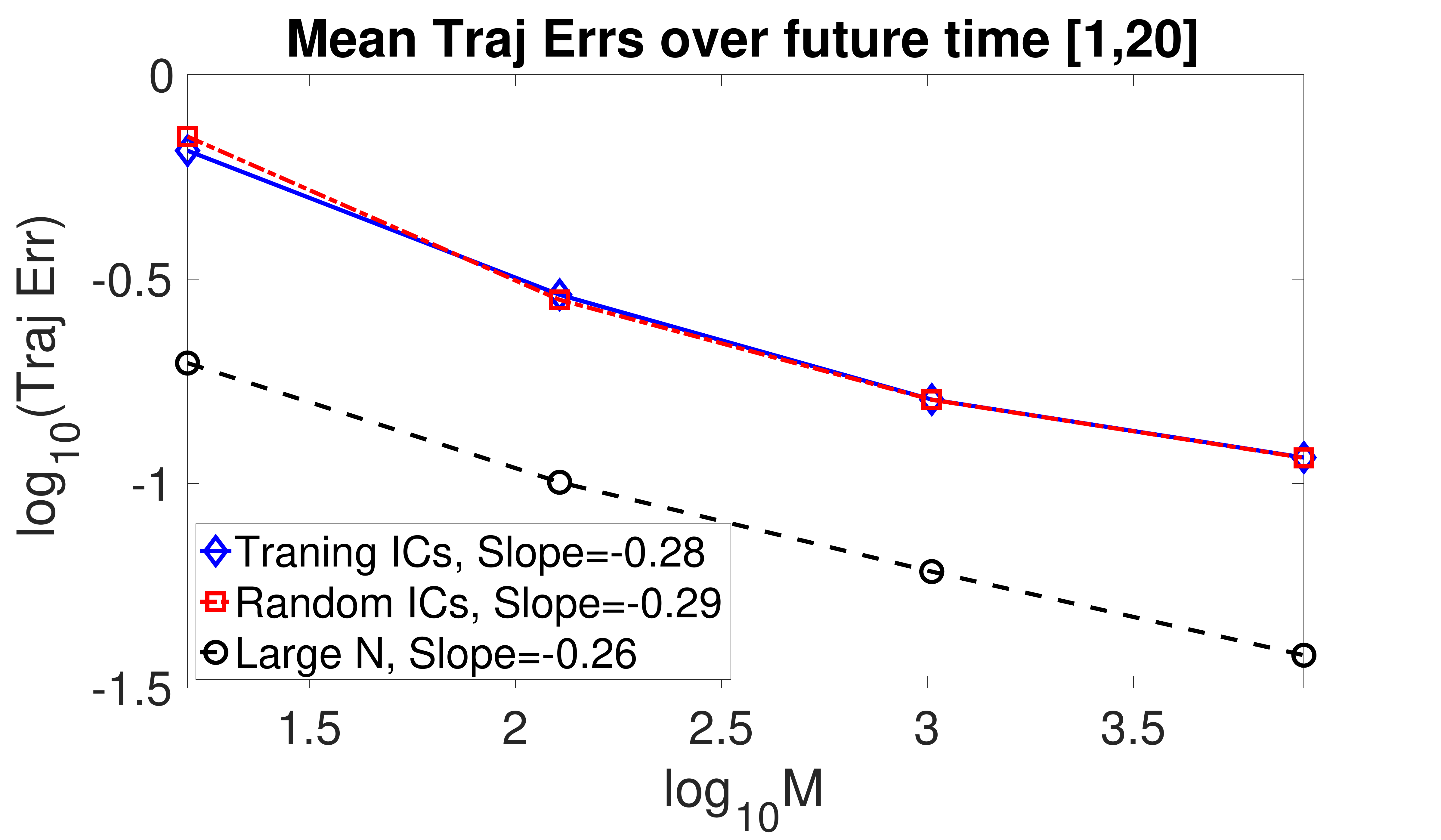}}

\caption{ \textmd{{ (Predator-Swarm Dynamics) Mean errors in trajectory prediction over 10 learning trials using  estimated interaction kernels obtained with different values of $M$: for initial conditions  in the training set (Training ICs), randomly drawn from $\mu_0$ (Random  ICs), and for  a system with $4N$ agents (Large $N$). {Left}: Errors over the training time interval [0,1]. {Right}: Errors over the future time interval [1,20]. { Upper right corner of left figure}: the learning rate of the smoothed learned interaction kernels. The decay rate of the mean trajectory prediction errors over the training time is faster than that of interaction kernels. On the prediction time interval, we still achieve good accuracy for trajectories, with a rate a bit slower than that of interaction kernels. }}}
\label{t:PS1_traj_err}
\end{figure}

\begin{figure}[!htbp]
\centering
\subfigure{\label{t:PS_Coercivityconstant}\includegraphics[width=0.48\textwidth]{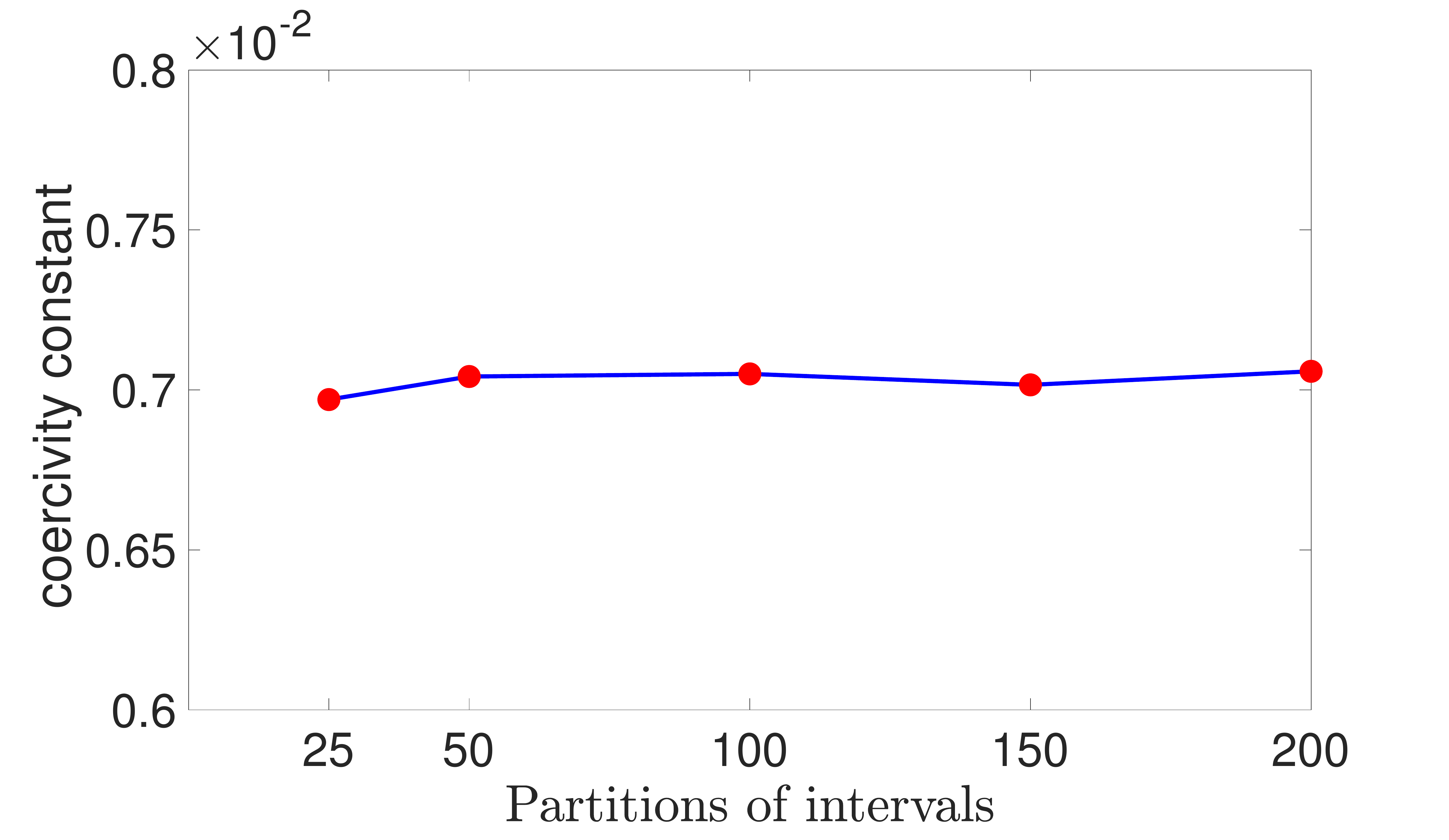}}
\subfigure{\label{t:PS_Convergencerate}\includegraphics[width=0.48\textwidth]{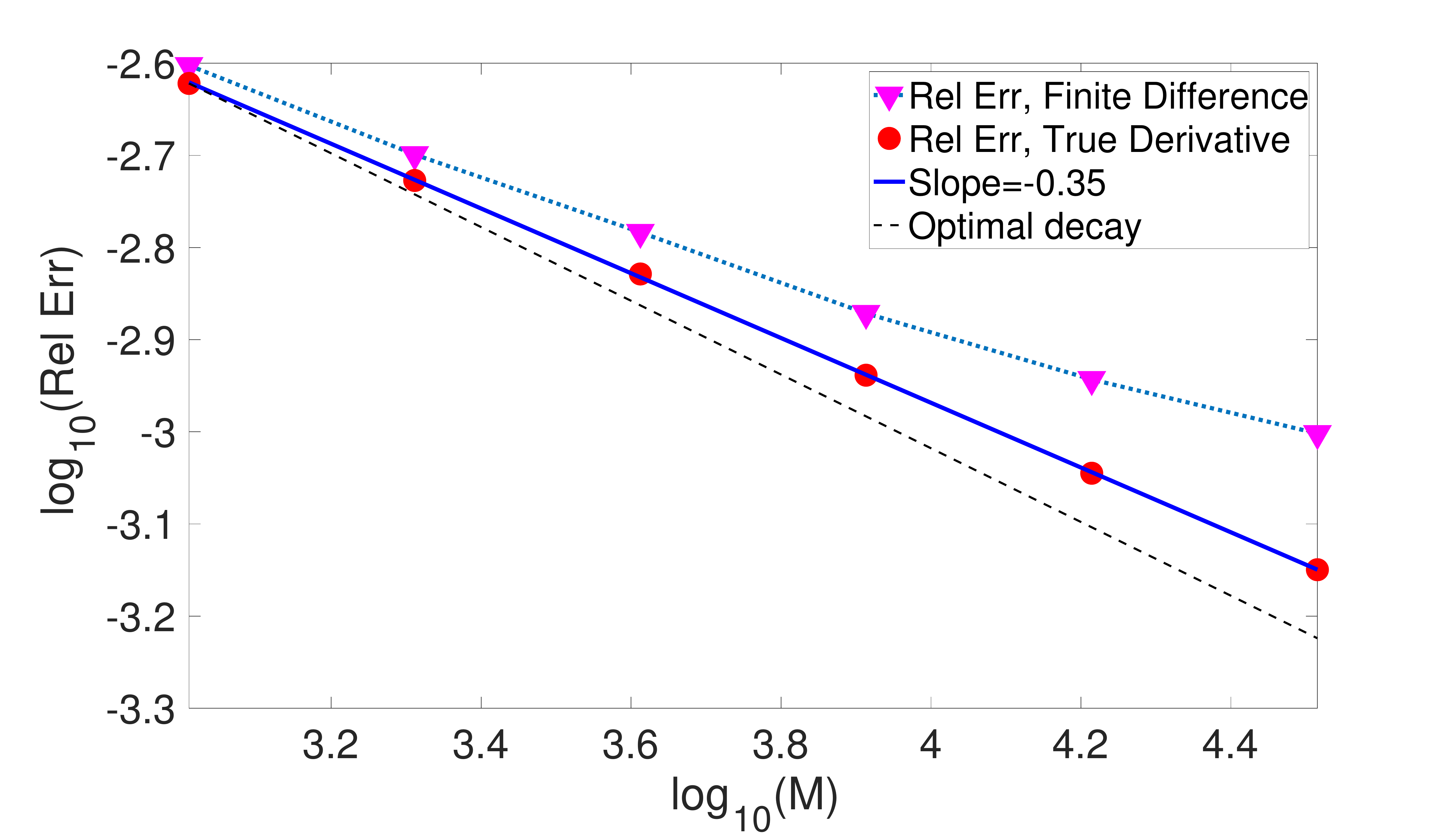}}
\caption{(PS). {Left}: The coercivity constant on $\mathcal{H}_n$ consisting of piecewise linear functions over $n$-uniform partitions of the support of $\brhoL$, computed from data consisting of $M=10^5$ trajectories.{Right}: 
 the relative $L^2(\brhoL)$ errors decay at a rate about $(M)^{-0.35}$, close to theoretical optimal min-max rate. 
}\label{t:PS1_Coercivity_Plot}
\end{figure}

\begin{figure}[!htbp]
\centering     
\subfigure{\label{figPSTraj_noisy:1}\includegraphics[width=0.96\textwidth]{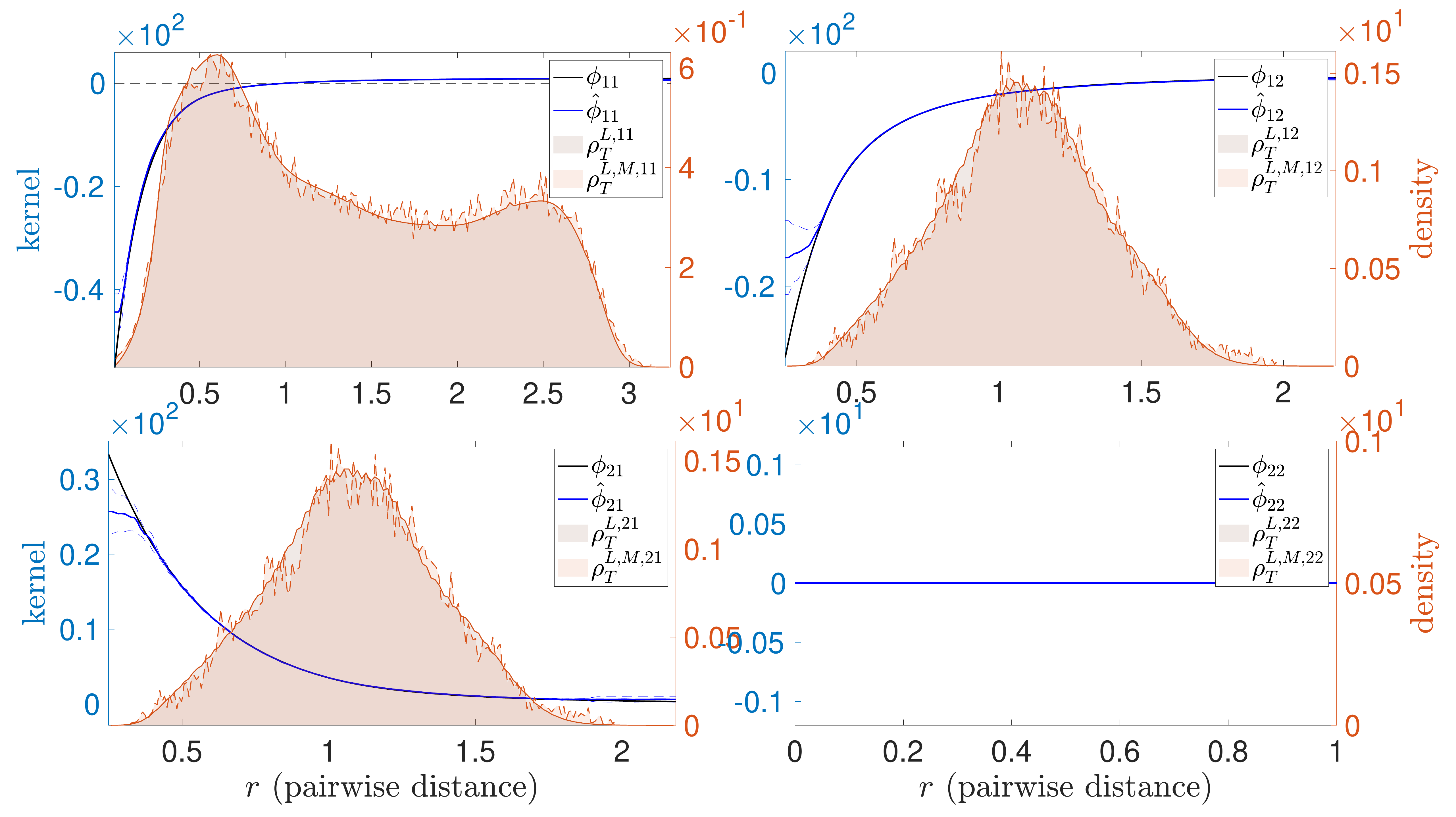}}
\caption{ \textmd{{ (Predator-Swarm Dynamics) Interaction kernels learned with Unif.$([-\sigma,\sigma])$ multiplicative noise, for $\sigma=0.1$, in the observed positions {\em{and observed velocities}}; here $M=16$,  with all the other parameters as in Table \ref{t:PS_system_params}. }}}\label{t:PS_kernel_err_noise}
\end{figure}

\begin{figure}[!htbp]
\centering     
\subfigure{\label{figPSTraj_noisy:2}\includegraphics[width=0.96\textwidth]{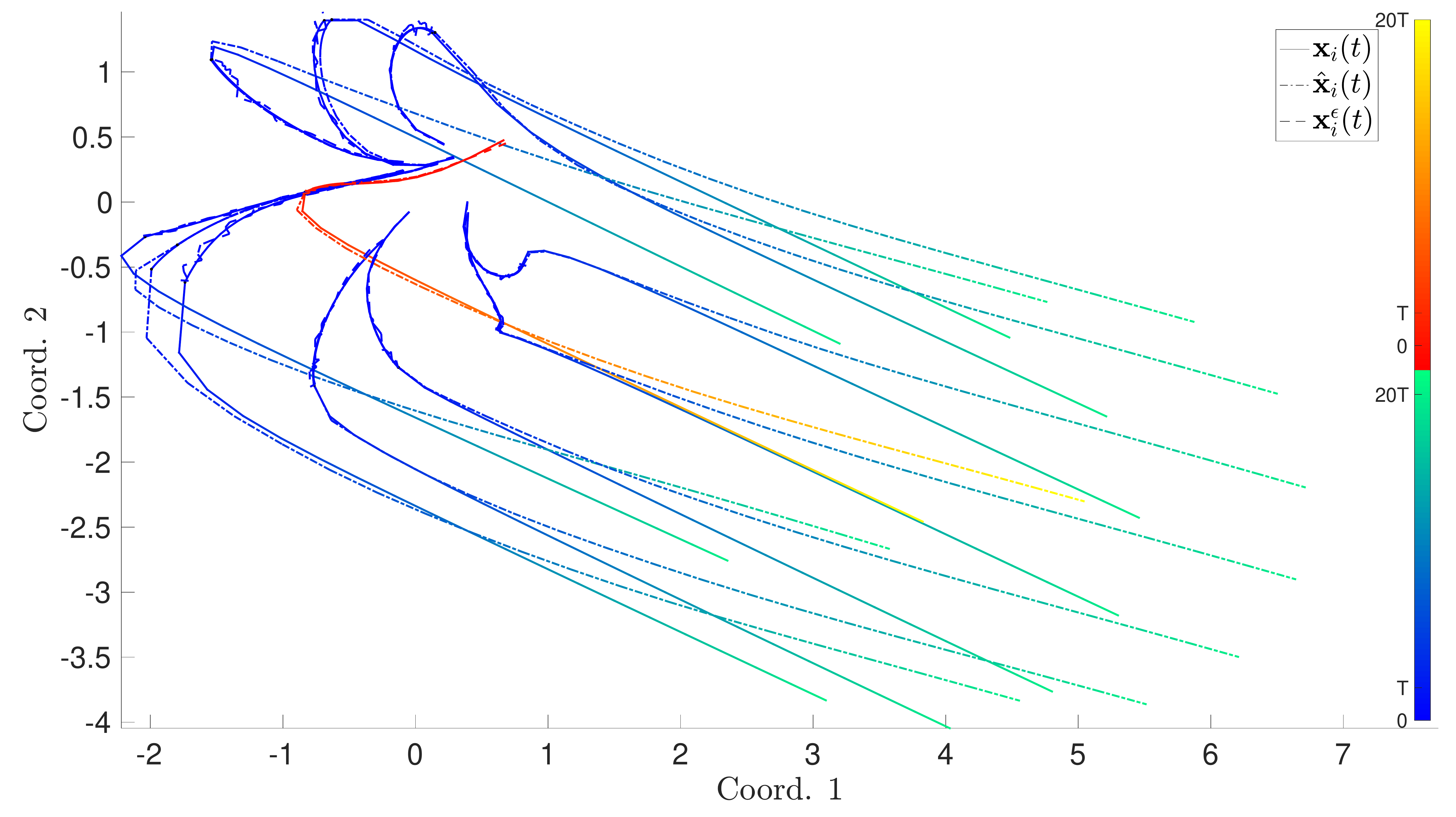}}
\caption{ \textmd{{ (Predator-Swarm Dynamics) One of the observed trajectories before and after being perturbed by the multiplicative noise drawn from Unif.$([-\sigma,\sigma])$ with $\sigma=0.1$. The solid lines represent the true trajectory; the dashed semi-transparent lines represent the noisy trajectory used as training data (together with noisy observations of the velocity); the dash dotted lines  are the predicted  trajectory learned from the noisy trajectory. }}}\label{f:PS_kernel_err_noise}
\end{figure}

\begin{figure}[!htbp]
\centering

\subfigure{\label{t:PS_Convergence_noise}\includegraphics[width=0.49\textwidth]{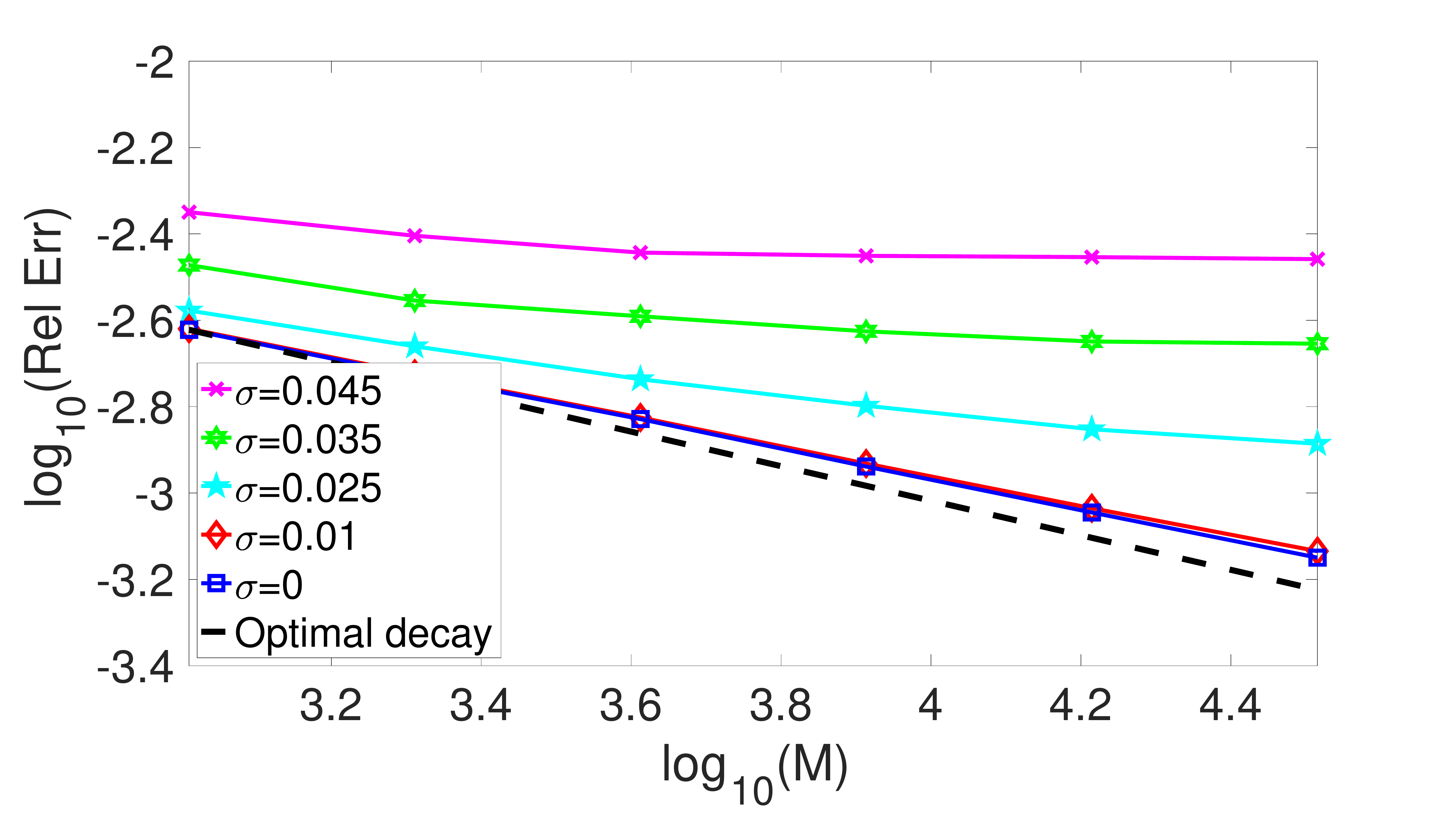}}
\subfigure{\label{t:PS_Convergence_noise_smooth}\includegraphics[width=0.49\textwidth]{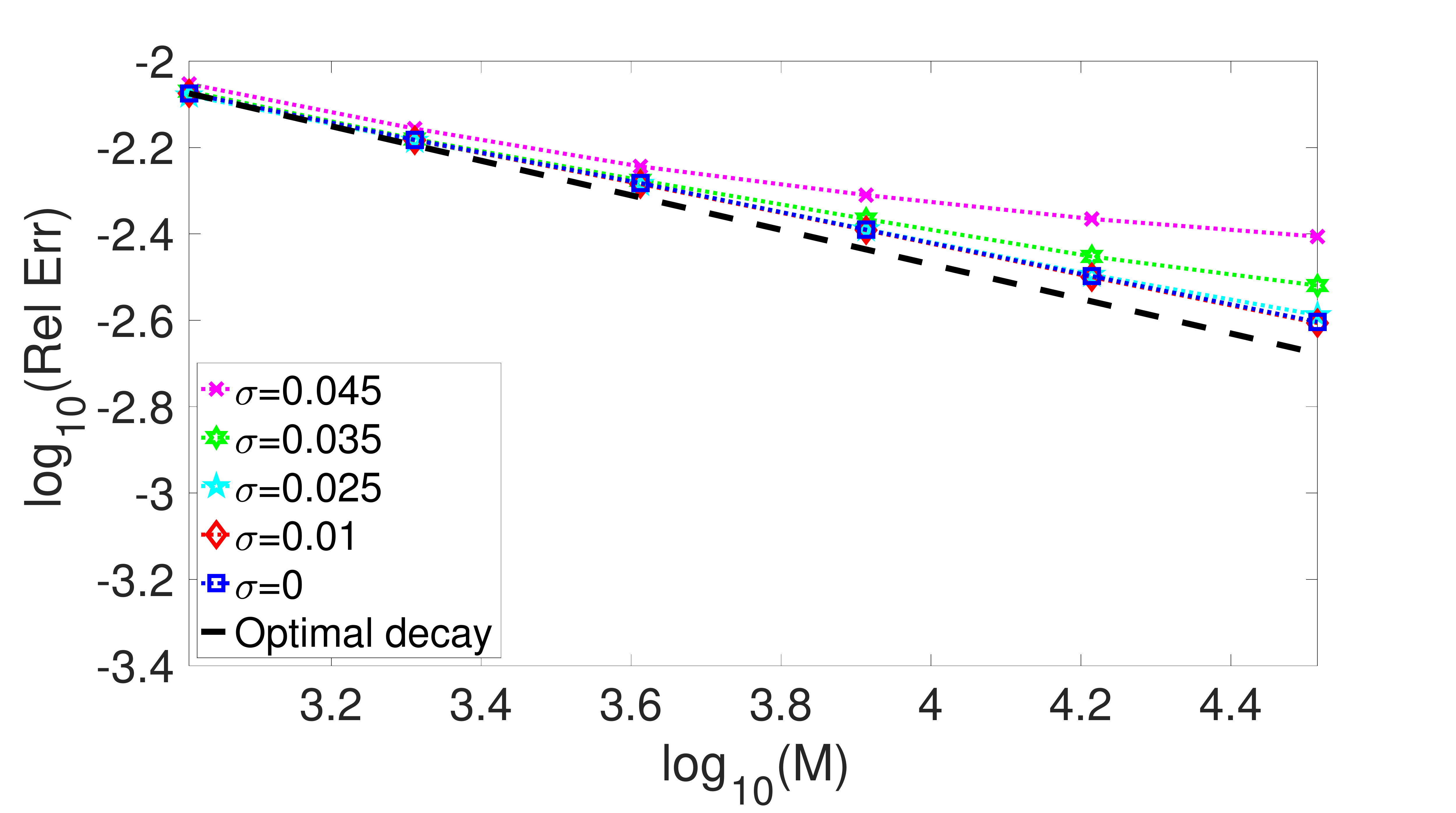}}

\caption{(Predator-Swarm Dynamics). The learning rates of estimators with different levels of multiplicative noise drawn from Unif.$([-\sigma,\sigma])$. The noise make the learning curve tend to be flat as its level increases. {Left}: Rates for estimators without smoothing. {Right}: Rates for smoothed estimators.  }\label{t:PS_Convergence_Plot_Noise} 
\end{figure}

\subsection {Heterogeneous particle dynamics}\label{MLJdescriptions}

We consider here another representative heterogeneous agent dynamics: a particle system with two types of particles (denoted by $\alpha$ and $\beta$) governed by Lennard-Jones type potentials (a popular choice for example to model atom-atom interactions in molecular dynamics and materials sciences). The general expression of the Lennard-Jones type potential is 
\[
\Phi(r) = \frac{p\epsilon}{(p-q)}\left[\frac{q}{p}\left(\frac{r_m}{r}\right)^{p}-\left(\frac{r_m}{r}\right)^{q} \right]
\]
where $\epsilon$ is the depth of the potential well, $r$ is the distance between the particles, and $r_m$ is the distance at which the potential reaches its minimum. At $r_m$, the potential function has the value $-\epsilon$.  The  $r^{-p}$ term, which is the repulsive term, describes Pauli repulsion at short ranges due to overlapping electron orbitals, and the $r^{-q}$ term, which is the attractive long-range term, describes attraction at long ranges (modeling van der Waals forces, or dispersion forces). Note that the corresponding Lennard-Jones type kernel $\intkernel(r)=\frac{\Phi'(r)}{r}$ is singular at $r=0$: we truncate it at $r_{\text{trunc}}$ by connecting it with an exponential function of the form $a\exp(-br^{12})$ so that it is continuous with a continuous derivative on $\mathbb{R}^+$.

\begin{figure}[!htbp]
\centering     
\subfigure{\label{figLJ:1}\includegraphics[width=0.96\textwidth]{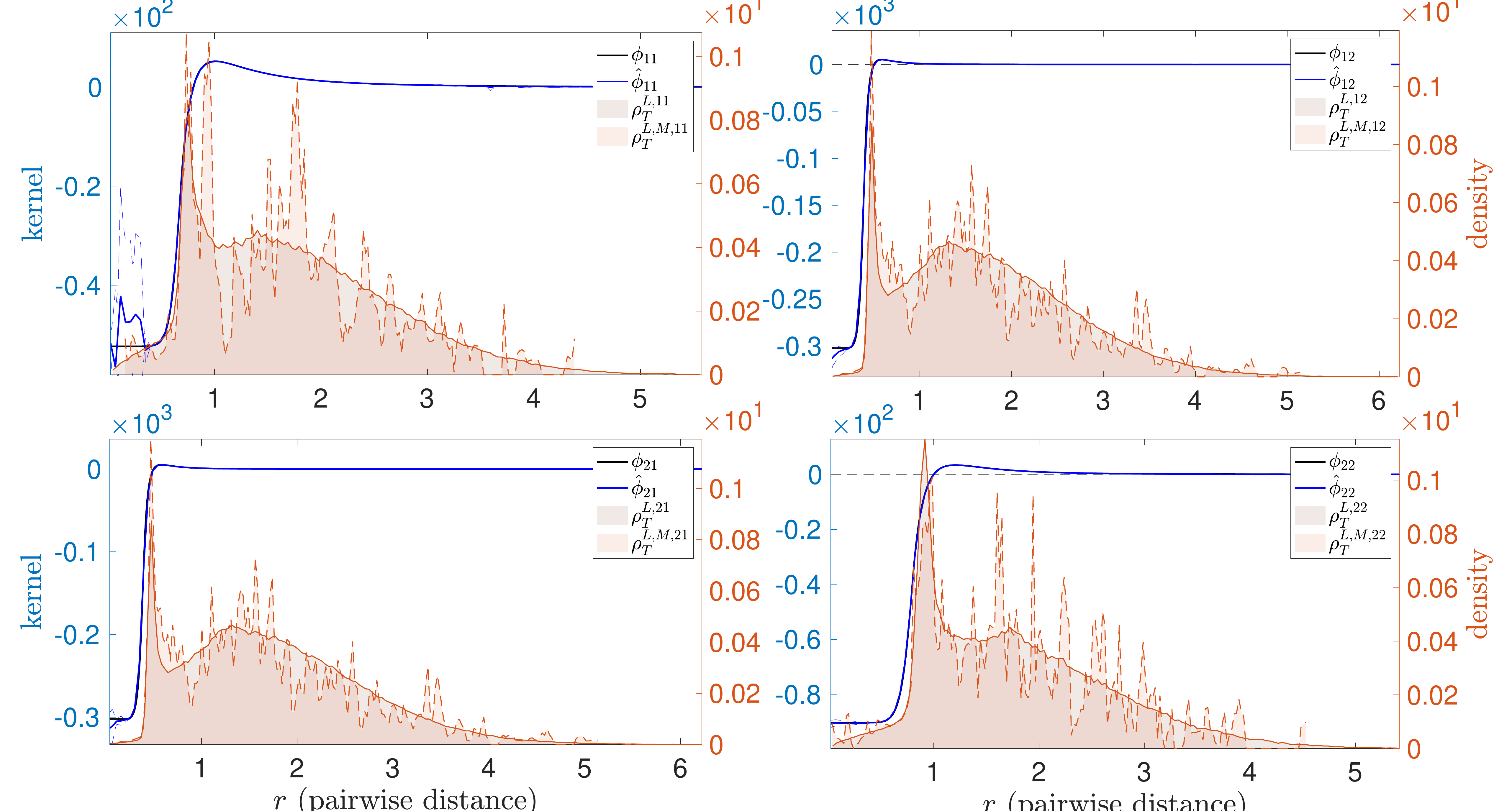}}
\subfigure{\label{figLJ:2}\includegraphics[width=0.96\textwidth]{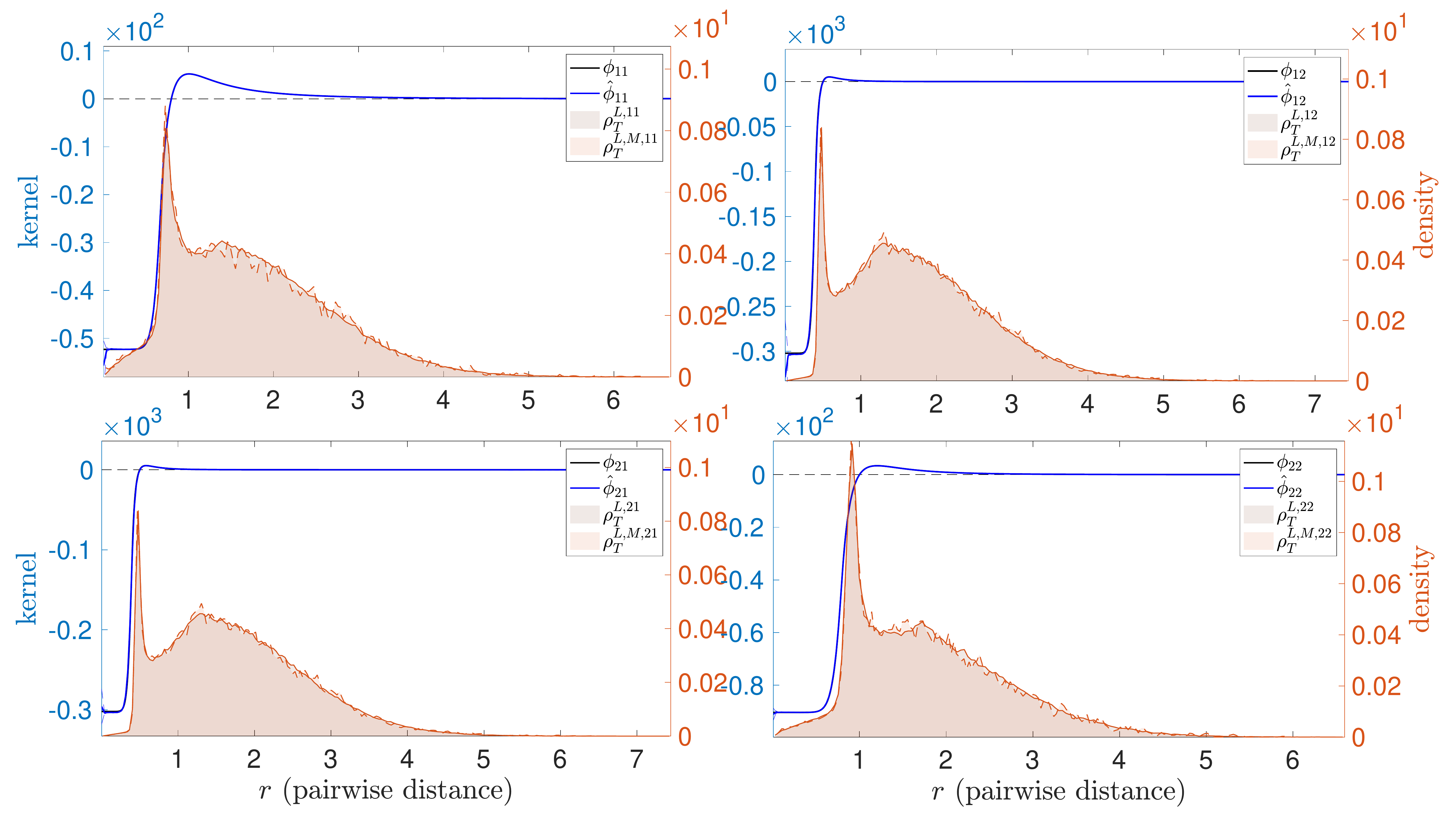}}
\caption{(Lennard Jones Dynamics) Comparison between true and estimated interaction kernels with $M=16$ (Top) and $M=512$ (Bottom). In black: the true interaction kernels. In blue: the learned interaction kernels using piecewise linear functions. The learned interaction kernels with $M=16$ match the true kernels very well except at the region near 0. For a larger $M=512$, the learned interaction kernels match more faithfully to the true interaction kernels at locations near 0; the standard deviation bars on the estimated interaction kernels become smaller and less visible. The relative errors  in $\mbf{L}^2(\brhoL)$ norm for the kernels are $(4\pm 2)\cdot 10^{-2}$  and $(1.2\pm 0.003.5)\cdot 10^{-2}$.  }
\label{t:LJ1_kernel_L100_noderivative}
\end{figure}

In our notation for heterogeneous systems, the set $C_1$ corresponds to $\alpha$-particles, and $C_2$ corresponds to $\beta$-particles. The intensity of interaction(s) between particles can be tuned with parameters, determining different types of mixture of particles. In the numerical simulations, we consider interaction kernels  $\phi_{1,1}$, $\phi_{1,2}$, $\phi_{2,1}$ and $\phi_{2,2}$ with parameters summarized in Table \ref{t:LJ_kernel_params}.

\begin{table}[H]
\centering
\begin{tabular}{|c | c | c | c | c | c | }
\hline 
kernels& $p$ &$q$ & $\epsilon$ & $r_m$ & $r_{\text{trunc}}$    \\ 
\hline 
 $\intkernel_{1,1}$&$4$ &1& $10$ & $0.8$ &0.68 \\
\hline
 $\intkernel_{1,2}$&$8$ &2& $1.5$ & $0.5$ & $0.4$\\
\hline
 $\intkernel_{2,1}$&$8$ &2& $1.5$ & $0.5$ & $0.4$ \\
\hline
$\intkernel_{2,2}$&$5$ &2& $5$ & $1$ & $0.8$\\
\hline
\end{tabular}
\caption{\textmd{\footnotesize{(LJ) Parameters for the Lennard Jones kernels} }}
\label{t:LJ_kernel_params}
\end{table}

\begin{table}[H]
\centering
\begin{tabular}{| c | c | c | c | c | c | c |c|c|}
\hline 
 $d$ &$N_{1}$ & $N_{2}$ &$M_{\rhoL}$& $L$ & $[t_1;t_L;t_f]$   & $\probIC$          & deg($\psi_{kk'}$) &$n_{kk'}$ \\ 
\hline 
 $2$ &5& $5$&$10^5$ & \tabincell{c}{100} & $[0;0.05;2]$ & $\mathcal{N}(\textbf{0},I_{20\times 20}) $& 1&$ 300(\frac{M}{\log M})^{\frac{1}{5}}$ \\
\hline
\end{tabular}
\caption{\textmd{\footnotesize{(LJ) Parameters for the system} }}
\label{t:LJ_system_params}
\end{table}

In the experiments, the particles are drawn i.i.d from standard Gaussian distribution in $\R^2$.  In this system, the particle-particle interactions are all short-range repulsions and long-range attractions. The short-range repulsion force prevents the particles to collide and long-range attractions keep the particles in the flock. Both the $\alpha$-particles and $\beta$-particles form the crystal-like clusters. Moreover, the $\alpha$-$\beta$ potential function is the same with the $\beta$-$\alpha$ potential function. Both of them have the smallest $r_m$ and relative large attractive force (see Table \ref{t:LJ_kernel_params}) when an $\alpha$-particle is far from a $\beta$-particle. As a result, the attraction force between two different types of particles will force them to mix homogeneously.

Since the system evolves to equilibrium configurations very quickly, we observe the dynamics up to a time $t_L$ which is a fraction of the equilibrium time. Since the particles only explore a bounded region due to the large-range attraction,  $\smash{\brhoL}$ is essentially supported on a bounded region (see the histogram background of Figure \ref{t:LJ1_kernel_L100_noderivative}), on which the interaction kernels are in the 2-H\"older space. Therefore, our learning theory is still applicable in this case.  Similar to the learning in Predator-Swarm dynamics, the estimator of each $\intkernel_{kk'}$ belongs to a piecewise linear function space over a uniform partition of $n$ intervals. The observation and learning parameters are summarized in Table \ref{t:LJ_system_params}.   As reported in Figure \ref{t:LJ1_kernel_L100_noderivative}, with rather small $M$, the learned interaction kernels $\widehat\bintkernel$ approximate the true interaction kernels $\bintkernel$ very well in the regions with large $\brhoL$, i.e., regions with an abundance of observed values of pairwise distances to reconstruct the interaction kernels.  The reasons for the error near 0 are similar to those for Predator-Swarm dynamics. 

\begin{figure}[H]
\centering     
\subfigure{\includegraphics[width=0.96\textwidth]{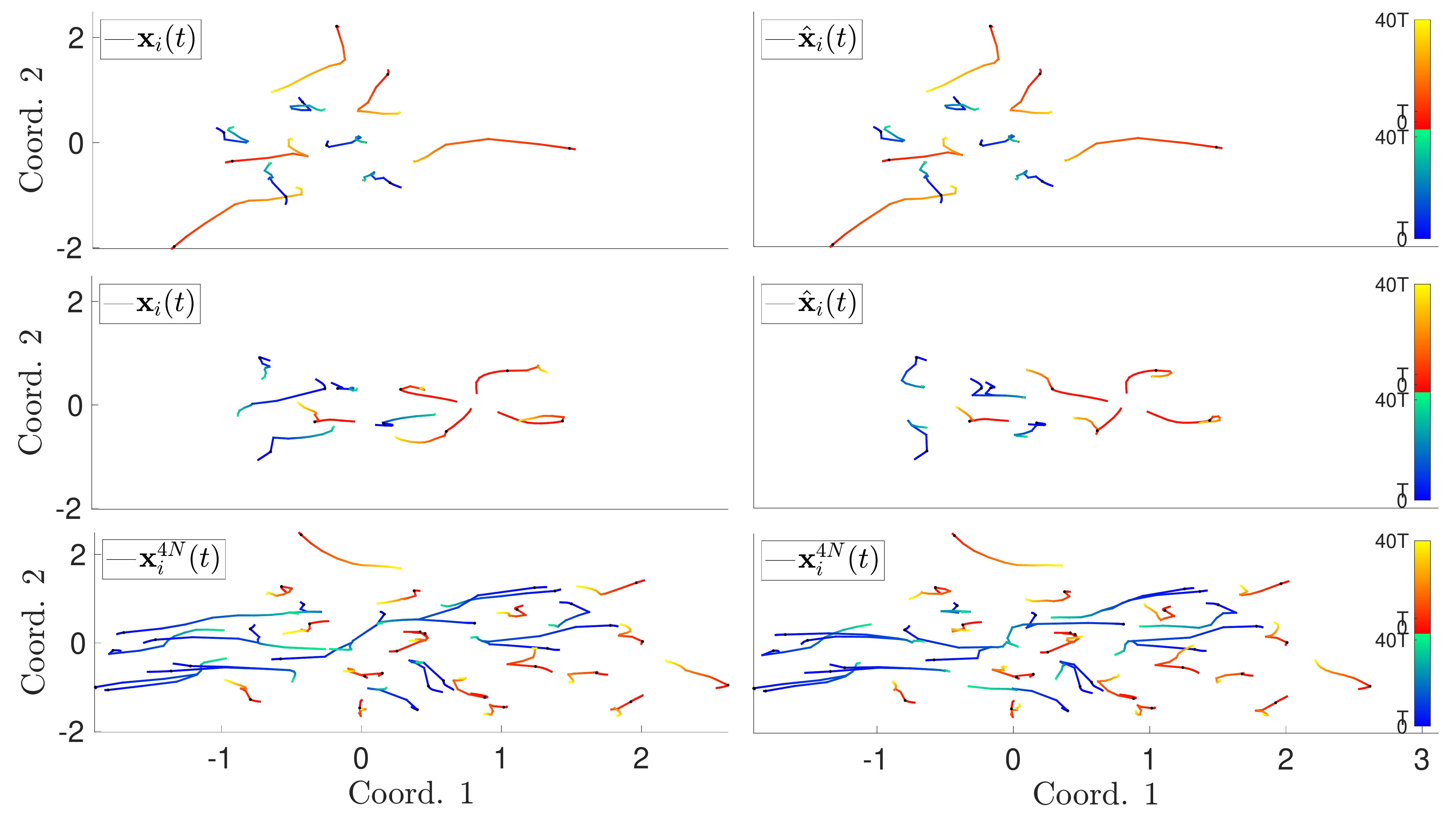}}
\caption{ \textmd{{ (Lennard Jones Dynamics)  $\bX(t)$ and $\hat \bX(t)$ obtained with $\bintkernel$ and $\widehat \bintkernel$  learned with  $M=16$ for an initial condition in the training data (Top row) and a new initial condition random drawn from $\mu_0$ (Middle row). The black dot at $t = 0.05$ divides the  training time interval $[0, 0.05]$ from the prediction time interval [0.05, 2]. Bottom row:  $\bX(t)$ and $\hat \bX(t)$ obtained with $\intkernel$ and $\hat \intkernel$  learned from $M=16$ trajectories respectively, for dynamics with larger $N_{new} = 4N$, over  a set of initial conditions. We achieve small errors, in average, in particular predicting the time and shape of particle aggregation. The means of trajectory prediction errors are in Figure \ref{t:LJ1_traj_err}.}}}\label{figLJTrajM16}
\end{figure}

Figure \ref{figLJTrajM16} shows that the learned interaction kernels not only produce an accurate approximation of transient behaviour of particles on the training time interval $[t_1,t_L]$, but also the aggregation of particles over a much larger prediction time interval $[t_L,t_f]$.

\begin{figure}[H]
\centering     
\subfigure{\label{figLJTraj:1}\includegraphics[width=0.48\textwidth]{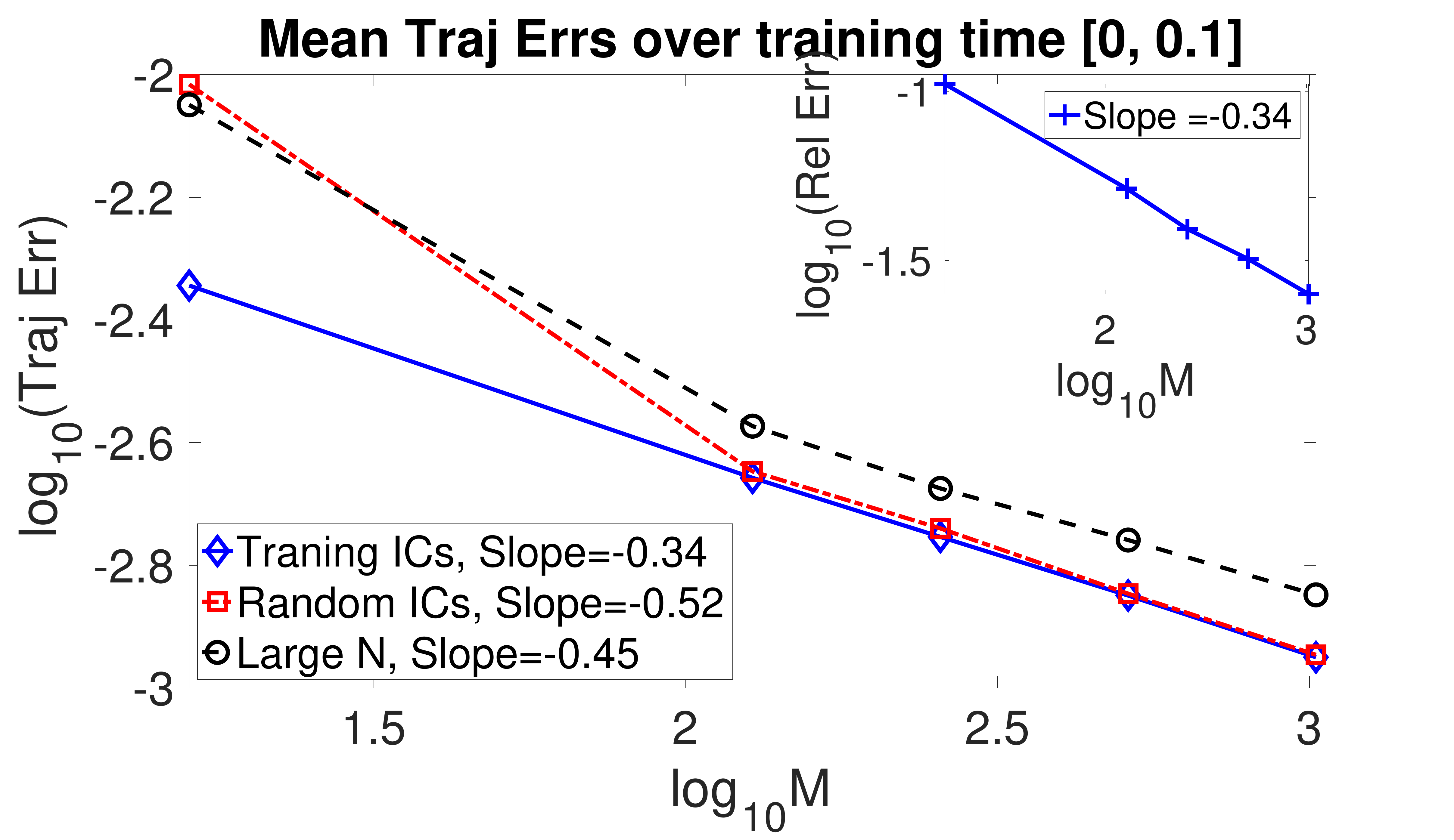}}
\subfigure{\label{figLJTraj:2}\includegraphics[width=0.48\textwidth]{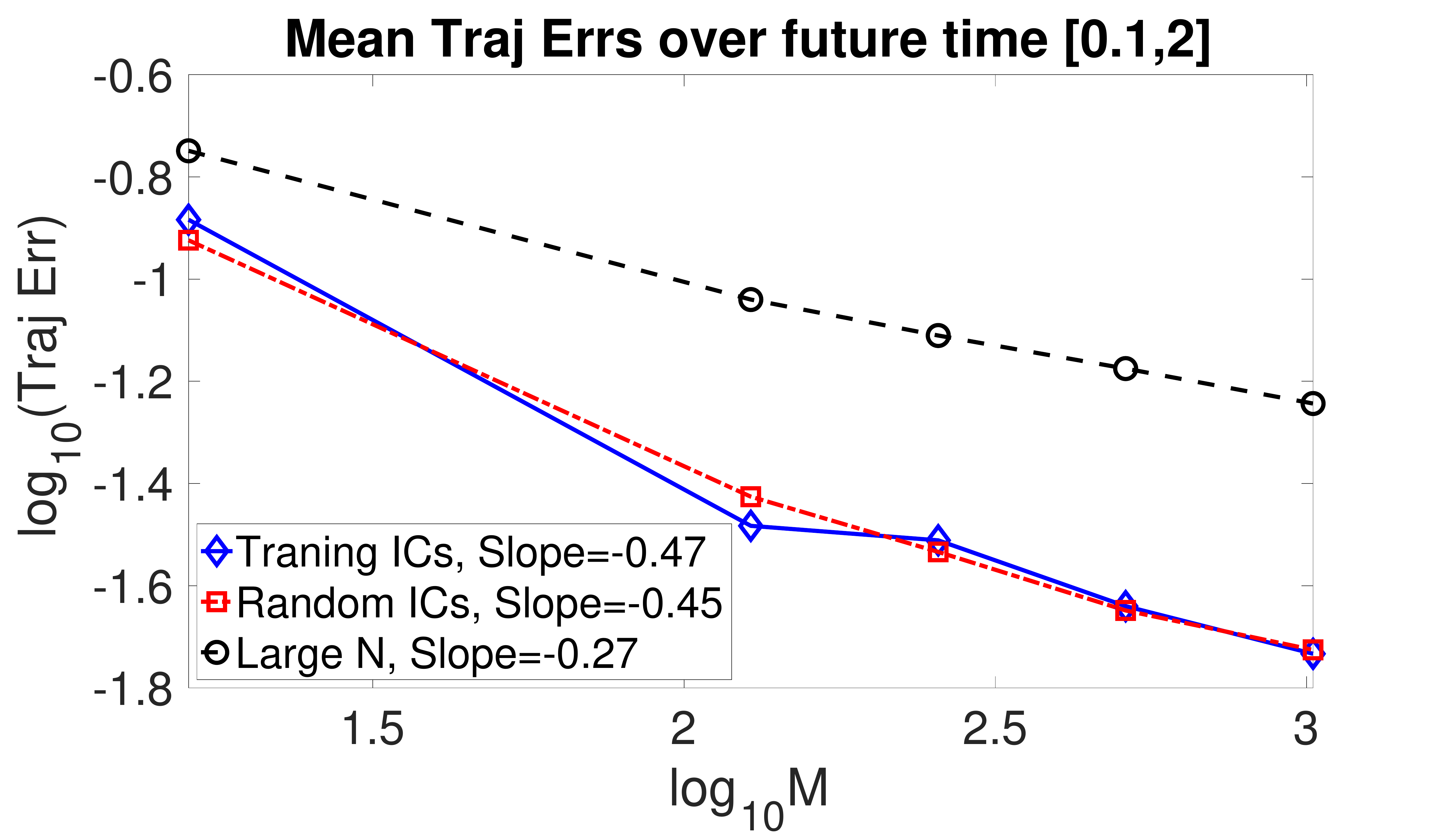}}
\caption{ \textmd{{ (Lennard Jones Dynamics) Mean errors in trajectory prediction over 10 learning trails using  estimated interaction kernels obtained with different values of $M$: for initial conditions  in the training set (Training ICs), randomly drawn from $\mu_0$ (Random  ICs), and for  a system with $4N$ agents (Large $N$). {Left}: Errors over the training time interval [0, 0.05]. {Right}: Errors over the future time interval [0.05, 2]. { Upper right corner of left figure}: the learning rate of  kernels (used to predict dynamics). The learning curves of trajectory prediction errors  over the training time interval are almost the same with those of interaction kernels. On the prediction time interval, we still achieve a good accuracy of trajectory prediction and just a slightly slower rate for predicting trajectories in all cases.
}}}\label{t:LJ1_traj_err}
\end{figure}

We summarize the mean trajectory prediction errors over 10 learning trials in Figure \ref{t:LJ1_traj_err}. It shows that the learning rate of the trajectory errors over the training time interval $[0,0.05]$ is the same with the learning rate of the kernels.  For the prediction time interval $[0.05, 2]$,  our learned interaction kernels still produced very accurate approximations of true trajectories in all cases, as demonstrated in Figure \ref{figLJTraj:2}.

We then compute   the minimal eigenvalue  of $A_{L,\infty,\bhypspace}$, inspired by Proposition  \ref{firstordersystem:coercivityconstant}.  In this case, $A_{L,\infty,\bhypspace}$ is block diagonal:   one block for learning $\intkernel_{1,1}$ and $\intkernel_{1,2}$ from velocities of $\alpha$-particles, and the other block for learning  $\intkernel_{2,1}$ and $\intkernel_{2,2}$ from velocities of the $\beta$-particles.   We display the minimal eigenvalues for each block in Figure \ref{t:LJ_Coercivityconstant}.  We see that the minimal eigenvalue of the type 1 block matrix  stay around $8.7\cdot10^{-2}$ and the type 2 block matrix stay around $8.9\cdot10^{-2}$  as the partitions get finer, suggesting a positive coercivity constant over $\bL^2(\brhoL)$.

\begin{figure}[!htbp]
\centering
\subfigure{\label{t:LJ_Coercivityconstant}\includegraphics[width=0.48\textwidth]{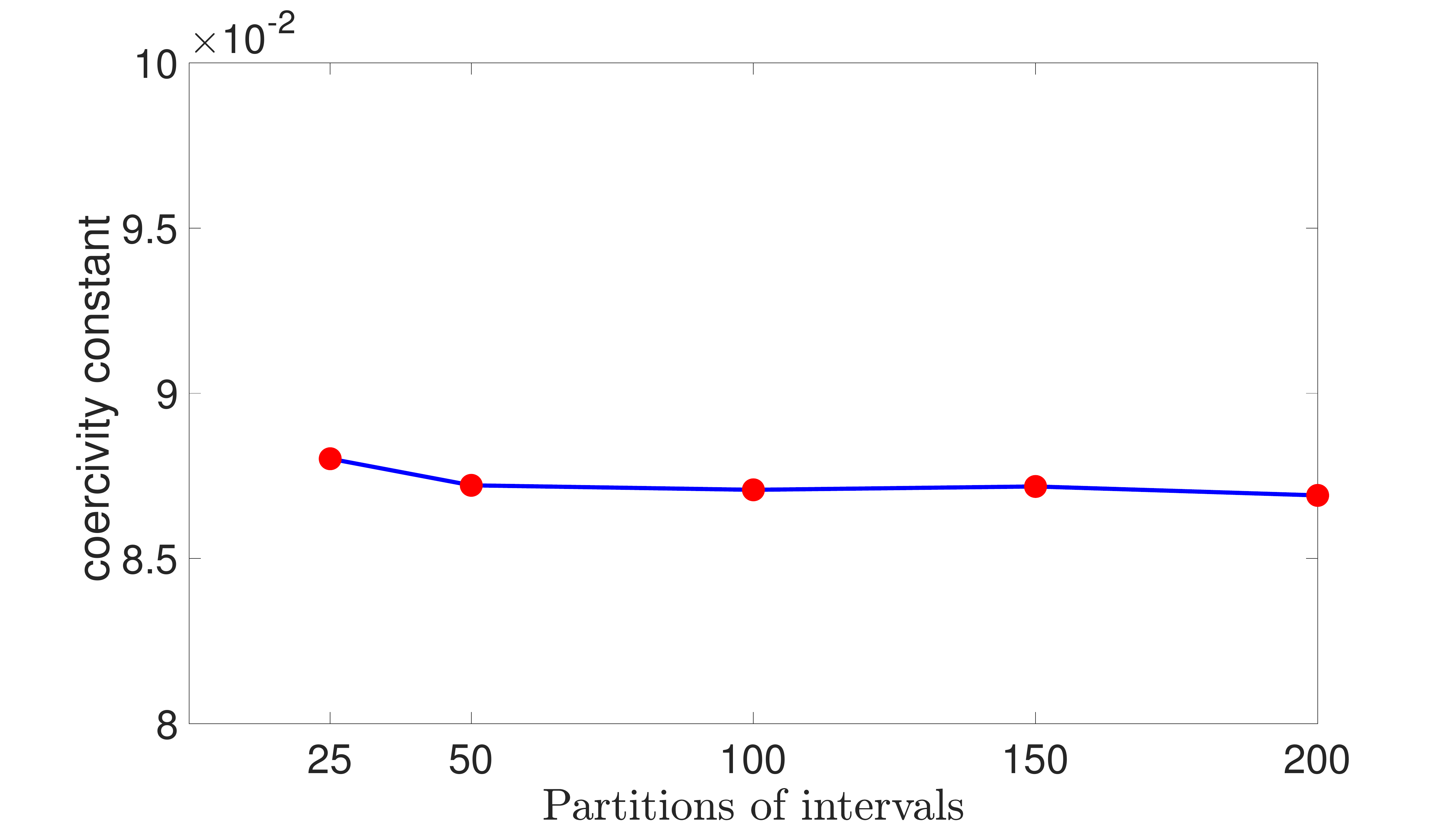}}
\subfigure{\label{t:LJ1_Convergence_Derivative}\includegraphics[width=0.48\textwidth]{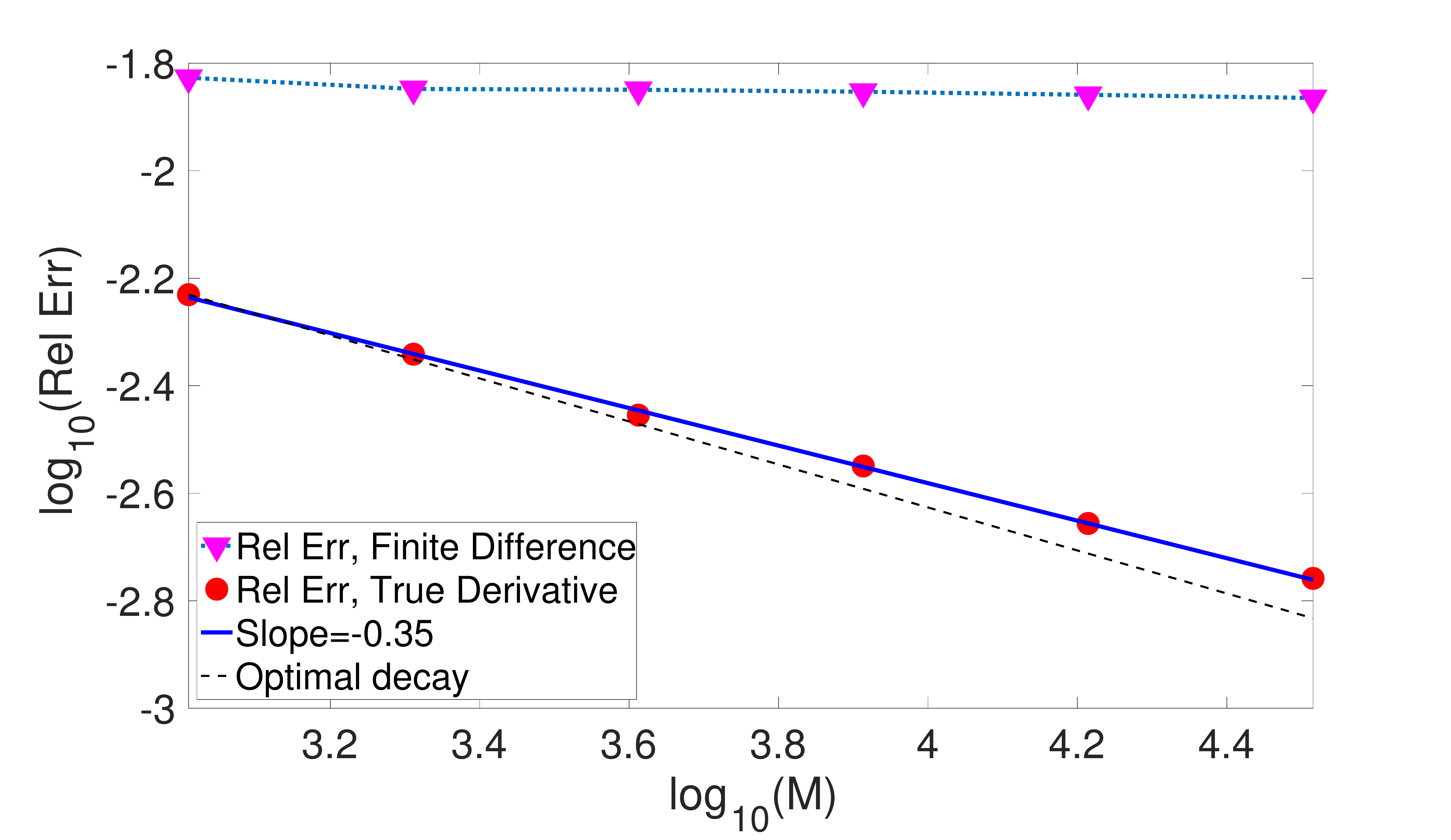}}
\caption{{Left}: 
 The coercivity constant on $\mathcal{H}_n$ consisting of piecewise linear functions over $n$-uniform partitions of the support of $\brhoL$, computed from data consisting of $M=10^5$ trajectories.
 {Right}: the relative $L^2(\brhoL)$ errors decay at a rate $M^{-0.35}$, close to theoretical optimal min-max rate $M^{-0.4}$ up to a logarithmic factor (shown in the black dot line) as in Theorem \ref{t:firstordersystem:thm_optRate}.   
}\label{t:LJ1_Convergence_Plot} 
\end{figure}
 
 We  choose the partition number $n_{kk'}$  for learning $\intkernel_{kk'}$ according to   Theorem \ref{t:firstordersystem:thm_optRate}. Given true velocities,  we obtain a learning rate for $\smash{\|\hat \bintkernel(\cdot)\cdot-\bintkernel(\cdot)\cdot\|_{L^2(\brhoL)}}$ around $M^{-0.35}$ (see right column of Figure \ref{t:LJ1_Convergence_Plot}), which is close to the theoretical optimal min-max rate $M^{-{2}/{5}}$. The error in finite difference approximation of the velocities  affects the learning rate, as is demonstrated in the right upper corner of Figure \ref{figLJTraj:1}, where the learning curves tends to flatten. 

\begin{figure}[!htbp]
\centering     
\subfigure{\label{figMLJTraj_noisy:1}\includegraphics[width=0.96\textwidth]{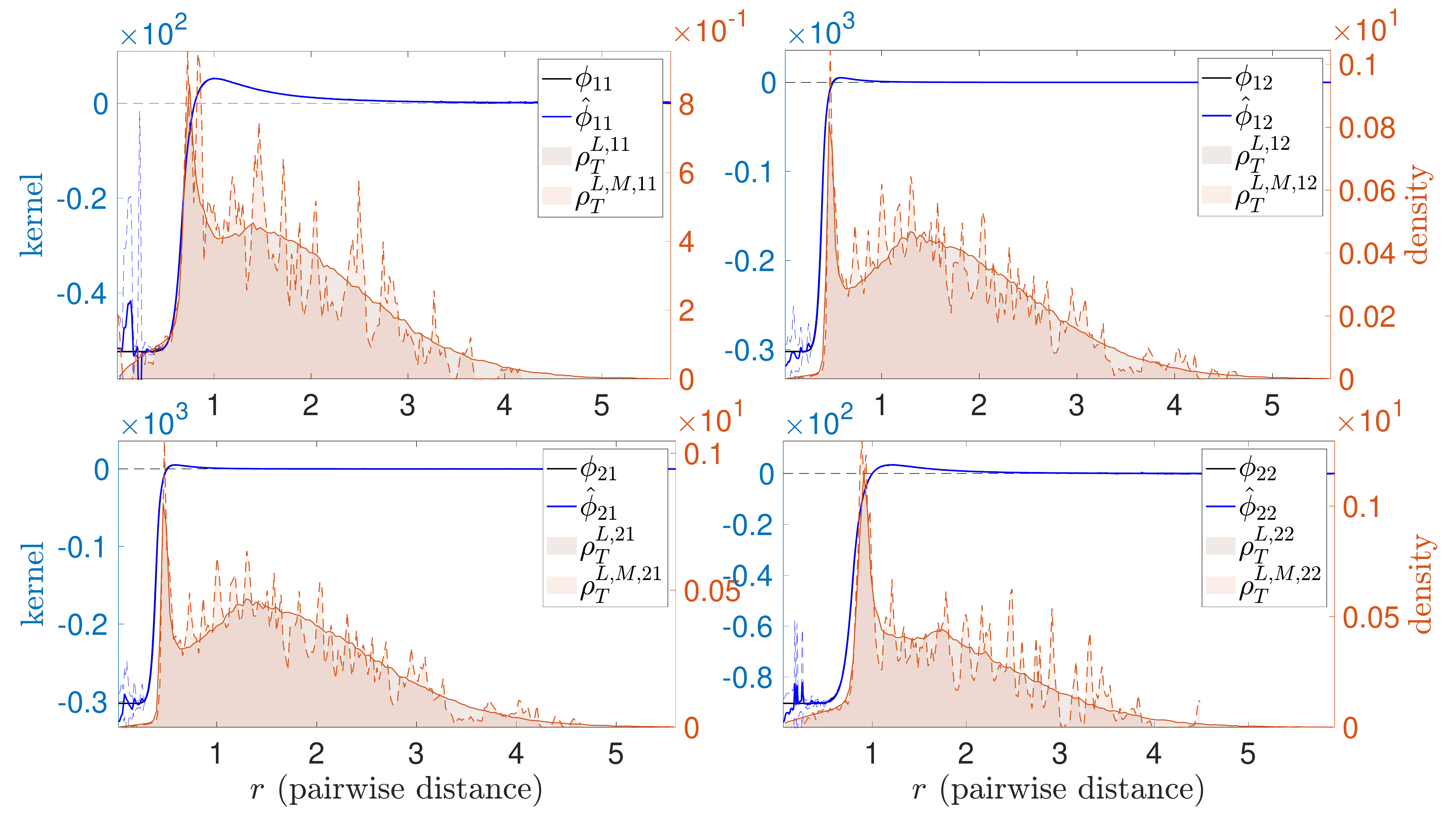}}
\caption{ \textmd{{ (Lennard Jones Dynamics) Interaction kernels learned with Unif.$([-\sigma,\sigma])$ additive noise, for $\sigma=0.02$, in the observed positions {\em{and observed velocities}}; here $M=16$,  with all the other parameters as in Table \ref{t:LJ_system_params}. }}}\label{t:MLJ_kernel_err_noise}
\end{figure}

\begin{figure}[!htbp]
\centering     

\subfigure{\label{figLJTraj_noisy:1}\includegraphics[width=0.96\textwidth]{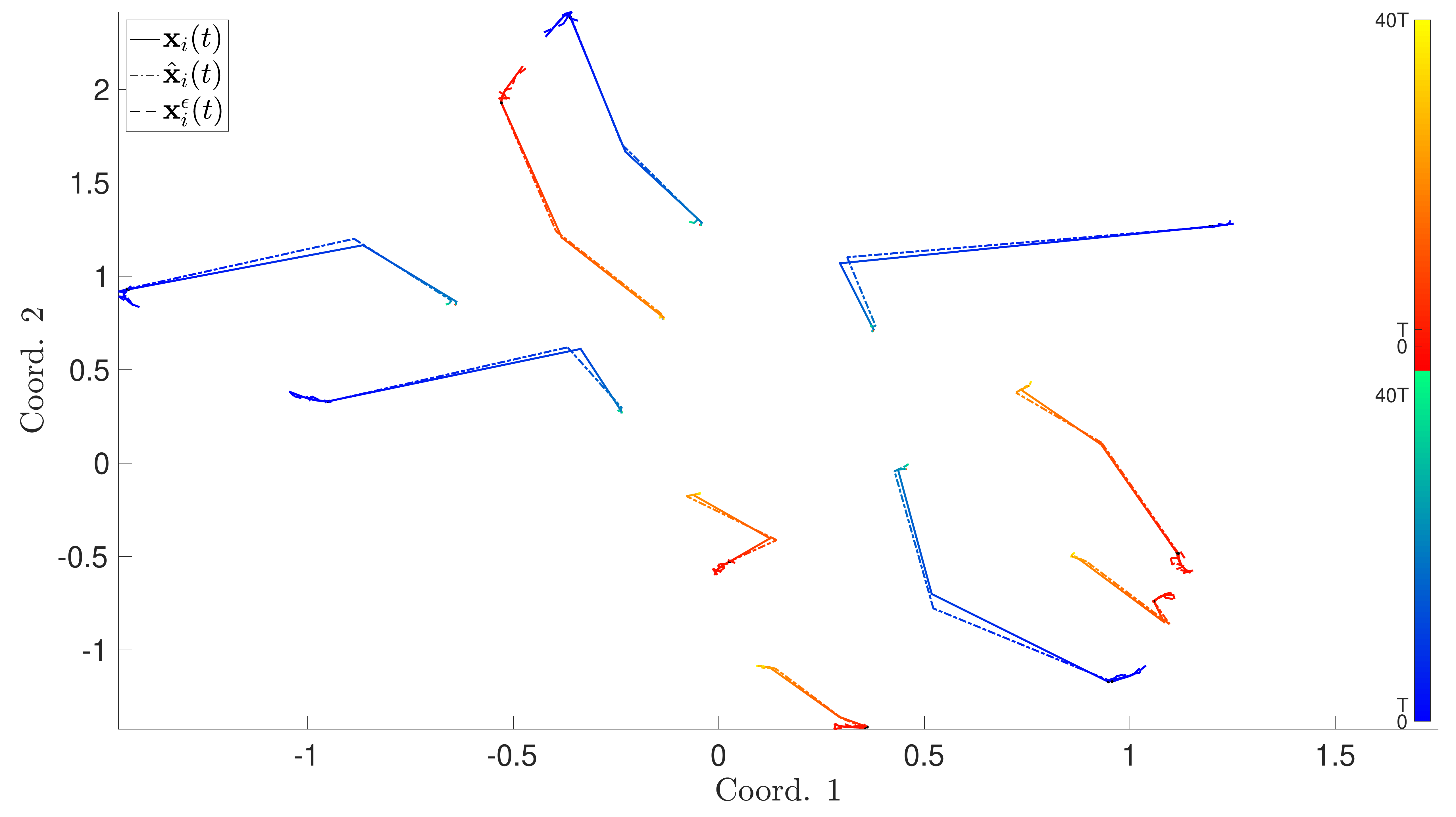}}
\caption{ \textmd{{ (Lennard Jones Dynamics) One of the observed trajectories before and after being perturbed by the additive noise drawn from Unif.$([-\sigma,\sigma])$ with $\sigma=0.02$. The solid lines represent the true trajectory; the dashed semi-transparent lines represent the noisy trajectory used as training data (together with noisy observations of the velocity); the dash dotted lines  are the predicted  trajectory learned from the noisy trajectory. }}}\label{t:LJ_kernel_err_noise}
\end{figure}

\begin{figure}[!htbp]
\centering
\subfigure{\label{t:LJ_Convergence_noise}\includegraphics[width=0.49\textwidth]{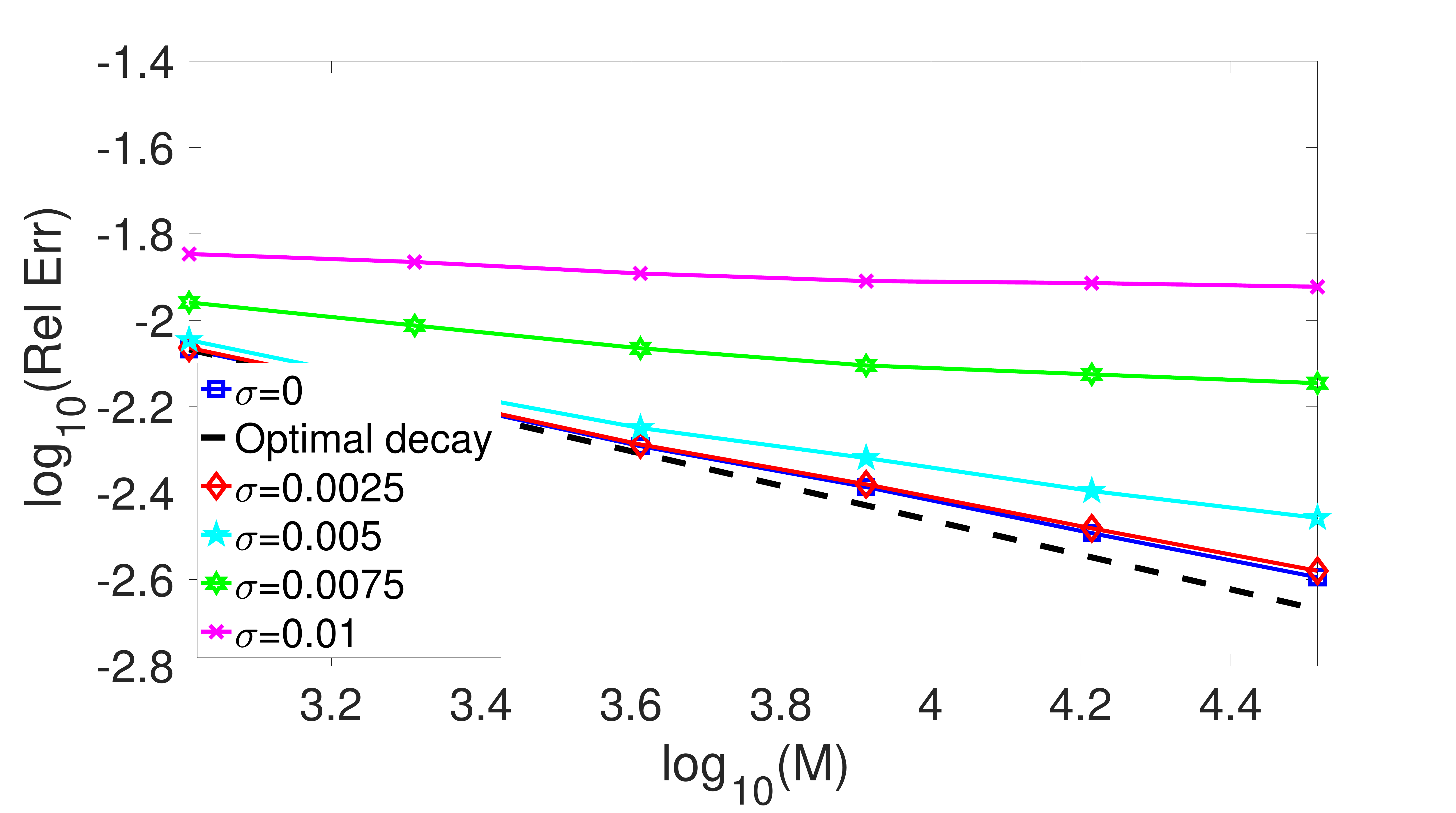}}
\subfigure{\label{t:LJ_Convergence_noise_smooth}\includegraphics[width=0.49\textwidth]{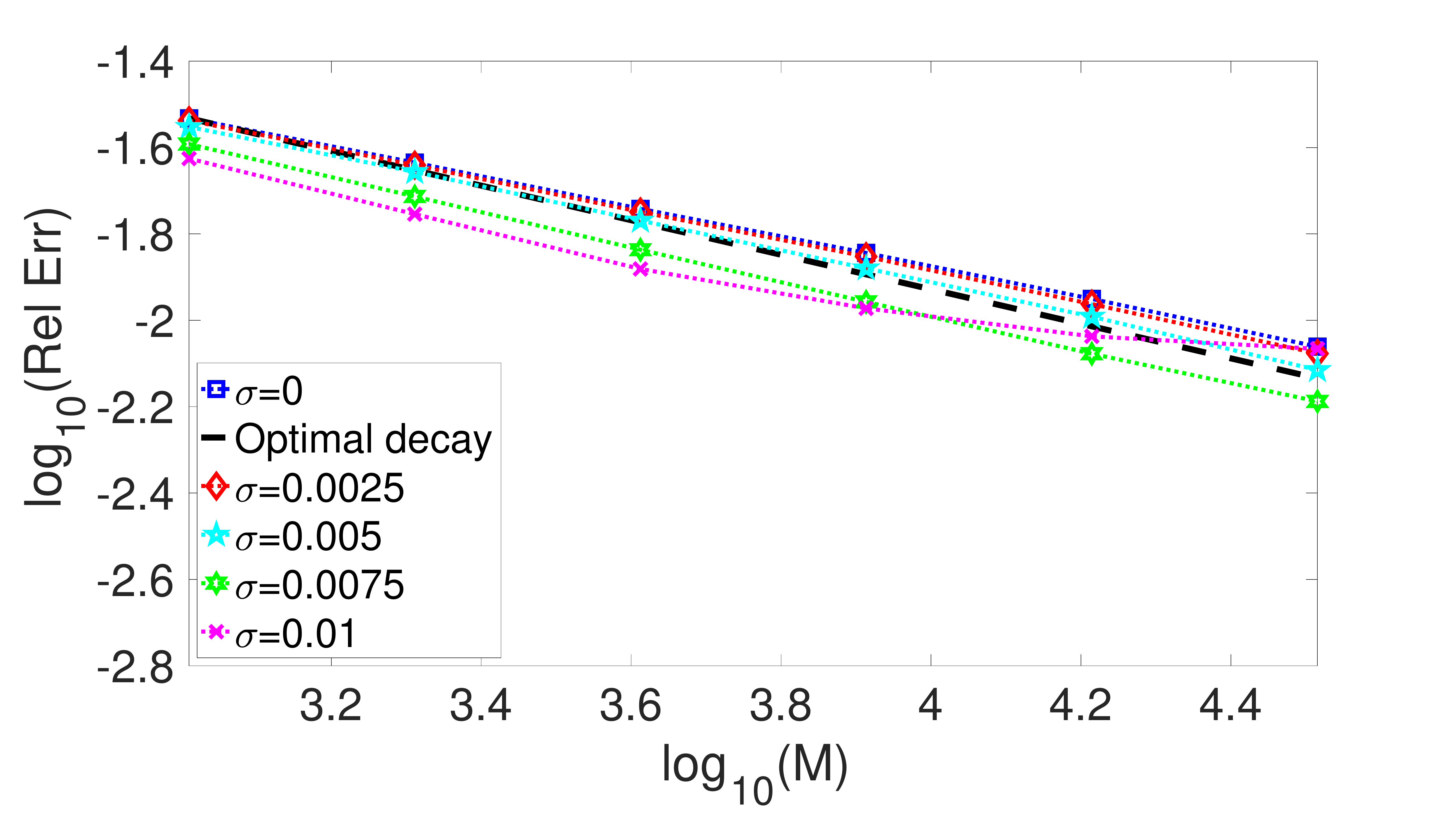}}

\caption{(Lennard Jones Dynamics). The learning rates of estimators with different levels of additive noise drawn from Unif.$([-\sigma,\sigma])$.  {Left}: Rates for estimators without smoothing. {Right}: Rates for smoothed estimators.  }\label{t:LJ_Convergence_Plot_Noise} 
\end{figure}

\subsection{Summary of the numerical experiments}
\begin{itemize}
\item Short time observations don't prevent us from learning the interaction kernels. The randomness of initial conditions enables the agents to explore large regions of state space, and in the space of pairwise distance in a short time. The estimated interaction kernels approximate well in the regions where $\brhoL$ is large, i.e. regions with an abundance pairwise distances to reconstruct the interaction kernels.  As the number of trajectories increases, we obtain more faithful approximations of the true interaction kernels, agreeing with the consistency Theorem \ref{main:consistency}. We also see that our estimators, even learned from a small amount of data sampled from a short period of transient dynamics,  not only can predict the dynamics on the training time interval $[t_1,t_L]$ but also produce accurate predictions of large time behaviour of the system. 

\item The decay rate of the mean trajectory prediction errors over the training time interval $[t_1, t_L]$  in terms of $M$ is the same with the learning rate of the estimated interaction kernels (that are used to predict dynamics), agreeing with Theorem \ref{firstordersystem:Trajdiff}. 

\item  The coercivity condition holds on hypothesis spaces that consist of piecewise polynomials,  for different kernel functions, and various initial distributions including Gaussian and uniform distributions, and for different $L$. The learning rates of the kernels match closely the rate we derived in Theorem \ref{t:firstordersystem:thm_optRate}, which are optimal up to a logarithmic factor.

 \item Our estimators are robust to the observational noise up to a certain level and still produced rather accurate predictions of the true dynamics. The learning rate of interaction kernels tends to be flat as the noise level increases, showing our estimators do not overcome the problem exhibited by other estimators of ODEs and PDEs, and do not denoise and recover the true interaction kernel even asymptotically. When the noise is significant,  the accuracy of the estimators did not improve as the number of observed trajectories increased. 
\end{itemize}


\section{Discussions on the coercivity condition} 
\label{coercivity}

The coercivity condition on the hypothesis space  $\bhypspace$, quantized by the constant $c_{L,N,\bhypspace}$,   plays a vital role in establishing the optimal convergence rate, as is demonstrated in Theorem \ref{firstorder:maintheorem} and Theorem \ref{t:firstordersystem:thm_optRate}. 
Proposition \ref{firstordersystem:coercivityconstant} provides a way to compute the coercivity constant on a finite dimensional $\bhypspace$ numerically: it is the minimal eigenvalue of the matrix that yields the estimator by choosing an orthonormal basis of $\bL^2([0,R],\brhoL)$.  We have performed extensive numerical experiments to test the coercivity condition for different dynamical systems. Numerical results suggest that the coercivity condition holds true on rather general hypothesis space for a large class of kernel functions, and for various initial distributions including Gaussian and uniform distributions, and for different values of $L$ as long as $\brhoL$ is not degenerate.

 In the following, we prove the coercivity condition on  general compact sets of $\bL^2([0,R],\brhoL)$ under suitable hypotheses, and even independently of $N$. This implies that the finite sample bounds we achieved in Theorem \ref{t:firstordersystem:thm_optRate} can  be  dimension free,  suggesting that the coercivity condition may be a  fundamental property of the system, even in the mean field limit. 

\subsection{Homogeneous systems}

In a homogeneous system, it is natural to assume the distribution $\mu_0$ of initial conditions is exchangeable
(i.e., the distribution is invariant under permutation of components).  We prove the coercivity condition for exchangeable Gaussian distributions  in the case of $L=1$.   We show that $c_{L,N,\hypspace}$ can be bounded below by a positive constant that is independent of $N$ for any compact  set  $\hypspace$ in $L^2([0,R],\rhoL)$. A key element is connect the coercivity with the strict positiveness of an integral operator, which follows from a M\"untz-Sz\'asz-type Theorem. We refer to \cite{li2019identifiability} for a further study on the coercivity condition for stochastic homogeneous systems.

\begin{theorem}\label{firstordersingle:coercivity}
Consider the system at time $t_1=0$ with initial distribution ${\bm \mu}_0$ being exchangeable Gaussian with $\mathrm{cov}(\bx_i(t_1))-\mathrm{cov}(\bx_i(t_1),\bx_j(t_1))=\lambda I_d$\,  for a constant $\lambda>0$. Then the coercivity condition holds true on $\hypspace= L^2([0,R],\rho^1_T)$ with the coercivity constant $c_{L,N,\mathcal{H}}=\frac{N-1}{N^2}$. If the hypothesis space $\hypspace$ is a compact subset of $L^2([0,R],\rho^1_T)$,  then we have $c_{L,N,\mathcal{H}}=\frac{N-1}{N^2}+\frac{(N-1)(N-2)}{N^2}c_{\mathcal{H}}$ for its coercivity constant, where  $c_{\mathcal{H}}>0$ is independent of $N$.
 \end{theorem}

\begin{proof}
With $L=1$, the right hand side of the coercivity inequality  \eqref{firstorder:gencoer} is
\begin{align}\label{normexpansion}
\mathbb{E}_{\mu_0}[ \big\|  \rhsfo_{\intkernelvar}(\bX(t_1)) \big\|_{\mathcal{S}}^2]
&=\frac{1}{N^3}\sum_{i=1}^{N} \mathbb{E}_{\mu_0}\bigg[\left(\sum_{j=k=1}^N + \sum_{j\neq k=1}^N\right)  \intkernelvar(\|\bx_{ji}(t_1)\|)  \intkernelvar(\|\bx_{ki}(t_1)\|)\langle \bx_{ji}(t_1), \bx_{ki}(t_1)\big\rangle \bigg] \nonumber \\
&=\frac{N-1}{N^2} \|\intkernelvar(\cdot)\cdot\|_{L^2([0,R],\rho_T^1)}^2 + \mathcal{R},
\end{align}
where  $\mathcal{R} = \frac{1}{N^3} \sum_{i=1}^N\sum_{j\neq k, j\neq i, k\neq i} C_{ijk}$ with 
\begin{align*}
C_{ijk}:= \mathbb{E}\bigg[   \intkernelvar(\|\bx_{ji}(t_1)\|)  \intkernelvar(\|\bx_{ki}(t_1)\|)\langle \bx_{ji}(t_1), \bx_{ki}(t_1)\big\rangle \bigg] ,
\end{align*}
Because of exchangeability, we have
\begin{align*}
C_{ijk} = \mathbb{E}\bigg[ \intkernelvar(\|Y-X\|) \intkernelvar(\|Z-X\|) \left\langle Y-X,  Z-X \right\rangle\bigg]
\end{align*}
for all $(i,j,k)$, where $X, Y, Z$ are exchangeable Gaussian random variables with $\mathrm{cov}(X)-\mathrm{cov}(X,Y)=\lambda I_d$. 
By Lemma \ref{coerlemma} below, $C_{ijk}\geq c_{\mathcal{H}}\|\intkernelvar(\cdot)\cdot\|_{L^2([0,R],\rho_T^1)}^2$. Therefore, $\mathcal{R} \geq \frac{(N-1)(N-2)}{N^2}c_{\mathcal{H}}\|\intkernelvar(\cdot)\cdot\|_{L^2([0,R],\rho_T^1)}^2$, and 
$\dbinnerp{\intkernelvar, \intkernelvar} \geq (\frac{N-1}{N^2}+ \frac{(N-1)(N-2)}{N^2}c_{\mathcal{H}}) \|\intkernelvar(\cdot)\cdot\|^2_{L^2([0,R],\rho_T^1)}\,.$
\end{proof}

The following lemma is a key element in the above proof of the coercivity condition. 
 \begin{lemma}\label{coerlemma}
Let $X, Y, Z$ be exchangeable Gaussian random vectors in $\R^d$ with $\mathrm{cov}(X)-\mathrm{cov}(X,Y)=\lambda I_d$ for a constant $\lambda>0$. Let $\rho_T^1$ be a probability measure over $\R^+$ with density function $C_\lambda^{-1} r^{d-1}e^{-\frac{1}{4\lambda}r^2}$ where $C_\lambda =  \frac{1}{2} (4\lambda)^{\frac{d}{2}}\Gamma (\frac{d}{2})$.  Then, 
 \begin{equation} \label{ineq_gauss}
 \E\left[\intkernelvar(|X-Y|)\intkernelvar(|X-Z|)  \langle X-Y, X-Z \rangle \right] \geq c_\mathcal{\mathcal{X}}\|\intkernelvar(\cdot)\cdot\|^2_{L^2(\rho_T^1)}
  \end{equation}  
 for all $\intkernelvar(\cdot)\cdot \in \mathcal{X}$, with $c_{\mathcal{X}}>0$ if $\mathcal{X}$ is compact  and $c_{\mathcal{X}}=0$ if $\mathcal{X}=L^2([0,R],\rho_T^1)$. 
 \end{lemma}

\begin{proof}  
Let $(U,V)=(X-Y, X-Z)$. Note that the covariance matrix of $(U,V)$ is  $\lambda \begin{pmatrix}
 2 I_d      & I_d\\
 I_d         & 2I_d
 \end{pmatrix}$. 
 Then 
 \begin{align}\label{covarianceoperator}
\E\left[\intkernelvar(|X-Y|)\intkernelvar(|X-Z|)  \langle X-Y, X-Z \rangle \right] &=\ \E\left[\intkernelvar(|U|)\intkernelvar(|V|)  \langle U,V \rangle \right] \nonumber \\ 
=& \int_{0}^{\infty}\int_{0}^{\infty} \intkernelvar(r)r\intkernelvar(s)s G(r,s)d\rho_T^1(r)d\rho_T^1(s),
 \end{align}
 where the integral kernel $G:\R^+\times \R^+ \to \R$ is 
\begin{equation*}
 G(r,s)=e^{-\frac{1}{12\lambda}(r^2+s^2)} \begin{cases}  
 \frac{1}{2}C_d\left(e^{c_\lambda rs} -e^{- c_\lambda rs} \right) \,, &\text{ if } d=1;  \\
 C_d \int_{S_1}\int_{S_1} \langle \xi,\eta\rangle  e^{c_\lambda rs\langle \xi,\eta \rangle} \frac{d\xi d\eta}{|S_1|^2}\, , & \text{ if } d\geq 2
 \end{cases}
\end{equation*}
 with  $C_d=\left(\frac{\sqrt{3}}{2}\right)^{-d}$,  $c_\lambda=\frac{1}{3\lambda}$ and with $|S_1|=2\frac{\pi^{\frac{d}{2}}}{\Gamma(\frac{d}{2})}$ being the surface area of the unit sphere. Define 
\begin{equation}
  G_{R}(r,s)=\begin{cases}G(r,s), & 0\leq r,s \leq R; \\ 0, &\text{otherwise}.  \end{cases}
  \label{e:G_R}
\end{equation}
 Note that $G_{R}(r,s) \in L^2([0,R]\times [0,R], \rho_T^1\times \rho_T^1)$ is real and symmetric, so its associated integral operator $Q_R: L^2([0,R],\rho_T^1)\rightarrow L^2([0,R],\rho_T^1)$
\begin{align}\label{integraloperator}
Q_Rg(r)=\int_{0}^{R}G_R(r,s)g(s) d\rho_T^1(s),\quad  r \in [0,R]
\end{align}
is  symmetric and compact.  Following from \eqref{covarianceoperator}, for any $\intkernelvar(\cdot)\cdot \in L^2([0,R],\rho_T^1)$
\begin{align}
& \ \E\left[\intkernelvar(|X-Y|)\intkernelvar(|X-Z|)  \langle X-Y, X-Z \rangle \right] =\langle Q_R\intkernelvar(\cdot)\cdot,\intkernelvar(\cdot)\cdot\rangle_{L^2([0,R],\rho_T^1)}.
 \end{align}

To show the existence of $c_\mathcal{X}\geq 0$ in \eqref{ineq_gauss}, we show that $Q_R$ is  strictly positive.  We first show  that $\langle Q_Rg,g\rangle_{L^2([0,R],\rhoT^1)}\geq 0$ for any $g \in L^2([0,R],\rhoT^1)$. When $d=1$, we have from Taylor expansion that 
\begin{align}\label{positivekernel}
 G(r,s)=e^{-\frac{1}{12\lambda}(r^2+s^2)}\sum_{k=0}^\infty \frac{1}{k!} a_k (rs)^{k} 
  \end{align}
  with $a_k= \frac{1}{2}C_dc_\lambda^k \left(1- (-1)^k \right)$. When $d\geq 2$, using the fact that 
\[
 \langle \xi,\eta\rangle  e^{c_\lambda rs \langle \xi,\eta\rangle} = \frac{1}{rs}\frac{\partial}{\partial {c_\lambda}}e^{c rs \langle \xi,\eta\rangle}
=\sum_{k=1}^\infty \frac{1}{(k-1)!} c_\lambda^{k-1} (rs)^{k-1}   \langle \xi,\eta\rangle^k, 
\]
and the fact that  
\[
b_{k+1}=\int_{S_1}\int_{S_1}  \langle \xi,\eta\rangle^k \frac{d\xi d\eta}{|S_1|^2} \left\{
                \begin{array}{ll}
                  = 0, & \textrm{ for even } k; \\
                  \in (0,1),   & \textrm{ for odd } k,
                \end{array}
              \right. 
 \]
we obtain again \eqref{positivekernel} with $a_k= C_d c_\lambda^k b_{k+1}\geq 0$. Note that for either $d=1$ or $d\geq 2$, we have $a_k>0$ when $k$ is odd and $a_k =0$ when $k$ is even. 
Therefore, for any $g \in L^2([0,R],\rho_T^1)$ we have 
\begin{equation*}
 \langle Q_Rg,g\rangle_{L^2([0,R],\rho_T^1)} = \sum_{k=1, k \text{ odd}}^\infty \frac{a_k}{k!} \left(\int_0^R g(r) r^{k}e^{-\frac{1}{12\lambda}r^2} d\rho_T^1 \right)^2\geq 0. 
  \end{equation*}
Next we show $\langle Q_Rg,g\rangle_{L^2([0,R],\rhoT^1)}=0$ implies $g =0$. Suppose $ \langle Q_Rg,g\rangle_{L^2([0,R],\rhoT^1)}=0$: this means that 
 $$g(r)r e^{-\frac{1}{12\lambda}r^2} \perp \mathrm{span} \{1,r^2, r^4, r^6,\cdots\} \subset L^2([0,R],\rho_T^1). $$
 But $\mathrm{span} \{1,r^2, r^4, r^6,\cdots\}$ is a dense set in $L^2([0,R],\rho_T^1)$ by M\"untz-Sz\'asz Theorem \cite[Theorem 6.5]{borwein1997generalizations}, therefore $g(r)re^{-\frac{1}{12\lambda}r^2}=0$ and hence $g=0$.  This proves that $Q_R$ is strictly positive, which implies that $Q_R$ only has positive eigenvalues with 0 as 
 an accumulation point of eigenvalues.  Therefore, for $\intkernelvar(\cdot)\cdot$ in the compact set $\mathcal{X}$ of $L^2([0,R],\rho_T^1)$, we have 
   \begin{equation*}
 \langle Q_R\intkernelvar(\cdot)\cdot, \intkernelvar(\cdot)\cdot \rangle_{L^2([0,R],\rho_T^1)} \geq c_{\mathcal{X}} \|\intkernelvar(\cdot)\cdot\|^2_{L^2([0,R],\rho_T^1)},
 \end{equation*} 
 where $c_{\mathcal{X}}>0$ only depends on $\mathcal{X}$. 
 \end{proof}

 The following lemma shows that for any $R>0$, the norm of the operator $Q_R$ is strictly less than $1$. 
 \begin{lemma}\label{lemmaSui}
 The compact operator $Q_R: L^2([0,R],\rho_T^1) \rightarrow L^2([0,R],\rho_T^1)$ defined in \eqref{integraloperator} satisfies $\|Q_R\|<1$. 
 \end{lemma}
 \begin{proof}
Note that $\|Q_R\|\leq 1$ follows directly from the Cauchy-Schwarz inequality: 
\begin{align*}
\langle Q_R(\intkernelvar),\intkernelvar \rangle_{L^2([0,R],\rho_T^1)}&=\E \left[\intkernelvar(|U|)\intkernelvar(|V|)  \langle \frac{U}{|U|},\frac{V}{|V|} \rangle \right] \leq \mathbb{E}\left[\intkernelvar(|U|)\intkernelvar(|V|)\right]   \nonumber 
\\ 
&\leq  \sqrt{\mathbb{E}\left[\intkernelvar^2(|V|)\right]}\sqrt{\mathbb{E}\left[\intkernelvar^2(|U|)\right]}
=\|\intkernelvar\|_{L^2([0,R], \rho_T^1)}^2,  
 \end{align*}
 for any $\intkernelvar \in L^2([0,R], \rho_T^1)$. 
Suppose $\|Q_R\|=1$. Since $Q_R$ is compact, then there exists an eigenfunction: 
$$Q_Rg(r)=g(r), \, r\in [0,R]. $$
Now define $\hat g(r)=\begin{cases} g(r), & 0\leq r\leq R; \\ 0,&R< r \leq 2R\end{cases}$. Note that, from its definition in \eqref{e:G_R}, $G_{2R}(r,s) = G_{R}(r,s)$ for all $(r,s)\in [0,R]\times [0,R]$; therefore for $r\in [0,R]$, 
\[
Q_{2R}\hat g (r) = \int_0^{2R} \hat g (s) G_{2R}(r,s)d\rho_T^1(s) =  \int_0^{R} g (s) G_{2R}(r,s)d\rho_T^1(s)  = Q_Rg(r).
\]
Therefore, using the fact that $\|Q_{2R}\|\leq 1$, we have
\begin{align*}
\langle Q_{2R}\hat g, Q_{2R}\hat g\rangle_{L^2([0,2R],\rho_T^1)}&= \langle Q_Rg, Q_Rg\rangle_{L^2([0,R],\rho_T^1)}+ \langle Q_{2R}\hat g, Q_{2R} \hat g\rangle_{L^2([R,2R],\rho_T^1)}\\& \leq \|\hat g\|^2_{L^2([0,2R],\rho_T^1)} =\|g\|^2_{L^2([0,R],\rho_T^1)}. 
\end{align*} 
This means that  $Q_{2R}\hat g=0$ a.e. on $[R, 2R]$. However, we now show that $Q_{2R}\hat g(r)$ is real analytic over $(0,2R)$, which leads to a contradiction. 
To see the analyticity of $Q_{2R}\hat g(r)$, we use the power series representation of the kernel $G_{2R}(r,s)$ defined in \eqref{positivekernel}, to see that for $s \in [0, 2R]$
\begin{align*}
 (Q_{2R}\hat g)(s)&=\int_{0}^{R}G_{2R}(r,s)g(r) d\rho_T^1(r)
=e^{-\frac{1}{12\lambda}s^2}\int_{0}^{R}\sum_{k=1,\text{odd}}^\infty \frac{1}{k!} a_k r^{k}s^k g(r) e^{-\frac{1}{12\lambda}r^2}d\rho_T^1(r)\\&=C_d e^{-\frac{1}{12\lambda}s^2}\sum_{k=1,\text{odd}}^\infty \frac{1}{k!} \frac{b_{k+1}}{3^k} c_k s^k, \end{align*}
where $c_k=\int_{0}^{\infty} r^{k}e^{-\frac{1}{12\lambda}r^2}g(r) d\rho_T^1(r)$ and $b_{k+1} \in (0,1)$. Then 
$|c_k| < \|\intkernelvar\|_{L^2(\rho_T^1)}\|r^k\|_{L^2(\rho_T^1)}$. By computation, $\|r^k\|_{L^2(\rho_T^1)}=\sqrt{3^k{\Gamma(k+\frac{d-1}{2})}/{\Gamma(\frac{d}{2})}}$. According to Stirling's formula, $\Gamma(z+1)\sim \sqrt{2\pi z} ({z/e})^z$  for positive $z$ and $k! \sim {(k/e)^k}\sqrt{2\pi k}$; applying the root test, we conclude that the convergence radius for the above series on the right is infinity. Therefore, it is a real analytic function over $(0, 2R)$.
 \end{proof}

Theorem \ref{firstordersingle:coercivity} shows a particular case in which the coercivity constant $c_{L,N,\mathcal{H}}$ is positive uniformly in $N$, and therefore coercivity is a property also of the system in the limit as $N\rightarrow\infty$, satisfying the mean-field equations. {The coercivity condition has been further discussed in \cite{li2019identifiability}, where it is proved for a class of linear and nonlinear stochastic systems of interacting agents. }

 \subsection{Heterogeneous systems}
 Intuitively, learning interaction kernels in heterogeneous systems seems more difficult than in homogeneous systems, as the observed velocities are the superposition of multiple interaction kernels evaluated at different locations.   However, the numerical experiments in subsection \ref{numericalexamples} demonstrate the efficiency of our algorithm in   learning multiple interaction kernels simultaneously from their superpositions. In this section, we generalize the arguments of Theorem \ref{firstordersingle:coercivity} to  heterogeneous systems. In particular, this requires considering the coercivity condition on the function space of multiple interaction kernels. For simplicity of notation, we consider a system with $K=2$ types of agents, with $C_1=\{1,\dots,N\}$ and $C_2=\{N+1,\dots,2N\}$.

 \begin{theorem}\label{firstorder:coercivity}
Consider the system at time $t_1=0$ with initial distribution ${\bm \mu}_0$ being exchangeable Gaussian with $\mathrm{cov}(\bx_i(t_1))-\mathrm{cov}(\bx_i(t_1),\bx_j(t_1))=\lambda I_d$\,  for a constant $\lambda>0$. Then the coercivity condition holds true on  a hypothesis space $\bhypspace$ that is compact in $\bL^2([0,R], \brho_{T}^{1})$,   with constant $c_{L,N,\bhypspace}> \frac{(1-c'_{\bhypspace})(N-1)}{N^2}$, where $c'_{\bhypspace}<1$. 
\end{theorem}

\begin{proof} Since the initial distribution is exchangeable Gaussian and $L=1$, we have
$\rho_T^{1,kk'} \equiv\rho_T^1$,  a probability measure over $\R^+$ with density function $C_\lambda r^{d-1}e^{-\frac{1}{3\lambda}r^2}$ where $C_\lambda =  \frac{1}{2} (3\lambda)^{\frac{d}{2}}\Gamma (\frac{d}{2})$. For $\bintkernelvar=(\intkernelvar_{kk'})_{k=1,k'=1}^{2,2} \in \bhypspace$,
\begin{equation}
\label{coercivityexpansion}
\begin{aligned}
 \mathbb{E}_{\mu_0}\big\| \rhsfo_{\bintkernelvar}(\bX(0)) \big\|_{\mathcal{S}}^2&=\frac{N-1}{N^2} \|\intkernelvar_{11}(\cdot)\cdot\|_{L^2([0,R],\rho_{T}^{1})}^2 +  \frac{1}{N}\|\intkernelvar_{12}(\cdot)\cdot\|_{L^2([0,R],\rho_{T}^{1})}^2+\frac{N-1}{N}\mathcal{R}_1\\&+\frac{1}{N}\|\intkernelvar_{21}(\cdot)\cdot\|_{L^2([0,R],\rho_{T}^{1})}^2+\frac{N-1}{N^2} \|\intkernelvar_{22}(\cdot)\cdot\|_{L^2([0,R],\rho_{T}^{1})}^2+\frac{N-1}{N}\mathcal{R}_2,
\end{aligned} 
\end{equation}
where
\begin{align*}
\mathcal{R}_{1}&=\frac{N-2}{N}\langle Q\intkernelvar_{11}(\cdot)\cdot,\intkernelvar_{11}(\cdot)\cdot \rangle_{L^2([0,R],\rho_T^1)}+2\langle Q\intkernelvar_{11}(\cdot)\cdot,\intkernelvar_{12}(\cdot)\cdot \rangle_{L^2([0,R],\rho_T^1)}+\langle Q\intkernelvar_{12}(\cdot)\cdot,\intkernelvar_{12}(\cdot)\cdot \rangle_{L^2([0,R],\rho_T^1)}, \\ \mathcal{R}_{2}&=\frac{N-2}{N}\langle Q\intkernelvar_{22}(\cdot)\cdot,\intkernelvar_{22}(\cdot)\cdot \rangle_{L^2([0,R],\rho_T^1)}+2\langle Q\intkernelvar_{22}(\cdot)\cdot,\intkernelvar_{21}(\cdot)\cdot \rangle_{L^2([0,R],\rho_T^1)} +\langle Q\intkernelvar_{21}(\cdot)\cdot,\intkernelvar_{21}(\cdot)\cdot \rangle_{L^2([0,R],\rho_T^1)},
\end{align*} 
and $Q$ is the integral operator defined in \eqref{integraloperator}. Suppose $\bhypspace=\bigoplus_{k,k'=1,1}^{2,2}\mathcal{H}_{k,k'}$ with $\mathcal{H}_{k,k'} \subset L^2([0,R],\rho_T^1)$. Let
$$c'_{\bhypspace}=\max_{k,k',f\in \mathcal{H}_{k,k'} }\frac{\langle Qf(\cdot)\cdot,f(\cdot)\cdot\rangle_{L^2([0,R],\rho_T^1)}}{\langle f(\cdot)\cdot,f(\cdot)\cdot\rangle_{L^2([0,R]\rho_T^1)}}\,;
$$
by Lemma \ref{lemmaSui}, 
we know $c'_{\bhypspace}=\max_{kk'}c'_{\hypspace_{kk'}}<1$,   then we have
\begin{align}\label{coercivitykey}
\|\intkernelvar_{kk'}(\cdot)\cdot\|_{L^2([0,R],\rho_T^1)}^2 \geq \langle Q\intkernelvar_{kk'}(\cdot)\cdot, \intkernelvar_{kk'}(\cdot)\cdot\rangle_{L^2([0,R],\rho_T^1)}+(1-c'_{\bhypspace})\|\intkernelvar_{kk'}(\cdot)\cdot\|_{L^2([0,R],\rho_T^1)}^2
\end{align}
Combining \eqref{coercivitykey} with \eqref{coercivityexpansion} yields 
\begin{align}\label{coercivityexpansion1}
\mathbb{E}_{\mu_0}\big\| \rhsfo_{\bintkernelvar}(\bX(0)) \big\|_{\mathcal{S}}^2 &\geq \frac{(1-c'_{\bhypspace})(N-1)}{N^2} \left(\|\intkernelvar_{11}(\cdot)\cdot\|_{L^2([0,R],\rho_T^1)}^2+\|\intkernelvar_{22}(\cdot)\cdot\|_{L^2([0,R],\rho_T^1)}^2\right)\nonumber \\&+ \frac{(1-c'_{\bhypspace})}{N} \left(\|\intkernelvar_{12}(\cdot)\cdot\|_{L^2([0,R], \rho_T^1)}^2+\|\intkernelvar_{21}(\cdot)\cdot\|_{L^2([0,R],\rho_T^1)}^2)\right)\nonumber\\&+\frac{N-1}{N}\left(\langle Q \psi_1, \psi_1\rangle_{L^2([0,R],\rho_T^1)}+\langle Q \psi_2, \psi_2\rangle_{L^2([0,R],\rho_T^1)}\right)
\end{align} 
where
\begin{align}
\psi_1=\sqrt{\frac{N-1}{N}}\intkernelvar_{11}(\cdot)\cdot+\sqrt{\frac{N}{N-1}}\intkernelvar_{12}(\cdot)\cdot \label{psi12}\quad\text{and}\quad
\psi_2=\sqrt{\frac{N-1}{N}}\intkernelvar_{22}(\cdot)\cdot+\sqrt{\frac{N}{N-1}}\intkernelvar_{21}(\cdot)\cdot \,.
\end{align}
Note that  $Q: L^2([0,R],\rho_T^1)\rightarrow L^2([0,R],\rho_T^1)$ defined on \eqref{integraloperator} is a compact strictly positive operator. Therefore \eqref{coercivityexpansion1} yields that 
$$
\mathbb{E}_{\mu_0}\big\| \rhsfo_{\bintkernelvar}(\bX(0)) \big\|_{\mathcal{S}}^2 >\frac{(1-c'_{\bhypspace})(N-1)}{N^2} \|\bintkernelvar(\cdot)\cdot\|^2_{\bL^2([0,R],\brho_T^1)}.
$$

\end{proof}

We remark that the inequality \eqref{coercivityexpansion1} indicates that $c_{L,N,\bhypspace}$ could be independent of $N$ if 
\begin{align}
\langle Q \psi_1, \psi_1\rangle_{L^2([0,R], \rho_T^1)} +\langle Q \psi_2, \psi_2\rangle_{L^2([0,R],\rho_T^1)} \approx \|\bintkernelvar(\cdot)\cdot\|_{\bL^2([0,R],\rho_T^1)}^2
\end{align} where the functions $\psi_1$ and $\psi_2$ are defined in \eqref{psi12}.  This would be implied by: for $k=1,2$,  the hypothesis spaces $\mathcal{H}_{k,1}$ and $\mathcal{H}_{k,2}$ have a positive angle as  subsets of $L^2([0,R],\rho_T^1)$, so that  for $f \in \mathcal{H}_{k,1}$ and $g\in \mathcal{H}_{k,2}$
\begin{align}\label{idealinequality}
\|f(\cdot)\cdot+g(\cdot)\cdot\|_{L^2([0,R],\rho_T^{1})} ^2 \geq c'(\|f(\cdot)\cdot\|_{L^2([0,R],\rho_T^{1})}^2+\|g(\cdot)\cdot\|_{L^2([0,R],\rho_T^1)}^2).
\end{align} For example, if  the supports of true interaction kernels $\intkernel_{k,1}$ and $\intkernel_{k,2}$ are disjoint for $k=1,2$ and this information is available  a priori, then we could choose $\mathcal{H}_{k,1}$ and $\mathcal{H}_{k,2}$ consisting of appropriate functions with disjoint supports. In this case,  \eqref{idealinequality} is true with $c'=1$. Using  arguments as in the case of  homogeneous systems, one can then show the coercivity constant is positive and independent of $N$.

 \section{Appendix: proofs}
In this section, we provide technical details of  our main  results. For reader's convenience and the sake of a self-contained presentation, we first show that the first order heterogeneous systems \eqref{e:firstordersystem} are well-posed provided the interaction kernels are in the admissible space.

\subsection{Well posedness of first order heterogeneous systems}
\begin{proposition} \label{firstordersystem:wellposedness}Suppose the kernels $\bintkernel=(\intkernel_{kk'})_{k,k'=1}^{K}$ lie in the admissible set $\mbf{\mathcal{K}}_{R,S}$, i.e., $\intkernel_{kk'} \in \mathcal{K}_{R,S}.$ Then the  first order heterogenous  system \eqref{e:firstordersystem} admits a unique global solution in $[0, T]$ for every initial datum $\bX(0) \in \mathbb{R}^{ dN}$ and the solution depends continuously on the initial condition.
\end{proposition}

The proof of the Proposition \ref{firstordersystem:wellposedness} uses  Lemma \ref{lipchitz} and the same techniques for proving the well-posedness of the homogeneous system (see Section 6 in \cite{BFHM17}). 

\begin{lemma}\label{lipchitz} For any  $\intkernelvar \in \mK_{R,\spaceM}$, the  function
$$
F_{[\intkernelvar]}(\bx)= \intkernelvar(\|\bx\|)\bx\quad,\, \bx \in \real^d 
$$
is Lipschitz continuous on $\mathbb{R}^d$.
\end{lemma}

\subsection{Proofs of properties of measures}
\label{s:proofsPropertiesOfMeasures}
\begin{proof} [Proof of Lemma \ref{averagemeasure}]For each $(\idxcl,\idxcl')$ with $1\leq \idxcl,\idxcl' \leq K$, $t$ and  each Borel set $A \subset \mathbb{R}_{+}$, define  
\[ \varrho^{\idxcl\idxcl'}(t,A)= \frac{1}{LN_{\idxcl\idxcl'}} \mathbb{E}_{\mu_0} \!\sum_{\substack{i \in C_{\idxcl}, i' \in C_{\idxcl'} \\ i\neq i'}}  \mathbbm1_{A}(r_{ii'}(t)),
\] 
where $\mathbbm1_{A}$ is the indicator function of the set $A$. Clearly, the measure $\varrho^{\idxcl\idxcl'}(t,\cdot)$ is the average of the probability distributions of the pairwise distances between type $\idxcl$ agents and type $\idxcl'$ agents, and therefore it is a probability distribution, and so is  $\rho_T^{L, \idxcl\idxcl'}=\frac{1}{L}\sum_{l=1}^L \varrho^{\idxcl\idxcl'}(t_l,\cdot)$. 

To show that $\rho_T^{\idxcl\idxcl'}(\cdot)= \frac{1}{T}\int_0^T \varrho^{\idxcl\idxcl'}(t,\cdot)dt$ is well-defined and is a probability measure, it suffices to show that the mapping $t \in [0,T] \rightarrow \varrho^{\idxcl\idxcl'}(t, A)$ is lower semi-continuous for every open set $A \subset \mathbb{R}_{+}$, and is  upper semi-continuous for any compact set $A$. Fix $t \in [0,T]$. Due to the continuity of the solution to ODE system (see Proposition \ref{firstordersystem:wellposedness}), we have $ \big \|\bX(t_n)-\bX(t)\big\| \rightarrow 0$   if $t_n \rightarrow t$,  therefore $ r_{ii'}(t_n)$ converges to $r_{ii'}(t)$   for each pair $(i,i')$ with $1\leq i, i'\leq N$.  Since the indicator function of an open set is lower semi-continuous, whereas the indicator function of a closed set is upper semi-continuous, the conclusion follows from the Portmanteau Lemma.

To prove the absolute continuity of $\rho_T^{L, \idxcl\idxcl'}$ and $\rho_T^{\idxcl\idxcl'}$ with respect to the Lebesgue measure, let $A \subset \mathbb{R}_{+}$ with Lebesgue measure zero. Let $P_{ii'}(\bX)= r_{ii'}: \R^{dN}\to\R_+$ for $i,i'=1,\dots,N$, and denote $F_t: \R^{dN}\to \R^{dN}$ the forward map such that $\bX_t=F_t(\bX_0)$. Observe that $B_{ii'} = P_{ii'}^{-1}(A)$  in a Lebesgue null set in $\R^{dN}$ for each $(i,i')$, and that the forward map $F_t$ of the dynamical system is continuous, we have  
\[
\mathbb{E}_{\mu_0}[ \mathbbm1_{A}(r_{ii'}(t))] = \mu_0\left(F_t^{-1}(P_{ij}^{-1}(A)) \right) =0 
\]
for each $t$ and each pair $(i,j)$. As a consequence, $$\rho_T^{L, \idxcl\idxcl'}(A)= 
\frac{1}{LN_{\idxcl\idxcl'}}\sum_{l=1}^{L}\!\sum_{\substack{i \in C_{\idxcl}, i' \in C_{\idxcl'} \\ i\neq i'}} \mathbb{E}_{\mu_0}[ \mathbbm1_{A}(r_{ii'}(t_l))]=0$$ and similarly $\rho_T^{\idxcl\idxcl'}(A)=0$ by Fubini's Theorem. 

The previous analysis can also imply that, or any Borel set $A$,
$$\rho_T^{L, \idxcl\idxcl'}(A)=\sup \{\rho_T^{L, \idxcl\idxcl'}(K), K \subset A, K \text{ compact}\}=\inf \{\rho_T^{L, \idxcl\idxcl'}(O), A \subset O, O \text{ open}\}.$$ 

Therefore, the measure $\rho_T^{L, \idxcl\idxcl'}$ is a regular measure on $\mathbb{R}_+$. 
 \end{proof}

 \begin{proof}[Proof of Proposition \ref{compactmeasure}]
 By integration of \eqref{e:firstordersystem} we obtain
\begin{align*}
\|\bx_i(t)\| &\leq \|\bx_i(0)\|+\int_{0}^{t} \sum_{i' = 1}^{N}\frac{1}{N_{\clof_{i'}}}|\intkernel_{\clof_i\clof_{i'}}(r_{ii'}(s))|{r}_{ii'}(s)ds
\leq \|\bx_i(0)\|+K\|\bintkernel\|_{\infty}Rt.
\end{align*}
Using the fact that $\mu_0$ is compactly supported, we obtain 
\begin{align*}
\max_{i} \|\bx_i(t)\| \leq  C_0+K\|\bintkernel\|_{\infty}Rt.
\end{align*}
 for some constant $C_0$ depending only on the size of  the support of $\mu_0$. Therefore, 
$$\max_{ii'} r_{ii'}(t)\leq 2\max_{i} \|\bx_i(t)\|  \leq R_0\,,\quad 0\leq t \leq T$$ where
$R_0=2C_0+2K\|\bintkernel\|_{\infty}RT$. The conclusion follows. 

\end{proof}

\subsection{Proofs of  Convergence of Estimators}
 Throughout this section, we assume that 
\begin{assumption}\label{compactnesscondition}
 $\bhypspace$ is a compact convex subset of $\mbf{L^\infty}([0,R])$ and is bounded above by $S_0\geq S$.
 \end{assumption}
 
It is easy to see that $\bhypspace$ can be naturally embedded as a compact set of $\bL^2(\brhoL)$. Assumption \ref{compactnesscondition} ensure that the existence of minimizers to the error functionals $\mbf{\mE}_{M}$ defined in \eqref{firstorder:empircalerrorfunctional}. We shall first estimate discrepancy of these functionals, prove the uniqueness of their minimizers, and then establish uniform estimates on the defect between $\mbf{\mE}_{M}$ and $\mbf{\mE}_{\infty}$. 
 
 For $t\in[0,T]$ and $\bintkernelvar \in \bhypspace$, we introduce  the random variable 
 
 \begin{align}\label{firstorder:randomvariable}
 \mbf{\mE}_{\bX(t)}(\bintkernelvar):=\norm{\dot{\bX}(t) - \rhsfo_{\bintkernelvar}(\bX(t))}^2_{\mathcal{S}},\end{align}where  $\|\cdot\|_{\mathcal{S}}$ is defined in $\eqref{firstorder: classnorm}$. Then we know that $\mbf{\mathcal{E}}_{\infty}(\bintkernelvar)=\frac{1}{L}\sum_{l=1}^{L}\mathbb{E}_{\mu_0}[\mbf{\mE}_{\bX(t_l)}(\bintkernelvar)]$. 

\subsection*{Continuity of the error functionals over $\bhypspace$}
\begin{proposition} 
For  $ \bintkernelvar_1, \bintkernelvar_2  \in \bhypspace$, we have
\begin{align}
|\mbf{\mE}_{\infty}(\bintkernelvar_1)-\mbf{\mE}_{\infty}( \bintkernelvar_2)| &\leq  K^2 \| \bintkernelvar_1(\cdot)\cdot- \bintkernelvar_2(\cdot)\cdot \|_{\bL^2(\brhoL)} \|2\bintkernel(\cdot)\cdot- \bintkernelvar_1(\cdot)\cdot- \bintkernelvar_2(\cdot)\cdot \|_{\bL^2(\brhoL)}\label{firstordersystem:expectationerrorfunctional}\\
 |\mbf{\mathcal{E}}_{M}( \bintkernelvar_1)-\mbf{\mathcal{E}}_{M}( \bintkernelvar_2)|&\leq K^4  \|  \bintkernelvar_1(\cdot)\cdot- \bintkernelvar_2(\cdot)\cdot \|_{\infty}\| 2\bintkernel(\cdot)\cdot- \bintkernelvar_1(\cdot)\cdot- \bintkernelvar_2(\cdot)\cdot\|_{\infty}\label{firstordersystem:empiricalerrorfunctional}
\end{align}
\end{proposition}
\begin{proof} {Let $\bintkernelvar_1=(\intkernelvar_{kk',1})$ and $\bintkernelvar_2=(\intkernelvar_{kk',2})$. In the following, we use $k, k'$ and $k''$ to index as agent types and $C_{k}$ as the indices of type $k$ agents.} Using Jensen's inequality, then
\begin{align}
&|\mE_{\bX(t)}( \bintkernelvar_1)-\mE_{\bX(t)}( \bintkernelvar_2)| \nonumber \\ &= \bigg| \sum_{k=1}^{K}\frac{1}{N_j} \sum_{i \in C_k} \big\langle \sum_{k'=1}^{K}\frac{1}{N_{j'}}\sum_{i'\in C_{k'}}(  \intkernelvar_{kk',1}- \intkernelvar_{kk',2})(r_{ii'}){\br}_{ii'}, \sum_{k''=1}^{K}\frac{1}{N_{k''}} \sum_{i'\in C_{k''}}(2\intkernel_{kk'}- \intkernelvar_{kk',1}- \intkernelvar_{kk',2})(r_{ii'}){\br}_{ii'}\big\rangle \bigg| \nonumber \\&\leq \sum_{k=1}^{K}
\sum_{k'=1}^{K} \sum_{k''=1}^{K}\frac{1}{N_k}\sum_{i \in C_k} \| \frac{1}{N_{k'}}\sum_{i' \in C_{k'}}( \intkernelvar_{kk',1}- \intkernelvar_{kk',2})(r_{ii'}){{\br}_{ii'}} \| \| \frac{1}{N_{k''}}\sum_{i' \in C_{k''}}(2\intkernel_{kk''}- \intkernelvar_{kk'',1}-\widehat \intkernel_{kk'',2})(r_{ii'}){{\br}_{ii'}} \| \nonumber \\
&< \sum_{j=1}^{K}\sum_{k'=1}^{K} \sum_{k''=1}^{K}\sqrt{\frac{1}{N_kN_{k'}} \sum_{i \in C_k, i'\in C_{k'}} ( \intkernelvar_{kk',2}- \intkernelvar_{kk',1})^2(r_{ii'})r^2_{ii'}} \times \nonumber\\ &\qquad\qquad\qquad\qquad\sqrt{\frac{1}{N_kN_{k''}} \sum_{i \in C_k, i'\in C_{k''}} (2\intkernel_{kk''}- \intkernelvar_{kk'',1}- \intkernel_{kk'',2})^2(r_{ii'})r^2_{ii'}}\nonumber\\ 
&< \sum_{k=1}^{K}\sum_{k'=1}^{K} \sum_{k''=1}^{K}  \| \intkernelvar_{kk',2}(\cdot)\cdot- \intkernelvar_{kk',1}(\cdot)\cdot\|_{L^2(\hat\rho_{T}^{t, kk'})}  \|2\intkernel_{kk''}(\cdot)\cdot- \intkernelvar_{kk'',1}(\cdot)\cdot- \intkernel_{kk'',2}(\cdot)\cdot\|_{L^2(\hat\rho_{T}^{t, kk''})}\nonumber \\&\leq K^2\| \bintkernelvar_1(\cdot)\cdot- \bintkernelvar_2(\cdot)\cdot \|_{\bL^2(\hat\brho_T^t)} \|2\bintkernel(\cdot)\cdot- \bintkernelvar_1(\cdot)\cdot- \bintkernelvar_2(\cdot)\cdot\|_{\bL^2(\hat\brho_T^t)},
\end{align}where 
\begin{align*}
\smash{\widehat\rho}_{T}^{ t, kk'}=\displaystyle \frac{1}{LN_{kk'}}\sum_{l= 1}^L\! \sum_{\substack{i \in C_{k}, i' \in C_{k'} \\ i\neq i'}} \delta_{r_{ii'}(t)}(r) dt, \text{ and } \widehat\brho_{T}^{t}=\bigoplus_{k,k'=1,1}^{K,K}\widehat\rho_{T}^{t, kk'}.
\end{align*}
Therefore
\begin{align}
&\big |\frac{1}{L}\sum_{l=1}^{L}\mathcal{E}_{\bX(t_l)}( \bintkernelvar_1)-\frac{1}{L}\sum_{l=1}^{L}\mathcal{E}_{\bX(t_l)}( \bintkernelvar_2) \big| \leq \frac{1}{L}\sum_{l=1}^{L} \big |\mathcal{E}_{\bX(t_l)}( \bintkernelvar_2)- \mathcal{E}_{\bX(t_l)}( \bintkernelvar_1 ) \big|   \nonumber \\&<\frac{K^2}{L}\sum_{l=1}^{L}\| \bintkernelvar_1(\cdot)\cdot- \bintkernelvar_2(\cdot)\cdot \|_{\bL^2(\widehat\brho_T^{t_l})} \|2\bintkernel(\cdot)\cdot- \bintkernelvar_1(\cdot)\cdot- \bintkernelvar_2(\cdot)\cdot\|_{\bL^2(\widehat\brho_T^{t_l})}\nonumber \\&\leq K^2\sqrt{\frac{1}{L}\sum_{l=1}^{L}\| \bintkernelvar_1(\cdot)\cdot- \bintkernelvar_2(\cdot)\cdot \|^2_{\bL^2(\widehat\brho_T^{t_l})}}\sqrt{\frac{1}{L}\sum_{l=1}^{L}\|2\bintkernel(\cdot)\cdot- \bintkernelvar_1(\cdot)\cdot- \bintkernelvar_2(\cdot)\cdot\|^2_{\bL^2(\widehat\brho_T^{t_l})}} \nonumber \\&=K^2 \| \bintkernelvar_1(\cdot)\cdot- \bintkernelvar_2(\cdot)\cdot \|_{\bL^2(\widehat{\mbf{\rho}}_T^L)} \|2\bintkernel(\cdot)\cdot- \bintkernelvar_1(\cdot)\cdot- \bintkernelvar_2(\cdot)\cdot \|_{\bL^2(\widehat{\mbf{\rho}}_T^L)} \label{usefuleq1}\\&\leq K^4 \| \bintkernelvar_1(\cdot)\cdot- \bintkernelvar_2(\cdot)\cdot \|_{\infty} \|2\bintkernel(\cdot)\cdot- \bintkernelvar_1(\cdot)- \bintkernelvar_2(\cdot) \cdot\|_{\infty}\label{usefuleq2}
\\&\leq K^4 R^2\| \bintkernelvar_1- \bintkernelvar_2 \|_{\infty} \|2\bintkernel- \bintkernelvar_1- \bintkernelvar_2 \|_{\infty}\label{usefuleq3}
\end{align}
Taking expectation with respect to $\mu_0$ on both sides and using the \eqref{usefuleq1} yields the first inequality. Notice that 
$$ |\mbf{\mathcal{E}}_{M}( \bintkernelvar_1)-\mbf{\mathcal{E}}_{M}( \bintkernelvar_2)|\leq \frac{1}{M}\sum_{m=1}^{M} \big|\frac{1}{L}\sum_{l=1}^{L}\mathcal{E}_{\bX^{(m)}(t_l)}( \bintkernelvar_1)-\frac{1}{L}\sum_{l=1}^{L}\mathcal{E}_{\bX^{(m)}(t_l)}( \bintkernelvar_2)\big|
,$$
Then the second inequality in proposition follows by applying \eqref{usefuleq2}.

\end{proof}

The following lemma can be immediately deduced using \eqref{firstordersystem:expectationerrorfunctional}, \eqref{firstordersystem:empiricalerrorfunctional}, and \eqref{usefuleq3}. 

\begin {lemma} \label{berstein} For all $\bintkernelvar \in \bhypspace$, we define the defect function 
\begin{equation}
L_{M}(\bintkernelvar)=\mbf{\mathcal{E}}_{\infty}(\bintkernelvar)-\mbf{\mathcal{E}}_{M}(\bintkernelvar)\,.
\label{e:defectFunction}
\end{equation}
Then for $\bintkernelvar_1,  \bintkernelvar_2 \in \bhypspace$, the estimate  
 $$|L_{M}(\bintkernelvar_1)-L_{M}(\bintkernelvar_2)| \leq  2K^4R^2\|\bintkernelvar_1-\bintkernelvar_2 \|_{\infty} \| \bintkernelvar_1+\bintkernelvar_2 -2\bintkernel\|_{\infty}$$
 holds true surely. 
\end{lemma}

 \subsection*{Uniqueness of minimizers over compact convex space}
Recall the bilinear functional $\dbinnerp {\cdot, \cdot}$ introduced in \eqref{eq:bilinearFn}  \begin{align*}
\dbinnerp {\bintkernelvar_1, \bintkernelvar_2}:=\frac{1}{L}\sum_{l=1}^{L}\mathbb{E}_{\mu_0}  \bigg[\left\langle \rhsfo_{\bintkernelvar_1}(\bX(t_l)) , \rhsfo_{\bintkernelvar_2}(\bX(t_l)) \right\rangle_{\mathcal{S}} \bigg],
\end{align*} 
for any $\bintkernelvar_1, \bintkernelvar_2 \in \bhypspace$. Then the coercivity condition  \eqref{firstorder:gencoer} can be rephrased as: for all $\bintkernelvar \in \bhypspace$
\[
c_{L,N,\bhypspace} \norm{\bintkernelvar(\cdot)\cdot }^2_{\bm{L}^2(\brhoL)}\leq  \dbinnerp {\bintkernelvar, \bintkernelvar} 
\]

\begin{proposition}\label{firstordersystem:convexity}
Let
 $$\widehat{\bintkernel}_{\infty, \bhypspace}:=\argmin{\bintkernelvar \in \bhypspace} \mbf{\mathcal{E}}_{\infty}(\bintkernelvar);$$ 
  then for all $\bintkernelvar \in \mathcal{H}$, 
 \begin{equation}\label{eq_minH}
 \mbf{\mathcal{E}}_{\infty}(\bintkernelvar)-\mbf{ \mathcal{E}}_{\infty}(\widehat{\bintkernel}_{\infty, \hypspace}) \geq c_{L,N,\bhypspace} \|\bintkernelvar(\cdot)\cdot-\widehat{\bintkernel}_{\infty, \hypspace}(\cdot)\cdot\|_{\bm{L}^2(\brhoL) }^2\,.
 \end{equation}
 As a consequence, the minimizer of $\mbf{\mathcal{E}}_{\infty}$ over $\bhypspace$ is unique in $\bm{L}^2(\brhoL)$. 
\end{proposition}
\begin{proof} For $\bintkernelvar \in \bhypspace$, we have 
\begin{align*}
\mbf{\mathcal{E}}_{\infty}(\bintkernelvar)- \mbf{\mathcal{E}}_{\infty}(\widehat{\bintkernel}_{\infty,\bhypspace})
&=\dbinnerp{\bintkernelvar-\bintkernel, \bintkernelvar-\bintkernel}-\dbinnerp{\widehat{\bintkernel}_{\infty,\bhypspace}-\bintkernel, \widehat{\bintkernel}_{\infty,\bhypspace}-\bintkernel}\\ 
&=\dbinnerp{ \bintkernelvar-\widehat{\bintkernel}_{\infty,\bhypspace}, \bintkernelvar+\widehat{\bintkernel}_{\infty,\bhypspace}-2\bintkernel}\\ 
&= \dbinnerp{\bintkernelvar-\widehat{\bintkernel}_{\infty,\bhypspace}, \bintkernelvar-\widehat{\bintkernel}_{\infty,\bhypspace} }+ 2\dbinnerp{\bintkernelvar-\widehat{\bintkernel}_{\infty,\bhypspace}, \widehat{\bintkernel}_{\infty,\bhypspace}-\bintkernel}.
\end{align*} 
Note that by the coercivity condition, $\dbinnerp{\bintkernelvar-\widehat{\bintkernel}_{\infty,\hypspace}, \bintkernelvar-\widehat{\bintkernel}_{\infty,\bhypspace} } \geq c_{L,N,\bhypspace} \|\bintkernelvar(\cdot)\cdot-\widehat{\bintkernelvar}_{\infty,\bhypspace}(\cdot)\cdot\|_{L^2(\brhoL) }^2$. Therefore, to prove \eqref{eq_minH}, it suffices to show that  $\dbinnerp{\bintkernelvar-\widehat{\bintkernel}_{L, \infty,\bhypspace}, \widehat{\bintkernel}_{\infty,\bhypspace}-\bintkernel} \geq 0$. To see this, the convexity of $\bhypspace$ implies $t\bintkernelvar+(1-t) \widehat{\bintkernel}_{\infty,\bhypspace} \in \bhypspace$, $\forall t \in [0, 1]$.  
For $t\in (0,1]$, we have 
\begin{align*}
\mbf{\mathcal{E}}_{\infty}(t\bintkernelvar+(1-t) \widehat{\bintkernel}_{\infty,\bhypspace} )- \mbf{\mathcal{E}}_{\infty}(\widehat{\bintkernel}_{\infty,\bhypspace}) \geq 0
\end{align*}
since $\widehat{\bintkernel}_{\infty,\bhypspace}$ is a minimizer in $\bhypspace$, therefore
\begin{align*}
 &t\dbinnerp{ \bintkernelvar-\widehat{\bintkernel}_{\infty, \bhypspace}, t\bintkernelvar+(2-t)\widehat{\bintkernel}_{\infty,\bhypspace}-2\bintkernel} \geq 0\\ &\Leftrightarrow \dbinnerp{\bintkernel-\widehat{\bintkernel}_{\infty,\bhypspace}, t\bintkernel+(2-t)\widehat{\bintkernel}_{\infty, \bintkernel}-2\bintkernel}\geq 0.
\end{align*} 
Since  the bilinear functional $\dbinnerp {\cdot, \cdot}$  is continuous over $\bhypspace \times \bhypspace$ (see Proposition \ref{firstordersystem:empiricalerrorfunctional}), letting $t \rightarrow 0^+$, by a continuity argument we have $\dbinnerp{\bintkernelvar-\widehat{\bintkernel}_{\infty, \bhypspace}, 2\widehat{\bintkernel}_{\infty, \bhypspace}-2\bintkernel} \geq 0.$ 

\end{proof}

 \subsection*{Uniform estimates on defect functions}
\begin{lemma} \label{firstorder:covering}Denote the minimizer of  $\bmE_{\infty}(\cdot)$ over $\bhypspace$ by  
\begin{align}
\widehat \bintkernel_{\infty, \bhypspace} =\argmin{\bintkernelvar \in \bhypspace} \bmE_{\infty}( \bintkernelvar). \end{align} For any $ \bintkernelvar \in \bhypspace$, define 
\begin{align}
& \mathcal{D}_{\infty, \bhypspace}(\bintkernelvar):=\bmE_{\infty}(\bintkernelvar)-\bmE_{\infty}(\widehat\bintkernel_{\infty,\bhypspace}) \,,\label{firstorder:difference1}\\
& \mathcal{D}_{M, \bhypspace}(\bintkernelvar):=\bmE_{M}(\bintkernelvar)-\bmE_{M}(\widehat\bintkernel_{\infty, \bhypspace})\,. \label{firstorder:difference2}
\end{align}
For all $\epsilon>0$ and $0<\alpha<1$,  if $\bintkernelvar_1 \in \bhypspace$ satisfies $$ \frac{\mathcal{D}_{\infty, \bhypspace}(\bintkernelvar_1)-\mathcal{D}_{M, \bhypspace}(\bintkernelvar_1)}{\mathcal{D}_{\infty, \bhypspace}(\bintkernelvar_1)+\epsilon}  < \alpha\,,$$
then for all $ \bintkernelvar_2 \in \bhypspace$ such that $\| \bintkernelvar_1 -\bintkernelvar_2\|_{\infty}\leq \frac{\alpha \epsilon}{8S_0R^2K^4}$ we have
 $$  \frac{\mathcal{D}_{\infty, \bhypspace}( \bintkernelvar_2)-\mathcal{D}_{M, \bhypspace}( \bintkernelvar_2)}{\mathcal{D}_{\infty, \bhypspace}( \bintkernelvar_2)+\epsilon}  <3 \alpha.  $$
\end{lemma}
\begin{proof}
 For $ \bintkernelvar \in \bhypspace$, recall the definition \eqref{e:defectFunction} of the defect function $L_{M}( \bintkernelvar)=\bmE_{L, \infty}( \bintkernelvar)-\bmE_{M}(\bintkernelvar)$. We have
\begin{align*}
 \frac{\mathcal{D}_{\infty, \bhypspace}( \bintkernelvar_2)-\mathcal{D}_{M, \bhypspace}( \bintkernelvar_2)}{\mathcal{D}_{\infty, \bhypspace}( \bintkernelvar_2)+\epsilon} &=\frac{\bmE_{\infty}( \bintkernelvar_2)-\bmE_{\infty}(\widehat \bintkernel_{\infty,\bhypspace})-(   \bmE_{M}( \bintkernelvar_2)-\bmE_{M}( \widehat \bintkernel_{\infty,\bhypspace}))}{\mathcal{D}_{\infty, \bhypspace}( \bintkernelvar_2)+\epsilon}\\&=\frac{ L_M( \bintkernelvar_2)-L_M(\bintkernelvar_1)}{\mathcal{D}_{\infty, \bhypspace}( \bintkernelvar_2)+\epsilon}+\frac{ L_M( \bintkernelvar_1)-L_M(\widehat \bintkernel_{\infty,\bhypspace} )}{\mathcal{D}_{\infty, \bhypspace}( \bintkernelvar_2)+\epsilon}
 \end{align*} By Lemma \ref{berstein}, we have 
 $$ L_M( \bintkernelvar_2)-L_M(\bintkernelvar_1) \leq  8S_0R^2K^4 \|\bintkernelvar_2-\bintkernelvar_1 \|_{\infty} \leq \alpha \epsilon\,.$$ 
Notice that $\mathcal{D}_{\infty, \bhypspace}( \bintkernelvar_2) \geq 0$ and therefore, 
 $$ \frac{ L_M( \bintkernelvar_1)-L_M(\bintkernelvar_2) }{\mathcal{D}_{\infty, \bhypspace}( \bintkernelvar_2)+\epsilon} \leq \alpha\,. $$
For the second term, we use  Proposition  \ref{firstordersystem:empiricalerrorfunctional} and the fact $\alpha<1$ to obtain
$$ \mathcal{E}_{\infty}(\bintkernelvar_1)-\mathcal{E}_{\infty}( \bintkernelvar_2)  < 4S_0R^2K^4\| \bintkernelvar_1-  \bintkernelvar_2\|_{\infty}  < \epsilon.$$ This implies that 
$$\mathcal{D}_{\infty, \bhypspace}(\bintkernelvar_1)-\mathcal{D}_{\infty,\bhypspace}( \bintkernelvar_2)= \mathcal{E}_{\infty}(\bintkernelvar_1)-\mathcal{E}_{\infty}( \bintkernelvar_2)  < \epsilon \leq \epsilon +\mathcal{D}_{\infty,\bhypspace}( \bintkernelvar_2),$$ which implies that 
$$\frac{\mathcal{D}_{\infty,\bhypspace}(\bintkernelvar_1)+\epsilon}{\mathcal{D}_{\infty,\bhypspace}( \bintkernelvar_2)+\epsilon} \leq 2.$$ But then 
$$
\frac{ L_M( \bintkernelvar_1)-L_M(\widehat \bintkernel_{\infty, \bhypspace} )}{\mathcal{D}_{\infty, \bhypspace}( \bintkernelvar_2)+\epsilon}= \frac{\mathcal{D}_{\infty, \bhypspace}(\bintkernelvar_1)-\mathcal{D}_{M, \bhypspace}(\bintkernelvar_1)}{\mathcal{D}_{\infty, \bhypspace}(\bintkernelvar_1)+\epsilon} \frac{\mathcal{D}_{\infty, \bhypspace}(\bintkernelvar_1)+\epsilon }{\mathcal{D}_{\infty, \bhypspace}(\bintkernelvar_2)+\epsilon} < 2 \alpha,
$$ 
and the conclusion follows by summing two estimates. 
\end{proof}

\begin{proposition}\label{firstorder:sampleerror}For all $\epsilon>0$ and $0<\alpha<1$, we have 
$$  P_{\mu_0}\bigg\{ \sup_{\bintkernelvar \in \bhypspace} \frac{\mathcal{D}_{\infty,\bhypspace}(\bintkernelvar)-\mathcal{D}_{M,\bhypspace}(\bintkernelvar) }{ \mathcal{D}_{\infty,\bhypspace}(\bintkernelvar)+\epsilon} \geq 3 \alpha  \bigg\}  \leq  \mathcal{N}\left(\bhypspace, \frac{\alpha\epsilon}{8S_0R^2K^4}\right) e^{-\frac{c_{L,N,\bhypspace}\alpha^2 M \epsilon}{32S_0K^4}}$$
where the covering number $\mathcal{N}(\bhypspace,\delta)$ denotes the minimal number of balls in $\bhypspace$, with respect to the $\infty$-norm, with radius $\delta$ covering $\bhypspace$. 
\end{proposition}
\begin{proof} For all $\epsilon>0$ and $0<\alpha<1$, we first show that for any $\bintkernelvar \in \bhypspace$,  
\begin{align*}  P_{\mu_0}\bigg\{ \frac{\mathcal{D}_{\infty, \bhypspace}(\bintkernelvar)-\mathcal{D}_{M, \bhypspace}(\bintkernelvar)}{\mathcal{D}_{\infty, \bhypspace}(\bintkernelvar)+\epsilon}  \geq \alpha \bigg\}  \leq e^{\frac{-c_{L,N,\bhypspace}\alpha^2  M \epsilon}{32\spaceM_0^2 K^4}}\,.
\end{align*}
We consider the random variable $\frac{1}{L} \sum_{l=1}^{L}\big(\bmE_{\bX(t_l)}( \bintkernelvar)- \bmE_{\bX(t_l)}(\widehat \bintkernel_{\infty, \bhypspace})\big)$, and let $\sigma^2$ be its variance. From  \eqref{firstordersystem:expectationerrorfunctional} and the coercivity condition \eqref{firstorder:gencoer},  we obtain that 
\begin{align}\label{firstorder:variance}
\sigma^2 &\leq \E\big[ \big | \frac{1}{L} \sum_{l=1}^{L}\big(\bmE_{\bX(t_l)}( \bintkernelvar)- \bmE_{\bX(t_l)}(\widehat \bintkernel_{\infty,\bhypspace})\big)\big|^2 \big] \nonumber \\ &\leq   \frac{1}{L} \sum_{l=1}^{L}  \E \big[ \big| \bmE_{\bX(t_l)}(\bintkernelvar)- \bmE_{\bX(t_l)}(\widehat \bintkernel_{\infty,\bhypspace}) \big|^2  \big]\nonumber \\ &\leq   K^4 \| \bintkernelvar(\cdot)\cdot-\widehat \bintkernel_{\infty,\bhypspace}(\cdot)\cdot\|_{\bL^2(\brhoL)}^2 \| \bintkernelvar(\cdot)\cdot+\widehat \bintkernel_{\infty,\bhypspace}(\cdot)\cdot-2\bintkernel(\cdot)\cdot \|_{\infty}^2 \nonumber \\ &\leq \frac{K^4 }{c_{L,N,\bhypspace}}(\bmE_{\infty}( \bintkernelvar)-\bmE_{\infty}(\widehat \bintkernel_{\infty,\bhypspace})) \| \bintkernelvar(\cdot)\cdot+\widehat \bintkernel_{\infty,\bhypspace}(\cdot)\cdot-2\bintkernel(\cdot)\cdot \|_{\infty}^2 \nonumber \\ &\leq \frac{(2S_0+2S)^2R^2K^4}{c_{L,N,\bhypspace}}(\bmE_{\infty}( \bintkernelvar)-\bmE_{\infty}(\widehat \bintkernel_{\infty,\bhypspace}))\nonumber \\&\leq   \frac{16S_0^2R^2K^4}{c_{L,N,\bhypspace}}\mathcal{D}_{\infty, \bhypspace}(\bintkernelvar).
\end{align}
We also have that $$| \frac{1}{L} \sum_{l=1}^{L}\big(\bmE_{\bX(t_l)}(\bintkernelvar)- \bmE_{\bX(t_l)}(\widehat \bintkernel_{\infty,\bhypspace})\big)| \leq 8S_0^2R^2K^4$$ 
holds true almost surely. Applying the one-sided Bernstein's inequality to the random variable $$\frac{1}{L} \sum_{l=1}^{L}\big(\bmE_{\bX(t_l)}(\bintkernelvar)- \bmE_{\bX(t_l)}(\widehat \bintkernel_{\infty, \bhypspace} )\big),$$  
we obtain that 
$$P_{\mu_0}\big\{   \frac{ \mathcal{D}_{\infty, \bhypspace}(\bintkernelvar)- \mathcal{D}_{M, \bhypspace}(\bintkernelvar)}{ \mathcal{D}_{\infty, \bhypspace}(\bintkernelvar)+\epsilon} \geq \alpha  \big\} \leq e^{-\frac{\alpha^2(\mathcal{D}_{\infty, \bhypspace}(\bintkernelvar)+\epsilon)^2 M }{2(\sigma^2+\frac{8S_0^2R^2K^4\alpha(\mathcal{D}_{\infty, \bhypspace}(\bintkernelvar)+\epsilon)}{3})}}.$$
Now we estimate
\begin{align*}
&\frac{c_{L,N,\bhypspace}  \epsilon}{32S_0^2R^2K^6}  \leq \frac{(\mathcal{D}_{\infty, \bhypspace}(\bintkernelvar)+\epsilon)^2  }{2(\sigma^2+\frac{8S_0^2R^2K^4\alpha(\mathcal{D}_{\infty, \bhypspace}(\bintkernelvar)+\epsilon)}{3})}\,,\text{i.e.} \\ 
& \frac{c_{L,N,\bhypspace}  \epsilon}{16S_0^2R^2K^6}(\sigma^2+\frac{8S_0^2R^2K^4\alpha(\mathcal{D}_{\infty, \bhypspace}(\bintkernelvar)+\epsilon)}{3} ) \leq (\mathcal{D}_{\infty, \bhypspace}(\bintkernelvar)+\epsilon)^2. 
\end{align*} 
By the estimate \eqref{firstorder:variance}, since $0< \alpha \leq 1$ and $0<c_L<K^2$ it suffices to show 
$$\mathcal{D}_{\infty, \bhypspace}(\bintkernelvar)\epsilon+\frac{\epsilon(\mathcal{D}_{\infty, \bhypspace}(\bintkernelvar)+\epsilon)}{6} \leq (\mathcal{D}_{\infty, \bhypspace}(\bintkernelvar)+\epsilon)^2,$$ which follows from
$2\mathcal{D}_{\infty, \bhypspace}(\bintkernelvar)\epsilon+\epsilon^2 \leq (\mathcal{D}_{\infty, \bhypspace}(\bintkernelvar)+\epsilon)^2$.

 Given $\epsilon>0$,  consider $\bintkernelvar_j$ such that the disks $D_j$ centered at $\bintkernelvar_j$ and with radius $\frac{\alpha\epsilon}{8S_0R^2K^4}$ cover $\bhypspace$ for $j=1,\cdots,\mathcal{N}(\bhypspace, \frac{\alpha\epsilon}{8S_0R^2K^4})$.  By Lemma \ref{firstorder:covering}, we have 
 $$\sup_{\bintkernelvar \in D_j} \frac{\mathcal{D}_{\infty,\bhypspace}(\bintkernelvar)-\mathcal{D}_{M,\bhypspace}(\bintkernelvar) }{ \mathcal{D}_{\infty,\bhypspace}(\bintkernelvar)+\epsilon} \geq 3 \alpha \Rightarrow  \frac{\mathcal{D}_{\infty,\bhypspace}(\bintkernelvar_j)-\mathcal{D}_{M,\bhypspace}(\bintkernelvar_j) }{ \mathcal{D}_{\infty,\bhypspace}(\bintkernelvar_j)+\epsilon}\geq \alpha.$$
  We conclude that, for each $j$,
 $$P_{\mu_0}\bigg\{ \sup_{\bintkernelvar \in D_j} \frac{\mathcal{D}_{\infty,\bhypspace}(\bintkernelvar)-\mathcal{D}_{M,\bhypspace}(\bintkernelvar) }{ \mathcal{D}_{\infty,\bhypspace}(\bintkernelvar)+\epsilon} \geq 3 \alpha  \bigg\} \leq P_{\mu_0}\bigg\{ \frac{\mathcal{D}_{\infty,\bhypspace}(\bintkernelvar_j)-\mathcal{D}_{M,\bhypspace}(\bintkernelvar_j) }{ \mathcal{D}_{\infty,\bhypspace}(\bintkernelvar_j)+\epsilon} \geq \alpha \bigg\} \leq e^{-\frac{c_{L,N,\bhypspace}\alpha^2 M \epsilon}{32S_0^2R^2K^6}}.$$ Since $\bhypspace=\cup_{j}D_j$, 
 \begin{align*}
 P_{\mu_0}\bigg\{ \sup_{\bintkernelvar \in \bhypspace} \frac{\mathcal{D}_{\infty,\bhypspace}(\bintkernelvar)-\mathcal{D}_{M,\bhypspace}(\bintkernelvar) }{ \mathcal{D}_{\infty,\bhypspace}(\bintkernelvar)+\epsilon} \geq 3 \alpha  \bigg\}  &\leq \sum_{j} P_{\mu_0}\bigg\{\sup_{\bintkernelvar \in D_j} \frac{\mathcal{D}_{\infty,\bhypspace}(\bintkernelvar)-\mathcal{D}_{M,\bhypspace}(\bintkernelvar) }{ \mathcal{D}_{\infty,\bhypspace}(\bintkernelvar)+\epsilon} \geq 3 \alpha  \bigg\} \\ &\leq  \mathcal{N}(\bhypspace, \frac{\alpha\epsilon}{8S_0R^2K^4}) e^{-\frac{c_{L,N,\bhypspace}\alpha^2 M \epsilon}{32S_0^2R^2K^6}}.
 \end{align*}
\end{proof}

\begin{proof} [\textbf{of Theorem \rm{\ref{firstorder:maintheorem}} }]
Put $\alpha=\frac{1}{6}$ in Proposition \ref{firstorder:sampleerror}. 
We know that,  with probability at least
 $$1-\mathcal{N}(\bhypspace, \frac{\epsilon}{48\spaceM_0R^2K^4}) e^{-\frac{c_{L,N,\bhypspace}M\epsilon}{1152\spaceM_0^2R^2K^6} }, $$ we have 
  $$\sup_{ \bintkernelvar \in \bhypspace}  \frac{\mathcal{D}_{\infty, \bhypspace}(\bintkernelvar)-\mathcal{D}_{M, \bhypspace}(\bintkernelvar)}{\mathcal{D}_{\infty, \bhypspace}(\bintkernelvar)+\epsilon}  <\frac{1}{2},$$
  and therefore, for all $\bintkernelvar \in \bhypspace$, $$\frac{1}{2}\mathcal{D}_{\infty, \bhypspace}(\bintkernelvar)<\mathcal{D}_{M, \bhypspace}(\bintkernelvar)+\frac{1}{2}\epsilon. $$
 Taking $\bintkernelvar=\widehat \bintkernel_{M,\bhypspace}$, we have
 $$  \mathcal{D}_{\infty, \bhypspace}(\widehat \bintkernel_{M,\bhypspace} )< 2\mathcal{D}_{M, \bhypspace}(\widehat \bintkernel_{M,\bhypspace})+\epsilon\,.$$
But $\mathcal{D}_{M, \bhypspace}(\widehat \bintkernel_{M,\bhypspace})=\mathcal{E}_{M}(\widehat \bintkernel_{M,\bhypspace})-\mathcal{E}_{M}(\widehat\bintkernel_{\infty,\bhypspace}) \leq 0$ and hence by Proposition \ref{firstordersystem:convexity} we have 
 $$c_{L,N,\bhypspace} \|\widehat \bintkernel_{M,\bhypspace}(\cdot)\cdot-\widehat\bintkernel_{\infty,\bhypspace}(\cdot)\cdot\|_{\bL^2( \brhoL)}^2 \leq \mathcal{D}_{\infty, \bhypspace}(\widehat \bintkernel_{M,\bhypspace} )<\epsilon.$$ Therefore, 
 \begin{align*}
 \|\widehat \bintkernel_{M,\bhypspace}(\cdot)\cdot-\bintkernel(\cdot)\cdot\|_{\bL^2(\brhoL)}^2 &\leq 2 \|\widehat \bintkernel_{M,\bhypspace}(\cdot)\cdot-\widehat\bintkernel_{\infty,\bhypspace}(\cdot)\cdot\|_{L^2(\brhoL)}^2
+2\|\widehat\bintkernel_{\infty,\bhypspace}(\cdot)\cdot-\bintkernel(\cdot)\cdot\|_{\bL^2(\brhoL)}^2 \\
&\leq  \frac{2}{c_{L,N,\bhypspace}}( \epsilon +\inf_{\bintkernelvar \in \bhypspace} \| \bintkernelvar(\cdot)\cdot -\bintkernel(\cdot)\cdot\|_{\bL^2(\brhoL)}^2) \\& \leq  \frac{2}{c_{L,N,\bhypspace}}( \epsilon +\inf_{\bintkernelvar \in \bhypspace} R^2\| \bintkernelvar -\bintkernel \|_{\infty}^2)
\end{align*}  
where the last two inequalities follows from the coercivity condition and by the definition of $\widehat\bintkernel_{\infty,\bhypspace}$. 
Given $0<\delta<1$, we let $$1-\mathcal{N}(\bhypspace, \frac{\epsilon}{48\spaceM_0R^2K^4 }) e^{-\frac{c_{L,N,\bhypspace}M\epsilon}{1152\spaceM_0^2R^2K^6} } \geq 1-\delta$$ and the conclusion follows.
\end{proof}

\begin{proof}[\textbf{of Theorem \rm{\ref{main:consistency}} }]
According to the definition of the coercivity constant, we have $c_{L,N,\bhypspace_M} \geq c_{L,N,\cup_{M}{\bhypspace_M}}$. 
Then by the proof of Theorem \ref{firstorder:maintheorem}, we obtain
\begin{align}\label{importantinequality}
c_{L,N,\cup_{M}\mbf{\bhypspace_M}}\|\widehat \bintkernel_{M,\bhypspace_M}(\cdot)\cdot-\bintkernel(\cdot)\cdot\|_{\bL^2(\brhoL)}^2&\leq \mbf{\mathcal{D}}_{\infty,\bhypspace_M}(\widehat\bintkernel_{M,\bhypspace_M})+\mbf{\mathcal{E}}_{\infty}(\widehat\bintkernel_{\infty,\bhypspace_M})
\end{align} For any $\epsilon>0$, the inequality \eqref{importantinequality} implies that 
\begin{align*}
&P_{\mu_0} \{c_{L,N,\cup_{M}{\bhypspace_M}} \|\widehat \bintkernel_{M,\bhypspace_M}(\cdot)\cdot-\bintkernel(\cdot)\cdot\|_{\bL^2(\brhoL)}^2 \geq \epsilon \} \\
&\qquad\leq P_{\mu_0} \{\mbf{\mathcal{D}}_{\infty,\bhypspace_M}(\widehat\bintkernel_{M,\bhypspace_M})+\mbf{\mathcal{E}}_{\infty}(\widehat\bintkernel_{\infty,\bhypspace_M}) \geq \epsilon \}\\
&\qquad\leq P_{\mu_0}\{\mbf{\mathcal{D}}_{\infty,\bhypspace_M}(\widehat\bintkernel_{M,\bhypspace_M})\geq \frac{\epsilon}{2}\}+P_{\mu_0}\{\mbf{\mathcal{E}}_{\infty}(\widehat\bintkernel_{\infty,\bhypspace_M}) \geq \frac{\epsilon}{2}\}. 
\end{align*}
 
 From the proof of Theorem \ref{firstorder:maintheorem},  we obtain that 
\begin{align*}
P_{\mu_0}\{\mbf{\mathcal{D}}_{\infty,\bhypspace_M}(\widehat\bintkernel_{M,\bhypspace_M}) \geq \epsilon\} &\leq \mathcal{N}(\bhypspace_M,\frac{\epsilon}{48S^2R^2K^4})e^{-\frac{c_{L,N,\bhypspace_M}M\epsilon}{1152S^2R^2K^6}}\\
&\leq \mathcal{N}(\cup_{M}{\bhypspace_M},\frac{\epsilon}{48S^2R^2K^4})e^{-\frac{c_{L,N,\cup_{M}\bhypspace_M}M\epsilon}{1152S^2R^2K^6}}\\&\leq C(\cup_{M}\bhypspace_M, \epsilon) e^{-\frac{c_{L,N, \cup_{M}\bhypspace_M}M\epsilon}{1152S^2R^2K^6}},
\end{align*} where $C_1$ is an absolute constant independent of $M$ and $C(\cup_{M}\bhypspace_M, \epsilon)$ is a finite positive constant due to the compactness of $\cup_{M}\bhypspace_M$.
Therefore, 
\begin{align*}
\sum_{M=1}^{\infty} P_{\mu_0}\{\mbf{\mathcal{D}}_{\infty,\bhypspace_M}(\widehat\bintkernel_{M,\bhypspace_M}) \geq \epsilon\}&\leq \sum_{M=1}^{\infty}C(\cup_{M}\bhypspace_M, \epsilon)e^{-\frac{c_{L,N,}M\epsilon}{1152S^2R^2K^6}}<\infty.
\end{align*}

On the other hand, the estimate \eqref{firstordersystem:expectationerrorfunctional} yields that 
$$\mbf{\mathcal{E}}_{\infty}(\widehat\bintkernel_{\infty,\bhypspace_M}) \leq 4K^4SR^2 \inf_{\mbf{f}\in \bhypspace_M} \|\mbf{f}-\bintkernel\|_{\infty}\xrightarrow{M\rightarrow \infty} 0.$$
Therefore, $P_{\mu_0}\{\mbf{\mathcal{E}}_{\infty}(\widehat\bintkernel_{\infty,\bhypspace_M}) \geq \epsilon\}=0$ when $M$ is large enough. So we have  $$ \sum_{M=1}^{\infty}P_{\mu_0}\{\mbf{\mathcal{E}}_{\infty}(\widehat\bintkernel_{\infty,\bhypspace_M}) \geq \epsilon\} < \infty$$  
By the Borel-Cantelli Lemma, we have 
$$P_{\mu_0}\big\{\limsup_{{M\rightarrow \infty}}\{ c_{L,N,\cup_{M}{\bhypspace_M}} \|\widehat \bintkernel_{M,\bhypspace_M}(\cdot)\cdot-\bintkernel(\cdot)\cdot\|_{\bL^2(\brhoL)}^2  \geq \epsilon\}\big\}=0,$$
which is equivalent to 
$$P_{\mu_0}\big\{\lim\inf_{{M\rightarrow \infty}}\{ c_{L,N,\cup_{M}{\bhypspace_M}} \|\widehat \bintkernel_{M,\bhypspace_M}(\cdot)\cdot-\bintkernel(\cdot)\cdot\|_{\bL^2(\brhoL)}^2  \leq \epsilon\}\big\}=1.$$
The conclusion follows. 
\end{proof}

\bibliography{learning_dynamics.bib}

\begin{thebibliography}{72}
\providecommand{\natexlab}[1]{#1}
\providecommand{\url}[1]{\texttt{#1}}
\expandafter\ifx\csname urlstyle\endcsname\relax
  \providecommand{\doi}[1]{doi: #1}\else
  \providecommand{\doi}{doi: \begingroup \urlstyle{rm}\Url}\fi

\bibitem[Abebe et~al.(2018)Abebe, Kleinberg, Parkes, and
  Tsourakakis]{abebe2018opinion}
Rediet Abebe, Jon Kleinberg, David Parkes, and Charalampos~E Tsourakakis.
\newblock Opinion dynamics with varying susceptibility to persuasion.
\newblock In \emph{Proceedings of the 24th ACM SIGKDD International Conference
  on Knowledge Discovery \& Data Mining}, pages 1089--1098, 2018.

\bibitem[Bellomo et~al.(2017)Bellomo, Degond, and Tadmor]{BDT2017}
N.~Bellomo, P.~Degond, and E.~Tadmor, editors.
\newblock \emph{Active Particles, Volume 1}.
\newblock Springer International Publishing AG, Switerland, 2017.

\bibitem[Binev et~al.(2005)Binev, Cohen, Dahmen, DeVore, and
  Temlyakov]{binev2005universal}
P.~Binev, A.~Cohen, W.~Dahmen, R.~DeVore, and V.~Temlyakov.
\newblock Universal algorithms for learning theory part i: piecewise constant
  functions.
\newblock \emph{Journal of Machine Learning Research}, 6\penalty0
  (Sep):\penalty0 1297--1321, 2005.

\bibitem[Blodel et~al.(2009)Blodel, Hendricks, and Tsitsiklis]{BHT2009}
V.~Blodel, J.~Hendricks, and J.~Tsitsiklis.
\newblock {On Krause's multi-agent consensus model with state-dependent
  connectivity}.
\newblock \emph{Automatic Control, IEEE Transactions on}, 54\penalty0
  (11):\penalty0 2586 -- 2597, 2009.

\bibitem[Bongard and Lipson(2007)]{bongard2007automated}
J.~Bongard and H.~Lipson.
\newblock Automated reverse engineering of nonlinear dynamical systems.
\newblock \emph{Proceedings of the National Academy of Sciences of the United
  States of America}, 104\penalty0 (24):\penalty0 9943--9948, 2007.

\bibitem[Bongini et~al.(2017)Bongini, Fornasier, Hansen, and Maggioni]{BFHM17}
M.~Bongini, M.~Fornasier, M.~Hansen, and M.~Maggioni.
\newblock Inferring interaction rules from observations of evolutive systems
  {I}: The variational approach.
\newblock \emph{Mathematical Models and Methods in Applied Sciences},
  27\penalty0 (05):\penalty0 909--951, 2017.

\bibitem[Boninsegna et~al.(2018)Boninsegna, N{\"u}ske, and
  Clementi]{boninsegna2018sparse}
L.~Boninsegna, F.~N{\"u}ske, and C.~Clementi.
\newblock Sparse learning of stochastic dynamical equations.
\newblock \emph{The Journal of Chemical Physics}, 148\penalty0 (24):\penalty0
  241723, 2018.

\bibitem[Borwein and Erd{\'e}lyi(1997)]{borwein1997generalizations}
P.~Borwein and T.~Erd{\'e}lyi.
\newblock Generalizations of {M{\"u}ntz's} theorem via a remez-type inequality
  for m{\"u}ntz spaces.
\newblock \emph{Journal of the American Mathematical Society}, pages 327--349,
  1997.

\bibitem[Brillinger et~al.(2012)Brillinger, Preisler, Ager, and
  Kie]{brillinger2012use}
David~R Brillinger, Haiganoush~K Preisler, Alan~A Ager, and JG~Kie.
\newblock The use of potential functions in modelling animal movement.
\newblock In \emph{Selected Works of David Brillinger}, pages 385--409.
  Springer, 2012.

\bibitem[Brillinger et~al.(2011)Brillinger, Preisler, and
  Wisdom]{brillinger2011modelling}
DR~Brillinger, HK~Preisler, and MJ~Wisdom.
\newblock Modelling particles moving in a potential field with pairwise
  interactions and an application.
\newblock \emph{Brazilian Journal of Probability and Statistics}, pages
  421--436, 2011.

\bibitem[{Brugna} and {Toscani}(2015)]{BT2015}
C.~{Brugna} and G.~{Toscani}.
\newblock {Kinetic models of opinion formation in the presence of personal
  conviction}.
\newblock \emph{Physical Review E}, 92\penalty0 (5):\penalty0 052818, 2015.
\newblock \doi{10.1103/PhysRevE.92.052818}.

\bibitem[Brunel(2008)]{brunel2008parameter}
N.~Brunel.
\newblock Parameter estimation of {ODE}s via nonparametric estimators.
\newblock \emph{Electronic Journal of Statistics}, 2:\penalty0 1242--1267,
  2008.

\bibitem[Brunton et~al.(2016)Brunton, Proctor, and Kutz]{BPK2016}
S.~Brunton, J.~Proctor, and J.~Kutz.
\newblock Discovering governing equations from data by sparse identification of
  nonlinear dynamical systems.
\newblock \emph{Proceedings of the National Academy of Sciences of the United
  States of America}, 113\penalty0 (15):\penalty0 3932--3937, 2016.
\newblock \doi{10.1073/pnas.1517384113}.
\newblock URL \url{http://www.pnas.org/content/113/15/3932.abstract}.

\bibitem[Brunton et~al.(2017)Brunton, Kutz, and Proctor]{BKP2017}
S.~Brunton, N.~Kutz, and J.~Proctor.
\newblock Data-drive discovery of governing physical laws.
\newblock \emph{SIAM News}, 50\penalty0 (1), 2017.
\newblock URL
  \url{https://sinews.siam.org/Details-Page/data-driven-discovery-of-governing-physical-laws}.

\bibitem[Cao et~al.(2011)Cao, Wang, and Xu]{cao2011robust}
J.~Cao, L.~Wang, and J.~Xu.
\newblock Robust estimation for ordinary differential equation models.
\newblock \emph{Biometrics}, 67\penalty0 (4):\penalty0 1305--1313, 2011.

\bibitem[Chartrand(2011)]{chartrand2011numerical}
R.~Chartrand.
\newblock Numerical differentiation of noisy, nonsmooth data.
\newblock \emph{ISRN Applied Mathematics}, 2011, 2011.

\bibitem[Chen(2014)]{chen2014multi}
X.~Chen.
\newblock \emph{Multi-agent systems with reciprocal interaction laws}.
\newblock PhD thesis, Harvard University, 2014.

\bibitem[Chuang et~al.(2007)Chuang, D'Orsogna, Marthaler, Bertozzi, and
  Chayes]{CDOMBC2007}
Y.~Chuang, M.~D'Orsogna, D.~Marthaler, A.~Bertozzi, and L.~Chayes.
\newblock {State transition and the continuum limit for the 2D interacting,
  self-propelled particle system}.
\newblock \emph{Physica D}, 232:\penalty0 33 -- 47, 2007.

\bibitem[Cisneros et~al.(2016)Cisneros, Wikfeldt, Ojam\"ae, Lu, Xu, Torabifard,
  Bartok, Csanyi, Molinero, and Paesani]{cisneros2016modeling}
Gerardo~Andr{\'e}s Cisneros, Kjartan~Thor Wikfeldt, Lars Ojam\"ae, Jibao Lu,
  Yao Xu, Hedieh Torabifard, Albert~P Bartok, Gabor Csanyi, Valeria Molinero,
  and Francesco Paesani.
\newblock Modeling molecular interactions in water: From pairwise to many-body
  potential energy functions.
\newblock \emph{Chemical reviews}, 116\penalty0 (13):\penalty0 7501--7528,
  2016.

\bibitem[Cleveland and Devlin(1988)]{cleveland1988locally}
W.~Cleveland and S.~Devlin.
\newblock Locally weighted regression: an approach to regression analysis by
  local fitting.
\newblock \emph{Journal of the American Statistical Association}, 83\penalty0
  (403):\penalty0 596--610, 1988.

\bibitem[Cohn and Kumar(2009)]{CK2009}
H.~Cohn and A.~Kumar.
\newblock {Algorithmic design of self-assembling structures}.
\newblock \emph{Proceedings of the National Academy of Sciences of the United
  States of America}, 106:\penalty0 9570 -- 9575, 2009.

\bibitem[Couzin et~al.(2005)Couzin, Krause, Franks, and Levin]{CKFL2005SI}
I.~Couzin, J.~Krause, N.~Franks, and S.~Levin.
\newblock {Effective leadership and decision-making in animal groups on the
  move}.
\newblock \emph{Nature}, 433\penalty0 (7025):\penalty0 513 -- 516, 2005.

\bibitem[Cucker and Smale(2002)]{CS02}
F.~Cucker and S.~Smale.
\newblock On the mathematical foundations of learning.
\newblock \emph{Bulletin of the American Mathematical Society}, 39\penalty0
  (1):\penalty0 1--49, 2002.

\bibitem[Dattner and Klaassen(2015)]{dattner2015optimal}
I.~Dattner and C.~Klaassen.
\newblock Optimal rate of direct estimators in systems of ordinary differential
  equations linear in functions of the parameters.
\newblock \emph{Electronic Journal of Statistics}, 9\penalty0 (2):\penalty0
  1939--1973, 2015.

\bibitem[De et~al.(2014)De, Bhattacharya, Bhattacharya, Ganguly, and
  Chakrabarti]{de2014learning}
Abir De, Sourangshu Bhattacharya, Parantapa Bhattacharya, Niloy Ganguly, and
  Soumen Chakrabarti.
\newblock Learning a linear influence model from transient opinion dynamics.
\newblock In \emph{Proceedings of the 23rd ACM International Conference on
  Conference on Information and Knowledge Management}, pages 401--410, 2014.

\bibitem[DeVore and Lucier(1992)]{devore1992wavelets}
R.~DeVore and B.~Lucier.
\newblock Wavelets.
\newblock \emph{Acta Numerica}, 1:\penalty0 1--56, 1992.

\bibitem[DeVore et~al.(2004)DeVore, Kerkyacharian, Picard, and
  Temlyakov]{devore2004mathematical}
R.~DeVore, G.~Kerkyacharian, D.~Picard, and V.~Temlyakov.
\newblock Mathematical methods for supervised learning.
\newblock \emph{IMI Preprints}, 22:\penalty0 1--51, 2004.

\bibitem[D'Orsogna et~al.(2006)D'Orsogna, Chuang, Bertozzi, and
  Chayes]{DOCBC2006}
M.~D'Orsogna, Y.~Chuang, A.~Bertozzi, and L.~Chayes.
\newblock {Self-propelled particles with soft-core interactions: patterns,
  stability, and collapse.}
\newblock \emph{Physical Review Letter}, 96:\penalty0 104 -- 302, 2006.

\bibitem[{Escobedo} et~al.(2014){Escobedo}, {Muro}, {Spector}, and
  {Coppinger}]{EMSC2014}
R.~{Escobedo}, C.~{Muro}, L.~{Spector}, and R.~P. {Coppinger}.
\newblock {Group size, individual role differentiation and effectiveness of
  cooperation in a homogeneous group of hunters}.
\newblock \emph{Journal of the Royal Society Interface}, 11:\penalty0 20140204,
  2014.

\bibitem[Fan and Gijbels(1996)]{FGlocalPolynomial}
J.~Fan and I.~Gijbels.
\newblock \emph{Local Polynomial Modelling and its Applications}.
\newblock 1996.

\bibitem[Fryxell et~al.(2007)Fryxell, Mosser, Sinclair, and Packer]{FMSP2007}
J.~Fryxell, A.~Mosser, A.~Sinclair, and C.~Packer.
\newblock {Group formation stabilizes predator-prey dynamics}.
\newblock \emph{Nature}, 449:\penalty0 1041 -- 1043, 2007.

\bibitem[Gy{\"o}rfi et~al.(2002)Gy{\"o}rfi, Kohler, Krzyzak, and
  Walk]{Gyorfi06}
L.~Gy{\"o}rfi, M.~Kohler, A.~Krzyzak, and H.~Walk.
\newblock \emph{A distribution-free theory of nonparametric regression}.
\newblock Springer, New York, 2002.

\bibitem[Han et~al.(2015)Han, Shen, Wang, and Di]{han2015robust}
X.~Han, Z.~Shen, W.~Wang, and Z.~Di.
\newblock Robust reconstruction of complex networks from sparse data.
\newblock \emph{Physical Review Letters}, 114\penalty0 (2):\penalty0 028701,
  2015.

\bibitem[Hemelrijk and Hildenbrandt(2011)]{hemelrijk2011some}
C.~Hemelrijk and H.~Hildenbrandt.
\newblock Some causes of the variable shape of flocks of birds.
\newblock \emph{PloS One}, 6\penalty0 (8):\penalty0 e22479, 2011.

\bibitem[Kang et~al.(2019)Kang, Liao, and Liu]{kang2019ident}
S.~Kang, W.~Liao, and Y.~Liu.
\newblock Ident: Identifying differential equations with numerical time
  evolution.
\newblock \emph{arXiv preprint arXiv:1904.03538}, 2019.

\bibitem[Katz et~al.(2011)Katz, Tunstrom, Ioannou, Huepe, and
  Couzin]{KTIHC2011}
Y.~Katz, K.~Tunstrom, C.~Ioannou, C.~Huepe, and I.~Couzin.
\newblock {Inferring the structure and dynamics of interactions in schooling
  fish}.
\newblock \emph{Proceedings of the National Academy of Sciences of the United
  States of America}, 108:\penalty0 18720--8725, 2011.

\bibitem[Ke et~al.(2002)Ke, Minett, Au, and Wang]{KMAW2002}
J.~Ke, J.~Minett, C.~Au, and W.~Wang.
\newblock {Self-organization and selection in the emergence of vocabulary}.
\newblock \emph{Complexity}, 7:\penalty0 41 -- 54, 2002.

\bibitem[Knowles and Renka(2014)]{knowles2014methods}
I.~Knowles and R.~Renka.
\newblock Methods for numerical differentiation of noisy data.
\newblock \emph{Electronic Journal of Differential Equations}, 21:\penalty0
  235--246, 2014.

\bibitem[Kolokolnikov et~al.(2013)Kolokolnikov, Carrillo, Bertozzi, Fetecau,
  and Lewis]{kolokolnikov2013emergent}
T.~Kolokolnikov, J.~Carrillo, A.~Bertozzi, R.~Fetecau, and M.~Lewis.
\newblock Emergent behaviour in multi-particle systems with non-local
  interactions.
\newblock \emph{Physica D}, 260:\penalty0 1--4, 2013.

\bibitem[Krause(2000)]{Krause2000SI}
U.~Krause.
\newblock {A discrete nonlinear and non-autonomous model of consensus
  formation}.
\newblock \emph{Communications in difference equations}, pages 227 -- 236,
  2000.

\bibitem[Li et~al.(2019)Li, Lu, Maggioni, Tang, and
  Zhang]{li2019identifiability}
Z.~Li, F.~Lu, M.~Maggioni, S.~Tang, and C.~Zhang.
\newblock On the identifiability of interaction functions in systems of
  interacting particles, 2019.

\bibitem[Liang and Wu(2008)]{liang2008parameter}
H.~Liang and H.~Wu.
\newblock Parameter estimation for differential equation models using a
  framework of measurement error in regression models.
\newblock \emph{Journal of the American Statistical Association}, 103\penalty0
  (484):\penalty0 1570--1583, 2008.

\bibitem[Long et~al.(2017)Long, Lu, Ma, and Dong]{long2017pde}
Z.~Long, Y.~Lu, X.~Ma, and B.~Dong.
\newblock {PDE}-net: Learning {PDE}s from data.
\newblock \emph{arXiv preprint arXiv:1710.09668}, 2017.

\bibitem[Lu et~al.(2019)Lu, Zhong, Tang, and Maggioni]{lu2019nonparametric}
F.~Lu, M.~Zhong, S.~Tang, and M.~Maggioni.
\newblock Nonparametric inference of interaction laws in systems of agents from
  trajectory data.
\newblock \emph{Proceedings of the National Academy of Sciences of the United
  States of America}, 116\penalty0 (29):\penalty0 14424--14433, 2019.

\bibitem[Lu et~al.(2020)Lu, Maggioni, and Tang]{LMT20}
F.~Lu, M.~Maggioni, and S.~Tang.
\newblock Learning interaction kernels in stochastic systems of interacting
  particles, 2020.

\bibitem[Lukeman et~al.(2010)Lukeman, Li, and Edelstein-Keshet]{LLEK2010}
R.~Lukeman, Y.~Li, and L.~Edelstein-Keshet.
\newblock {Inferring individual rules from collective behavior}.
\newblock \emph{Proceedings of the National Academy of Sciences of the United
  States of America}, 107:\penalty0 12576 -- 12580, 2010.

\bibitem[Miao et~al.(2011)Miao, Xia, Perelson, and Wu]{miao2011identifiability}
H.~Miao, X.~Xia, A.~Perelson, and H.~Wu.
\newblock On identifiability of nonlinear {ODE} models and applications in
  viral dynamics.
\newblock \emph{SIAM Review}, 53\penalty0 (1):\penalty0 3--39, 2011.

\bibitem[{Mostch} and {Tadmor}(2014)]{MT2014}
S.~{Mostch} and E.~{Tadmor}.
\newblock {Heterophilious dynamics enhances consensus}.
\newblock \emph{SIAM Review}, 56\penalty0 (4):\penalty0 577 -- 621, 2014.

\bibitem[Nowak(2006)]{Nowak2006}
M.~Nowak.
\newblock {Five rules for the evolution of cooperation}.
\newblock \emph{Science}, 314:\penalty0 1560 -- 1563, 2006.

\bibitem[Olfati-Saber and Murray(2004)]{olfati2004consensus}
R.~Olfati-Saber and R.~Murray.
\newblock Consensus problems in networks of agents with switching topology and
  time-delays.
\newblock \emph{IEEE Transactions on Automatic Control}, 49\penalty0
  (9):\penalty0 1520--1533, 2004.

\bibitem[Parrish and Edelstein-Keshet(1999)]{PEK1999}
J.~Parrish and L.~Edelstein-Keshet.
\newblock {Complexiy, pattern, and evolutionary trade-offs in animal
  aggregation}.
\newblock \emph{Science}, 284:\penalty0 99 -- 101, 1999.

\bibitem[Raissi(2018)]{raissi2018deep}
M.~Raissi.
\newblock Deep hidden physics models: Deep learning of nonlinear partial
  differential equations.
\newblock \emph{The Journal of Machine Learning Research}, 19\penalty0
  (1):\penalty0 932--955, 2018.

\bibitem[Raissi and Karniadakis(2018)]{raissi2018hidden}
M.~Raissi and G.~Karniadakis.
\newblock Hidden physics models: Machine learning of nonlinear partial
  differential equations.
\newblock \emph{Journal of Computational Physics}, 357:\penalty0 125--141,
  2018.

\bibitem[Raissi et~al.(2018)Raissi, Perdikaris, and
  Karniadakis]{raissi2018multistep}
M.~Raissi, P.~Perdikaris, and G.~Karniadakis.
\newblock Multistep neural networks for data-driven discovery of nonlinear
  dynamical systems.
\newblock \emph{arXiv preprint arXiv:1801.01236}, 2018.

\bibitem[Ramsay and Hooker(2018)]{ramsay2005functional}
J.~Ramsay and G.~Hooker.
\newblock Dynamic data analysis: Modeling data with differential equations.
\newblock \emph{Springer Series in Statistics}, 2018.

\bibitem[Ramsay et~al.(2007)Ramsay, Hooker, Campbell, and
  Cao]{ramsay2007parameter}
J.~Ramsay, G.~Hooker, D.~Campbell, and J.~Cao.
\newblock Parameter estimation for differential equations: a generalized
  smoothing approach.
\newblock \emph{Journal of the Royal Statistical Society: Series B (Statistical
  Methodology)}, 69\penalty0 (5):\penalty0 741--796, 2007.

\bibitem[Rudy et~al.(2019)Rudy, Kutz, and Brunton]{rudy2019deep}
H.~Rudy, N.~Kutz, and S.~Brunton.
\newblock Deep learning of dynamics and signal-noise decomposition with
  time-stepping constraints.
\newblock \emph{Journal of Computational Physics}, 2019.

\bibitem[Rudy et~al.(2017)Rudy, Brunton, Proctor, and Kutz]{RBPK2017}
S.~Rudy, S.~Brunton, J.~Proctor, and N.~Kutz.
\newblock Data-driven discovery of partial differential equations.
\newblock \emph{Science Advances}, 3\penalty0 (4):\penalty0 e1602614, 2017.
\newblock \doi{10.1126/sciadv.1602614}.
\newblock URL \url{http://advances.sciencemag.org/content/3/4/e1602614}.

\bibitem[Schaeffer et~al.(2013)Schaeffer, Caflisch, Hauck, and
  Osher]{Schaeffer6634}
H.~Schaeffer, R.~Caflisch, C.~Hauck, and S.~Osher.
\newblock Sparse dynamics for partial differential equations.
\newblock \emph{Proceedings of the National Academy of Sciences of the United
  States of America}, 110\penalty0 (17):\penalty0 6634--6639, 2013.
\newblock ISSN 0027-8424.
\newblock \doi{10.1073/pnas.1302752110}.
\newblock URL \url{https://www.pnas.org/content/110/17/6634}.

\bibitem[Schaeffer et~al.(2018)Schaeffer, Tran, and
  Ward]{schaeffer2018extracting}
H.~Schaeffer, G.~Tran, and R.~Ward.
\newblock Extracting sparse high-dimensional dynamics from limited data.
\newblock \emph{SIAM Journal on Applied Mathematics}, 78\penalty0 (6):\penalty0
  3279--3295, 2018.

\bibitem[Schmidt and Lipson(2009)]{schmidt2009distilling}
M.~Schmidt and H.~Lipson.
\newblock Distilling free-form natural laws from experimental data.
\newblock \emph{Science}, 324\penalty0 (5923):\penalty0 81--85, 2009.

\bibitem[Toner and Tu(1995)]{toner1995long}
J.~Toner and Y.~Tu.
\newblock Long-range order in a two-dimensional dynamical xy model: how birds
  fly together.
\newblock \emph{Physical Review Letters}, 75\penalty0 (23):\penalty0 4326,
  1995.

\bibitem[Tran and Ward(2017)]{TranWardExactRecovery}
G.~Tran and R.~Ward.
\newblock Exact recovery of chaotic systems from highly corrupted data.
\newblock \emph{Multiscale Modeling and Simulation}, 15\penalty0 (3):\penalty0
  1108--1129, 2017.

\bibitem[Tropp(2015)]{tropp2015introduction}
J.~Tropp.
\newblock An introduction to matrix concentration inequalities.
\newblock \emph{Foundations and Trends in Machine Learning}, 8\penalty0
  (1-2):\penalty0 1--230, 2015.

\bibitem[Varah(1982)]{varah1982spline}
J.~Varah.
\newblock A spline least squares method for numerical parameter estimation in
  differential equations.
\newblock \emph{SIAM Journal on Scientific and Statistical Computing},
  3\penalty0 (1):\penalty0 28--46, 1982.

\bibitem[{Vicsek} et~al.(1995){Vicsek}, {Czirok}, {Ben-Jacob}, {Cohen}, and
  {Shochet}]{VCBJCS1995}
T.~{Vicsek}, A.~{Czirok}, E.~{Ben-Jacob}, I.~{Cohen}, and O.~{Shochet}.
\newblock {Novel Type of Phase Transition in a System of Self-Driven
  Particles}.
\newblock \emph{Physical Review Letters}, 75\penalty0 (6):\penalty0 1226 --
  1229, 1995.

\bibitem[Wagner et~al.(2015)Wagner, Mazurek, and
  Morawski]{wagner2015regularised}
J.~Wagner, P.~Mazurek, and R.~Morawski.
\newblock Regularised differentiation of measurement data.
\newblock In \emph{Proc. XXI IMEKO World Congress" Measurement in Research and
  Industry}, pages 1--6, 2015.

\bibitem[Wang et~al.(2018)Wang, Herzog, Kiss, Schwartz, Bloch, Sebek,
  Granados-Fuentes, Wang, and Li]{wang2018inferring}
S.~Wang, E.~Herzog, W.~Kiss, I.~Schwartz, G.~Bloch, M.~Sebek,
  D.~Granados-Fuentes, L.~Wang, and J.~Li.
\newblock Inferring dynamic topology for decoding spatiotemporal structures in
  complex heterogeneous networks.
\newblock \emph{Proceedings of the National Academy of Sciences of the United
  States of America}, 115\penalty0 (37):\penalty0 9300--9305, 2018.

\bibitem[Wu and Xiu(2019)]{wu2019numerical}
Kailiang Wu and Dongbin Xiu.
\newblock Numerical aspects for approximating governing equations using data.
\newblock \emph{Journal of Computational Physics}, 384:\penalty0 200--221,
  2019.

\bibitem[Zhang and Lin(2018)]{zhang2018robust}
S.~Zhang and G.~Lin.
\newblock Robust data-driven discovery of governing physical laws with error
  bars.
\newblock \emph{Proceedings of the Royal Society A: Mathematical, Physical and
  Engineering Sciences}, 474\penalty0 (2217):\penalty0 20180305, 2018.

\bibitem[Zhong et~al.(2020{\natexlab{a}})Zhong, Miller, and Maggioni]{MMZ2020}
Ming Zhong, Jason Miller, and Mauro Maggioni.
\newblock Data-driven modeling of celestial dynamics from {J}et {P}ropulsion
  {L}aboratory’s development ephemeris.
\newblock \emph{In preparation}, 2020{\natexlab{a}}.

\bibitem[Zhong et~al.(2020{\natexlab{b}})Zhong, Miller, and
  Maggioni]{zhong2020data}
Ming Zhong, Jason Miller, and Mauro Maggioni.
\newblock Data-driven discovery of emergent behaviors in collective dynamics.
\newblock \emph{Physica D: Nonlinear Phenomena}, page 132542,
  2020{\natexlab{b}}.

\end{thebibliography}

\end{document}